\documentclass[11pt,onecolumn,letterpaper]{article}
\ifdefined\pdfminorversion\pdfminorversion=7\fi

\usepackage[textwidth=6.75in,textheight=9.25in]{geometry}
\usepackage[T1]{fontenc}
\usepackage{lmodern}

\usepackage{amsmath,amssymb,amsthm,mathtools}
\usepackage[dvipsnames,svgnames]{xcolor}

\usepackage[round,sort&compress]{natbib}

\usepackage{microtype}
\usepackage{enumitem}
\usepackage{graphicx}
\graphicspath{{./}{figures/}}
\usepackage{float}
\usepackage{booktabs}
\usepackage{longtable}
\usepackage{array}
\usepackage[hypertexnames=false]{hyperref}
\usepackage[font=small,labelfont=bf,skip=6pt]{caption}
\usepackage{placeins}
\usepackage{wrapfig}

\allowdisplaybreaks
\newtheorem{theorem}{Theorem}
\newtheorem{lemma}[theorem]{Lemma}
\newtheorem{proposition}[theorem]{Proposition}
\newtheorem{corollary}[theorem]{Corollary}

\theoremstyle{definition}
\newtheorem{definition}[theorem]{Definition}
\newtheorem{assumption}{Assumption}

\theoremstyle{remark}

\numberwithin{equation}{section}

\usepackage[tight,k-tight]{minitoc}

\mtcsetdepth{parttoc}{3}

\mtcsetfeature{parttoc}{before}{}
\mtcsetfeature{parttoc}{pagestyle}{}
\mtcsetfeature{parttoc}{after}{}

\usepackage{lmodern}
\usepackage{microtype}
\usepackage[dvipsnames,svgnames]{xcolor}
\usepackage[most]{tcolorbox}
\usepackage{hyperref}
 
\colorlet{inkblue}{RoyalBlue!80!black}
\definecolor{softgray}{RGB}{95,95,95}
 
\hypersetup{colorlinks=true, 
            linkcolor=inkblue,     
            citecolor=Peach,
            urlcolor=inkblue}
 
\newcommand{\papertitle}{From Global to Tail Convergence: \\
Dynamics of $\alpha$-Stable Score-Based Generative Models}
\newcommand{\papersubtitle}{}
\newcommand{\paperdate}{October 2026}
 
\newcommand{\aff}[1]{\textsuperscript{#1}}
 
\newcommand{\paperabstract}{
Score-based generative models driven by $\alpha$-stable Lévy processes have been proposed as an alternative to Brownian diffusion models, with empirical evidence suggesting improved sampling and better modeling of heavy-tailed data. Yet theoretical guarantees for these models remain limited. We develop a non-asymptotic convergence analysis for scale-preserving $\alpha$-stable generative dynamics, covering both compactly supported and power-law targets. In $p$-Wasserstein distance, we quantify the effects of early stopping, initialization, score approximation, and time discretization, and show exponential decay of the initialization error. A comparison with Brownian dynamics reveals that global Wasserstein convergence alone does not fundamentally distinguish the two models: both initialization errors decay exponentially with the forward horizon on the common class of heavy-tailed targets. A sharp separation emerges instead in the tails. At any positive generation time, stable dynamics produces power-law tails with index $\alpha$, whereas Brownian dynamics remains light-tailed. For a target with tail index $\nu$, asymptotically correct extreme-quantile scaling is possible precisely when $\alpha=\nu$. These results identify tail-index matching as a principled mechanism behind the ability of stable generative models to capture extreme events.
}
\newcommand{\paperkeywords}{
Score-based generative models $\cdot$ 
$\alpha$-stable Lévy processes $\cdot$ 
Heavy-tailed distributions $\cdot $
Non-asymptotic convergence $\cdot $
Wasserstein bound $\cdot$ 
Tail-index matching
}
 
\newcommand{\makepapertitle}{\thispagestyle{empty}\begingroup
  \noindent{\color{inkblue}\rule{\linewidth}{1.4pt}}\par
  \vspace{5pt}
  \noindent{\sffamily\footnotesize\color{softgray}PREPRINT\hfill\paperdate}\par
  \vspace{22pt}
  \noindent{\fontsize{23}{28}\selectfont\bfseries\color{inkblue}\papertitle\par}
  \vspace{7pt}
  \noindent{\fontsize{15.5}{19}\selectfont\itshape\color{softgray}\papersubtitle\par}
  \vspace{20pt}
  \noindent{\large
    Ianis Hammani\aff{1,2}\qquad
    Claire Boyer\aff{2,3}\par}
  \vspace{8pt}
  \noindent{\small\color{softgray}\aff{1}Ecole Normale Supérieur Paris-Saclay, France\\[1pt]
    \aff{2}Universit\'e Paris-Saclay, CNRS, Inria, Laboratoire de
    math\'ematiques d'Orsay, 91405 Orsay, France\\[1pt]
    \aff{3}Institut Universitaire de France\par}
  \vspace{6pt}
  \noindent{\small\color{softgray}Correspondence:
    \href{mailto:ianis.hammani@ens-paris-saclay.fr}{\texttt{ianis.hammani@ens-paris-saclay.fr}}\par}
  \vspace{16pt}
  \begin{tcolorbox}[enhanced, sharp corners, boxrule=0pt,
      colback=inkblue!4, colframe=inkblue!4,
      borderline west={2.2pt}{0pt}{inkblue},
      left=12pt, right=12pt, top=9pt, bottom=9pt]
    {\sffamily\bfseries\footnotesize\color{inkblue}ABSTRACT}\par
    \vspace{5pt}
    {\small\paperabstract\par
    }
    \vspace{7pt}
    {\footnotesize\sffamily\bfseries\color{inkblue}Keywords}\enspace
    {\footnotesize\paperkeywords}
  \end{tcolorbox}
  \vspace{10pt}
  \endgroup
}

\begin{document}
\pagestyle{plain}
\makepapertitle
\setcounter{tocdepth}{3}

\doparttoc
\faketableofcontents

\section{INTRODUCTION}
\label{sec:intro}

Score-based generative models learn to reverse a noising process that
maps the data distribution to a tractable reference law
\citep{Song2021}. Standard constructions use Brownian motion, based on Gaussian noising transitions.
However, its continuous,
light-tailed paths evolve through local increments, making distant or rare
modes harder to explore and recover
\citep{Yoon2023,Paquet2024LevyProtein}. Similarly, Brownian score-based models may also
underrepresent power-law tails \citep{Pandey2025} and be inadequately
initialized for heavy-tailed targets \citep{Fassina2026}. 
These limitations
motivate considering non-Gaussian noising processes \citep{Nachmani2021,Deasy2021}.

Isotropic $\alpha$-stable noise represents a heavy-tailed alternative, where the parameter $\alpha$ controls the tails
and jump scaling. In particular, $\alpha=2$ recovers the Gaussian case, while $\alpha\in(1,2)$ yields heavy-tailed noise.
\citet{Yoon2023} use an $\alpha$-stable noise in the continuous scale-preserving L\'evy--It\^o
model (LIM), with explicit forward transitions and fractional score. LIM
and the discrete model of \citet{Shariatian2025} have been reported to
accelerate and diversify sampling, improve tail and rare-mode coverage, and
better explore high-dimensional spaces
\citep{Yoon2023,Shariatian2025,Popov2025,Paquet2024LevyProtein}. Why stable
noise helps remains only partially understood theoretically: for the discrete model in the finite-sample regime,
\citet{cherkaoui2026heavytails} show that its initialization gain can be
offset by a harder estimation problem in the reverse dynamics.

In continuous time, the main difficulty lies in constructing a tractable backward dynamics. For jump processes, the exact time reversal is generally no longer an additive Lévy SDE
\citep{ConfortiLeonard2022,Qi2026}, making it impractical for generative modeling. 
\citet{Popov2025} circumvent this difficulty by introducing a family of backward SDEs with additive Lévy noise that, while not pathwise reversals, exactly reproduce the forward marginal flow in reverse.
 
Non-asymptotic convergence guarantees are well developed for Brownian diffusion models \citep{Conforti2025KL,Strasman2024,Gao2025}, while general Lévy-type denoising Markov models have recently been studied through a unified generator framework \citep{Ren2025}. In contrast, quantitative guarantees for continuous $\alpha$-stable score-based models remain largely unexplored, particularly for heavy-tailed data.

\paragraph{Contributions.} We develop a convergence theory for continuous $\alpha$-stable generative models and compare their global and tail behavior with Brownian dynamics.
Our main contributions are:
\begin{itemize}
\item Building on a fractional Tweedie identity, we relate reverse-drift regularity to the Lipschitz modulus of the optimal denoiser and establish the required control for both compactly supported and power-law targets.

\item For $1\le p<\alpha$, we establish a non-asymptotic $W_p$ bound for the practical $\alpha$-stable sampler, separating early stopping, initialization, score approximation, and discretization errors. For power-law targets, we further show that the Brownian initialization error also decays exponentially despite the loss of the contractive coupling available in the stable case. Thus, global Wasserstein convergence, established so far, alone does not fundamentally separate the two dynamics.

\item A sharp distinction emerges in the tails. For exact-score dynamics initialized at the reference laws, stable outputs have power-law tails with index $\alpha$, whereas Brownian outputs are light-tailed at every positive generation time. For a target with tail index $\nu$, the stable extreme-quantile ratio vanishes for $\alpha>\nu$, diverges for $\alpha<\nu$, and converges to a finite nonzero limit precisely when $\alpha=\nu$. Hence, tail-index matching is the unique regime allowing asymptotically correct extreme-quantile scaling. 

\end{itemize}

Complete proofs and the protocols for numerical experiments are deferred to the appendices.
\section{\texorpdfstring{$\alpha$}{alpha}-STABLE DYNAMICS}
\label{sec:setup}

In this section, we introduce the scale-preserving forward process and its marginally exact additive reverse dynamics, then we characterize its regularity and dissipativity through the denoiser representation of the fractional score.

\subsection{Forward and Backward Dynamics}
For any $\alpha \in (1, 2]$, an isotropic $\alpha$-stable L\'evy process $(L_t^\alpha)_{t\ge0}$ on $\mathbb{R}^d$ has stationary independent increments and characteristic function $\mathbb{E}[e^{i\xi\cdot L_t^\alpha}]=e^{-t\left\lVert \xi\right\rVert^\alpha}$ for $\xi\in \mathbb{R}^d$. It satisfies the scaling relation $L_t^\alpha\overset{\mathcal L}=t^{1/\alpha}L_1^\alpha$.
Since $L_t^2\overset{\mathcal L}=\sqrt{2}W_t$ for a standard Brownian motion $W$, we henceforth reserve ``$\alpha$-stable'' for $\alpha\in(1,2)$ and treat the conventionally normalized Brownian case separately.

Denote the standard isotropic $\alpha$-stable law of $L_1^\alpha$ by $\mathrm{S}\alpha\mathrm{S}_d(1)$, 
with density $\varphi_\alpha$, which is smooth and satisfies
$\varphi_\alpha(z)\asymp\|z\|^{-(d+\alpha)}$ as $\|z\|\to\infty$
\citep{Applebaum2009,SamorodnitskyTaqqu1994}. Following \citet{Yoon2023},
we consider the scale-preserving stochastic differential equation (SP-SDE)
\begin{equation}\label{eq:spsde}
  \mathrm{d}X_t=-\frac{\beta(t)}{\alpha}X_t\,\mathrm{d}t
  +\beta(t)^{1/\alpha}\,\mathrm{d}L_t^\alpha,
  \qquad X_0\sim\pi_{\mathrm{data}},
\end{equation}
where $\beta>0$ is measurable and locally bounded, and $X_0$ is
independent of $L^\alpha$. For $t>0$, write $\pi_t=\mathrm{Law}(X_t)$
and use the same symbol for its density. Setting
\begin{equation*}
  a(t):=\exp\!\left(-\frac1\alpha\int_0^t\beta(s)\,\mathrm{d}s\right),
  \qquad
  \gamma(t)^\alpha:=1-a(t)^\alpha,
\end{equation*}
its solution satisfies
\[
  X_t\overset{\mathcal L}=a(t)X_0+\gamma(t)Z,
  \qquad Z\sim\mathrm{S}\alpha\mathrm{S}_d(1),
\]
with $Z$ independent of $X_0$. In particular, the SP-SDE preserves the isotropic stable law $\mathrm{S}\alpha\mathrm{S}_d(1)$, and the conditional density of $X_t$ given $X_0=x_0$ takes the explicit form $\pi_t(x\mid x_0)=\gamma(t)^{-d}\varphi_\alpha\big((x-a(t)x_0)/\gamma(t)\big)$. 

The reverse dynamics involves the fractional score
\begin{equation}\label{eq:score_def}
  S_t^{(\alpha)}(x)
  :=\frac{\Delta^{(\alpha-2)/2}\nabla\pi_t(x)}{\pi_t(x)},
\end{equation}
where $\Delta^{(\alpha-2)/2}\nabla$ has Fourier multiplier
$-i\xi\|\xi\|^{\alpha-2}$ under the convention
$\widehat f(\xi)=\int f(x)e^{i\xi\cdot x}\,\mathrm{d}x$;
for $\alpha=2$, this reduces to the classical score $\nabla\log\pi_t$.
Since the exact time reversal of \eqref{eq:spsde} has a state-dependent jump
kernel \citep{ConfortiLeonard2022}, we instead use the marginally exact
additive family of \citet[Theorem~1]{Popov2025}.

\begin{proposition}[Marginally exact reverse dynamics, \citet{Popov2025}]
\label{prop:family}
Fix $\eta>0$ and an early-stopping time $0<\varepsilon<T$, write $\bar\beta(\tau):=\beta(T-\tau)$, and assume $\mathbb{E}\left\lVert X_0\right\rVert<\infty$. Let
\begin{equation}\label{eq:backward_drift}
  \bar b(\tau,x)
  :=\frac{\bar\beta(\tau)}{\alpha}x
  +(1+\eta)\bar\beta(\tau)S_{T-\tau}^{(\alpha)}(x).
\end{equation}
Let $\overleftarrow L^\alpha$ be an isotropic $\alpha$-stable process
independent of $\overleftarrow Z_0$. Under the well-posedness and
Fokker--Planck uniqueness conditions of Appendix~\ref{app:matching_marginals},
the solution of
\begin{equation}\label{eq:backward}
  \mathrm{d}\overleftarrow Z_\tau
  =\bar b(\tau,\overleftarrow Z_\tau)\,\mathrm{d}\tau
  +\bigl(\eta\bar\beta(\tau)\bigr)^{1/\alpha}
   \mathrm{d}\overleftarrow L_\tau^\alpha,
  \qquad
  \overleftarrow Z_0\sim\pi_T,
\end{equation}
satisfies
\(
  \mathrm{Law}(\overleftarrow Z_\tau)=\pi_{T-\tau}
\)
for $0\le\tau\le T-\varepsilon$.
\end{proposition}
The parameter $\eta>0$ is a genuine degree of freedom in the reverse sampler.
Every choice preserves the same marginal evolution, but changes the backward
noise intensity and drift.
The family interpolates between deterministic and stochastic sampling:
$\eta\downarrow0$ yields the probability-flow ODE, while $\eta=1$ matches
the forward noise intensity, paralleling the probability-flow ODE and
reverse-time SDE of Brownian variance-preserving (VP) models \citep{Song2021}.

Equation~\eqref{eq:backward} defines the ideal reverse dynamics. In practice,
three approximations are required to obtain a tractable sampler: (i) the
unknown fractional score $S^{(\alpha)}_t$ is replaced by a learned
approximation $s_\theta$; (ii) the inaccessible initialization $\pi_T$ is
replaced by the reference distribution
$\pi_\infty^\alpha=\mathrm{S}\alpha\mathrm{S}_d(1)$; and (iii) the
continuous-time dynamics is approximated by an exponential integrator on a
uniform grid of $N$ steps, with early stopping at time $\varepsilon>0$. 
We
denote by $\hat\pi_N^{(\beta,\theta,\eta,\varepsilon)}$ the law generated by
the resulting algorithm, which we study in the next sections.

\subsection{Fractional Tweedie Formula}

As in Gaussian score-based models, the fractional score admits a Tweedie-type
representation in terms of the optimal denoiser. The following identity,
essentially due to \citet{Yoon2023}, extends the classical Tweedie formula
\citep{Efron2011} to the $\alpha$-stable setting.
\begin{proposition}[Fractional Tweedie formula, \citet{Yoon2023}]\label{prop:tweedie}
For the SP-SDE \eqref{eq:spsde} with any initial distribution, the denoiser and fractional score are related, at every $t>0$, by
\begin{equation}\label{eq:tweedie}
  m_t(x)
=\frac{x+\alpha\gamma(t)^\alpha S_t^{(\alpha)}(x)}{a(t)},
\end{equation}
where $m_t(x):=\mathbb{E}[X_0\mid X_t=x]$ is the conditional mean,
which is MSE-optimal when $X_0\in L^2$.
\end{proposition}

In what follows, we therefore quantify the regularity of $m_t$ through its global Lipschitz constant
\[
C_t:=\sup_{x\in\mathbb{R}^d}\|\nabla m_t(x)\|_{\mathrm{op}}.
\]

\subsection{Dissipativity}

The denoiser representation \eqref{eq:tweedie} yields dissipativity estimates for the fractional score and one-sided Lipschitz bounds for the reverse drift, which determine the contractivity and stability of the sampling dynamics.
\begin{proposition}[Dissipativity and one-sided Lipschitz estimates]
\label{prop:dissipativity}
Let $t\in(0,T]$, set $\tau=T-t$, and suppose $C_t<\infty$. Then, for all
$x,y\in\mathbb{R}^d$, the score and reverse drift satisfy
\begin{align}
  \langle S_t^{(\alpha)}(y)-S_t^{(\alpha)}(x),x-y\rangle &\ge C_t^{(\alpha)}\left\lVert x-y\right\rVert^2,\notag\\
  \langle\bar b(\tau,x)-\bar b(\tau,y),x-y\rangle &\le \bar L(\tau)\left\lVert x-y\right\rVert^2,\label{eq:osl}
\end{align}
\begin{align*}
  \text{where} \quad C_t^{(\alpha)}&:=\frac{1-a(t)C_t}{\alpha\gamma(t)^\alpha},\\
  \phantom{\text{where}} \quad  \bar L(\tau)&:=\beta(T{-}\tau)\Big(\tfrac1{\alpha}-(1+\eta)\,C^{(\alpha)}_{T-\tau}\Big).
\end{align*}
\end{proposition}
The reverse drift is therefore contractive whenever
$C_t^{(\alpha)}>1/(\alpha(1+\eta))$. We note here that increasing $\eta$ facilitates contractivity.

\section{COMPACTLY SUPPORTED TARGETS}
\label{sec:bound_compactly}

\subsection{Non-asymptotic \texorpdfstring{$W_p$}{Wp} bound}
\label{sec:mainbound}

We bound
\[
W_p\left(\pi_{\mathrm{data}},
\hat\pi_N^{(\beta,\theta,\eta,\varepsilon)}\right)
\]
for the exponential-integrator (EI) sampler defined in
Appendix~\ref{app:ei_scheme}, Equation~\eqref{eq:app_ei}.
It integrates the dissipative linear part in~\eqref{eq:dissip_reform}
and the stochastic convolution exactly, while freezing the score
residual over each step.

\paragraph{Backward processes.}
Choose a uniform grid of $[0,T-\varepsilon]$ as $\tau_k=kh$, for $k=0,\dots,N$, of constant time step $h=(T-\varepsilon)/N$.
To disentangle the different sources of approximation error, we introduce three backward processes: (i) the exact process $\overleftarrow Z$ solves \eqref{eq:backward} with the exact score and is initialized from  $\overleftarrow Z_0\sim\pi_T$; 
(ii) the ideal process  $\overleftarrow Y$ also solves \eqref{eq:backward} but is initialized from $\overleftarrow Y_0\sim\pi_\infty^\alpha$ instead; 
and (iii) the (piecewise-constant) numerical process $\bar X^{\theta,h}$ is the EI approximation driven by the learned score $s_\theta$, initialized from the same distribution $\pi_\infty^\alpha$. 
All three processes use the same backward L\'evy process
$\overleftarrow L^\alpha$; the EI noise is its exponentially weighted
stochastic integral in~\eqref{eq:app_ei}. 
Consequently, the only differences between the three dynamics arise from the initialization, the score approximation, and the time discretization.

\paragraph{Standing assumptions.} The analysis rests on the following assumptions, bearing on the noise schedule, score approximation, temporal score regularity, and spatial denoiser regularity after early stopping.
\begin{assumption}[Noise schedule]\label{ass:schedule}
The schedule $\beta:[0,\infty)\to(0,\infty)$ is continuous, nondecreasing, and satisfies $B(T)=\int_0^T\beta(t)\,\mathrm{d} t\to\infty$ as $T\to\infty$; for a given horizon it is used restricted to $[0,T]$.
\end{assumption}

\begin{assumption}[Score approximation error]\label{ass:score}
The spatial modeling error of the network $s_\theta$ is uniformly bounded in $L^p$ over the backward horizon: there exists $\mathrm{err}\ge0$ with
\begin{equation*}
  \sup_{0\le k\le N-1}\big\lVert s_\theta(\tau_k,\bar X^{\theta,h}_{\tau_k})-S^{(\alpha)}_{T-\tau_k}(\bar X^{\theta,h}_{\tau_k})\big\rVert_{L^p(\Omega)}\le\mathrm{err} .
\end{equation*}
\end{assumption}

\begin{assumption}[Temporal regularity of the score]\label{ass:time}
For the early-stopping threshold $\varepsilon\in(0,T)$ there is $L^{\mathrm{time}}_\varepsilon>0$, possibly depending on $T$, such that, for all $x\in\mathbb{R}^d$ and all $s,t\in[\varepsilon,T]$,
\begin{equation*}
  \left\lVert S_t^{(\alpha)}(x)-S^{(\alpha)}_s(x)\right\rVert\le L^{\mathrm{time}}_\varepsilon\left\lvert t-s\right\rvert\,(1+\left\lVert x\right\rVert).
\end{equation*}
\end{assumption}

\begin{assumption}[Uniform denoiser regularity]\label{ass:denoiser}
For every early-stopping threshold $\varepsilon>0$, $\bar C_\varepsilon:=\sup_{t\ge\varepsilon}C_t<\infty$.
\end{assumption}

The assumptions play distinct roles. Assumption~\ref{ass:schedule} fixes
the time scale, while Assumptions~\ref{ass:score}--\ref{ass:time} control the two approximations made
by the numerical sampler. They are stable analogues of standard
Brownian SGM conditions \citep{Strasman2024,Conforti2025KL}.
Finally, Assumption~\ref{ass:denoiser} controls the spatial regularity of the score after early stopping; it is target-dependent and is verified separately in the regimes studied.

\paragraph{Error decomposition.}
Since
$\mathrm{Law}(\overleftarrow Z_{T-\varepsilon})=\pi_\varepsilon$, one has
\begin{equation}\label{eq:decomp}
W_p\left(\pi_{\mathrm{data}},\hat\pi_N^{(\beta,\theta,\eta,\varepsilon)}\right)
\le \colorbox{yellow!30}{$\mathcal E_0$}
+\colorbox{Peach!30}{$\mathcal E_1$}
+\colorbox{RoyalBlue!30}{$\mathcal E_2$},
\end{equation}
$$ \text{where } \quad
\left\{
\begin{array}{l}
\mathcal E_0=W_p(\pi_{\mathrm{data}},\pi_\varepsilon), \\
\mathcal E_1\!=\!W_p(\mathrm{Law}(\overleftarrow Z_{T-\varepsilon}),
\mathrm{Law}(\overleftarrow Y_{T-\varepsilon})), \\
\mathcal E_2=W_p(\mathrm{Law}(\overleftarrow Y_{T-\varepsilon}),
\mathrm{Law}(\bar X^{\theta,h}_{T-\varepsilon})),
\end{array}
\right.
$$ 
respectively correspond to the stopping,
initialization, and score+discretization errors.

For clarity of exposition, we state hereafter the result for a constant noise schedule. See Appendix~\ref{app:bound} for a  general version, allowing time-dependent schedules.
\begin{theorem}[Constant schedule]\label{cor:rate}
Suppose $\pi_{\mathrm{data}}$ is supported in $\bar B(0,R)$. Fix $\eta>0$, choose $p\in[1,\alpha)$, a constant noise schedule $\beta\equiv\beta_0$, and arbitrary times $0<\varepsilon<T$. Under Assumptions~\ref{ass:schedule}--\ref{ass:time}, for $h$ small enough there exist $\mathsf c_0,\dots,\mathsf c_4>0$ depending only on $(\alpha,\eta,d,p,\beta_0,R,\varepsilon,L^{\mathrm{time}}_\varepsilon)$ such that
\begin{align}\label{eq:rate}
    W_p\bigl(\pi_{\mathrm{data}},\hat\pi_N^{(\beta,\theta,\eta,\varepsilon)}\bigr)
    &\le \colorbox{yellow!30}{$\mathsf c_0\varepsilon^{1/\alpha}$}+
    \colorbox{Peach!30}{$\mathsf c_1 e^{-\rho_\star(T-\varepsilon)}$} \\
    &\qquad +\colorbox{RoyalBlue!30}{$\mathsf c_2(T-\varepsilon)\mathrm{err}
    +\mathsf c_3 \tfrac h\varepsilon$} +\colorbox{RoyalBlue!30}{$\mathsf c_4
    \bigl(1+T+\log(1+\tfrac1\varepsilon)\bigr)h^{1/\alpha}$}, \notag
  \end{align}
with $\rho_\star=\beta_0\frac{1+\eta}{\alpha}$. 
\end{theorem}
Following the Brownian analysis of \citet{Conforti2025KL}, the proof first recasts the
backward drift so that its linear part becomes dissipative. The stationary
fractional score is linear, $S^{(\alpha)}_\infty(x)=-x/\alpha$, and subtracting
it yields
\begin{align}\label{eq:dissip_reform}
  \bar b(\tau,x)={}&-\tfrac{\eta}{\alpha}\,\beta(T-\tau)\,x\\*
  &+(1+\eta)\beta(T-\tau)\bigl[S^{(\alpha)}_{T-\tau}(x)-S^{(\alpha)}_\infty(x)\bigr].\notag
\end{align}
Equation \eqref{eq:backward} is thus a perturbation of a stable
Ornstein--Uhlenbeck process of rate $\eta\beta(T-\tau)/\alpha$ by a term
vanishing at the reference, a structure essential for the analysis of $\mathcal E_2$. 
The same reformulation of the Brownian reverse SDE gives a rate $\beta(T-\tau)/2$, exactly recovered here at $\alpha=2$ and $\eta=1$.

The five terms of (\ref{eq:rate}) have distinct origins. The first is the early-stopping bias, the price of not integrating the dynamics down to $t=0$; it decreases with $\varepsilon$. The second is the initialization error of the backward dynamics, which decays exponentially in the horizon $T$ at the rate $\rho_\star$, itself increasing in $\eta$. The next two control score approximation and drift discretization. The last discretization term, of order $h^{1/\alpha}$, reflects stable fluctuations over an EI step, where the nonlinear drift in~\eqref{eq:dissip_reform} is frozen; its bound also absorbs drift and temporal-freezing contributions. This matches the usual $h^{1/2}$ term in the $W_2$ analysis of the Brownian dynamics by \citet{Strasman2024}.

Decreasing $\varepsilon$ therefore reduces the first term but inflates the last three, through the singular behavior of the score near forward time $t=0$. The restriction $p<\alpha$ is required
for the stable stationary law to have a finite $p$-th moment. 

\subsection{Denoiser regularity}
\label{sec:denoiser_decay}

The proof of Theorem~\ref{cor:rate} uses Assumption~\ref{ass:denoiser}, which the following proposition verifies for compactly supported targets. It also shows that the denoiser Lipschitz constant vanishes asymptotically.
\begin{proposition}[Decay of $C_t$]\label{prop:ct_decay}
If $\pi_{\mathrm{data}}$ is supported in $\bar B(0,R)$, then with $\Lambda:=\sup_u\left\lVert \nabla^2\log \varphi_\alpha(u)\right\rVert<\infty$,
\begin{equation}\label{eq:ct_decay}
  C_t\le\frac{d\,\Lambda\,R^2\,a(t)}{\gamma(t)^2}\xrightarrow[\;t\to\infty\;]{}0 .
\end{equation}
\end{proposition}
Consequently, Assumption~\ref{ass:denoiser} is satisfied for compactly supported targets. Compact support also ensures the well-posedness conditions required by Proposition~\ref{prop:family}, so that the convergence guarantee of Theorem~\ref{cor:rate} applies for any early-stopping time $0<\varepsilon<T$. Moreover, since $C_t\to0$, the dissipativity constant satisfies
\(
C_t^{(\alpha)}\longrightarrow \frac{1}{\alpha}.
\)
The reverse drift is therefore contractive whenever the corresponding forward time is sufficiently large. Any possible lack of contractivity is thus confined to a finite transient regime, whose contribution can be absorbed into the constants of the convergence bound; see Appendix~\ref{app:bound}.
Section~\ref{sec:heavy} shows that the same asymptotic decay persists beyond compact support, for a class of heavy-tailed target distributions.

\section{HEAVY-TAILED TARGETS}
\label{sec:heavy}

We now turn to heavy-tailed target distributions, where compact-support arguments no longer apply. Such distributions naturally arise in applications where rare but extreme events carry a significant fraction of the information, including financial and insurance losses, network traffic, and climate records \citep{Resnick2007, Pandey2025}.
Our next result shows that the convergence guarantees established in the previous section extend beyond compactly supported targets to a broad class of power-law distributions, as specified by the following assumption. 
\begin{assumption}[Heavy-tailed target]\label{ass:heavy} The target
$\pi_{\mathrm{data}}$ admits a bounded, positive and continuous density with an asymptotically radial power-law tail of index $\nu>1$:
\begin{equation}\label{eq:a5}
  \pi_{\mathrm{data}}(x)=\bigl(c_{\pi_{\mathrm{data}}}+o(1)\bigr)\left\lVert x\right\rVert^{-(d+\nu)}
\end{equation}
when $\left\lVert x\right\rVert\to\infty$ uniformly in the direction $x/\left\lVert x\right\rVert$.
\end{assumption}

In order to extend the convergence analysis to heavy-tailed targets, the main
additional difficulty is to control the denoiser regularity without relying on bounded
support. Under Assumption~\ref{ass:heavy}, its value at the origin is
controlled by the finite first moment,
\[
\|m_t(0)\|
\le
M_1:=\mathbb E[\|X_0\|]
<\infty.
\]
The next result characterizes the denoiser regularity for heavy-tailed targets.
\begin{proposition}[Denoiser regularity]
\label{prop:heavy_denoiser}
Under Assumptions~\ref{ass:schedule} and~\ref{ass:heavy}, the following hold.
\begin{itemize}\itemsep2pt
\item[(i)] If $1<\nu<\alpha$, then
$C_t\ge a(t)^{-1}\xrightarrow[t\to\infty]{}\infty.$
\item[(ii)] If $\alpha\le\nu<1+\alpha$, Assumption~\ref{ass:heavy} alone does not ensure $C_t<\infty$: for each such $\nu$, there is a smooth, positive density in dimension one satisfying this assumption for which $C_t=\infty$ for every $t>0$.

\item[(iii)] If $\nu\ge1+\alpha$, then for every $\varepsilon>0$ there exists $c_\varepsilon<\infty$, independent of $t$, such that
$C_t\le c_\varepsilon a(t)$ for all $t\ge\varepsilon$.
Consequently, $C_t\to0$ and Assumption~\ref{ass:denoiser} holds.
\end{itemize}
\end{proposition}

For $\nu\ge1+\alpha$, we can therefore apply the convergence analysis of
Theorem~\ref{cor:rate} without imposing any additional regularity assumption
on the denoiser.

\begin{theorem}[$W_p$-bound for heavy-tailed targets]
\label{thm:main_ht}
Let $\pi_{\mathrm{data}}$ satisfy Assumption~\ref{ass:heavy} with
$\nu\ge1+\alpha$. Under the remaining hypotheses of Theorem~\ref{cor:rate},
the same bound holds. In the drift estimates, the support-radius control of
$\lVert m_t(0)\rVert$ is replaced by $M_1$. The resulting
constants may additionally depend on
$\mu_p=\mathbb E[\|X_0\|^p]^{1/p}$ and on the quantitative constants in
Assumption~\ref{ass:heavy}.
\end{theorem}

\section{COMPARISON WITH BROWNIAN DYNAMICS}

We now compare \(\alpha\)-stable and Brownian dynamics on heavy-tailed targets from two complementary perspectives: global convergence of the generated distribution (which primarily reflects reconstruction of the bulk) and preservation of its tail behavior.

\subsection{Global comparison}\label{sec:global_comparison}

Our goal is to determine whether $\alpha$-stable and Brownian dynamics exhibit a meaningful separation at the level of the generated distribution as a whole, in $W_p$. For the Brownian model, we consider a VP-SDE with a general noise schedule and its associated reverse-time dynamics. We denote its forward marginals by $\pi_t^{(2)}$ and its stationary law by $\pi_\infty^{(2)}=\mathcal N(0,I_d)$. Since the stable and Brownian formulations share the same algebraic structure, we use $a(t)$, $\beta(t)$, $\gamma(t)$, and $C_t$ for the corresponding quantities in both settings. When a distinction is needed, we make it explicit, e.g.\ writing $\mathcal E_1^{(\alpha)}$ for the stable initialization error and $\mathcal E_1^{(2)}$ for its Brownian counterpart.

Following the error decomposition of Section~\ref{sec:bound_compactly}, we discuss the two dynamics in terms of denoiser regularity, mixing of the reverse dynamics, discretization, and early-stopping error.

\paragraph{Denoiser regularity.}
We first compare the regularity of the denoiser, as measured by its Lipschitz constant $C_t$.
Under stable noise, Propositions~\ref{prop:ct_decay}
and~\ref{prop:heavy_denoiser} give $C_t\to0$ for compactly supported
targets and power-law targets with $\nu\ge1+\alpha$. 
In the Brownian case, the behavior actually depends on the target type.
\begin{proposition}[Brownian denoiser regularity]\label{prop:brown_ct_divergence}
Under Assumption~\ref{ass:schedule}, the Brownian VP denoiser satisfies:
\begin{itemize}
  \item If $\pi_{\mathrm{data}}$ is supported on $\bar B(0,R)$, then
  $C_t\le a(t)R^2/\gamma(t)^2\to0$ as $t\to\infty$.
  \item Under Assumption~\ref{ass:heavy},
  $\nabla m_t(x)\to a(t)^{-1}I_d$ as $\lVert x\rVert\to\infty$ for every $t>0$;
  hence $C_t\ge a(t)^{-1}\to\infty$ as $t\to\infty$.
\end{itemize}
\end{proposition}
Thus, for heavy-tailed targets with $\nu\ge1+\alpha$, denoiser regularity separates the two models: stable noise asymptotically flattens the denoiser, yielding eventual contractivity, whereas Brownian noise produces the opposite behavior. 

\begin{wrapfigure}{r}{0.5\linewidth} 
\includegraphics[width=0.8\linewidth]{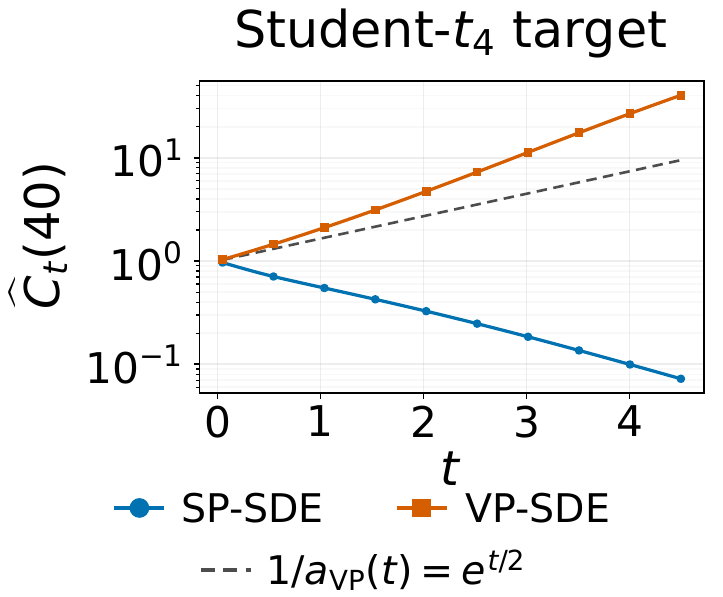}
\caption{Illustration of $\widehat C_t(40)$ on a Student-$t_4$ target for the stable ($\alpha=1.5$) and VP models.}
\label{fig:c_t}
\end{wrapfigure}
Figure~\ref{fig:c_t} illustrates this dichotomy on a Student-$t_4$ target with $\beta\equiv1$ in both models
through $\widehat C_t(40)$, the maximum numerically evaluated $|m_t^{\prime}|$ on a grid of $[-40,40]$: 
it decreases under stable noise ($\alpha=1.5$) and grows under Brownian noise, consistently with the lower bound $C_t\ge1/a(t)=e^{t/2}$; see Appendix~\ref{app:exp:ct} for details.

Crucially, this regularity gap does not imply a separation in global convergence. As we show next, the loss of Brownian contractivity makes the synchronous-coupling argument ineffective, but does not prevent the Brownian reverse dynamics from mixing.

\paragraph{Mixing.}
Here, mixing error denotes the propagated initialization error. For stable dynamics, the one-sided Lipschitz bound yields a contractive
synchronous coupling, as
$\mathcal{E}_1^{(\alpha)}\le W_p(\pi_{\mathrm{data}},\pi_\infty^\alpha)
e^{-\Phi^{(\alpha)}(T)}$ where 
\begin{equation*}
  \Phi^{(\alpha)}(T):=\frac{1+\eta}{\alpha}\int_\varepsilon^T\frac{\beta(t)\bigl(1-a(t)C_t\bigr)}{\gamma(t)^\alpha}\,\mathrm{d} t.
\end{equation*}
Combined with the decay of $C_t$, this gives an
explicit exponential rate for the initialization error.
\begin{corollary}[Stable mixing error]\label{cor:mixing}
Under Assumptions~\ref{ass:schedule} and~\ref{ass:heavy} with $\nu\ge1+\alpha$, for every $p\in[1,\alpha)$ stable noise mixes exponentially in $B$,
\begin{equation*}
  \mathcal{E}_1^{(\alpha)}\le C\,W_p(\pi_{\mathrm{data}},\pi_\infty^\alpha)\,e^{-\frac{1+\eta}{\alpha}(B(T)-B(\varepsilon))}\xrightarrow[T\to\infty]{}0,
\end{equation*}
where $C<\infty$ is independent of $T$.
\end{corollary}
For Brownian dynamics, Proposition~\ref{prop:brown_ct_divergence} gives
$a(t)C_t\ge1$, so the analogous synchronous-coupling bound is vacuous.
Nevertheless, a direct analysis of the initialization error still yields
exponential mixing.
\begin{theorem}[Brownian mixing error]\label{thm:gauss_init}
Under Assumption~\ref{ass:schedule}, suppose $\pi_{\mathrm{data}}$ satisfies Assumption~\ref{ass:heavy} with $\nu>2$. Let $\mu_{T,\varepsilon}^{(2)}$ denote the law at reverse time $T-\varepsilon$ of the Brownian reverse SDE with the exact score, initialized at the reference $\pi_\infty^{(2)}$ in place of the forward marginal $\pi_T^{(2)}$. For $p\in[1,\nu)$ and any $q\in(p,\nu)$ there is $C_q<\infty$, independent of $\varepsilon$, such that for $T$ large enough
\begin{equation*}
  \mathcal{E}_1^{(2)}=W_p(\mu_{T,\varepsilon}^{(2)},\pi_\varepsilon^{(2)})\le C_q\,e^{-\frac{q-p}{2pq}B(T)}\xrightarrow[T\to\infty]{}0.
\end{equation*}
\end{theorem}
Figure~\ref{fig:exp2} illustrates the decay of the initialization error on a Student-$t_4$ target. For each model, we compare samplers initialized from the forward marginal $\pi_T$ and from the stationary reference, using the same numerically evaluated score, step size $h=0.0125$, and stopping time $\varepsilon=0.05$. The resulting $W_{1.1}$ and $W_{1.4}$ discrepancies decrease with $T$ for both stable and Brownian dynamics, eventually approaching the finite-sample baseline obtained from two independent samples from $\pi_\varepsilon$. This is consistent with the predicted loss of initialization dependence in both models, including the non-contractive Brownian regime; see Appendix~\ref{app:exp:init}.

Thus, on the common class of targets with $\nu\ge1+\alpha$, both initialization errors decay
exponentially in $W_p$, despite fundamentally different mechanisms:
contractive coupling for stable dynamics and a direct mixing argument for
Brownian dynamics. Global initialization error therefore does not separate
the two models.
\begin{figure}[!htbp]
    \centering\includegraphics[width=0.7\linewidth]{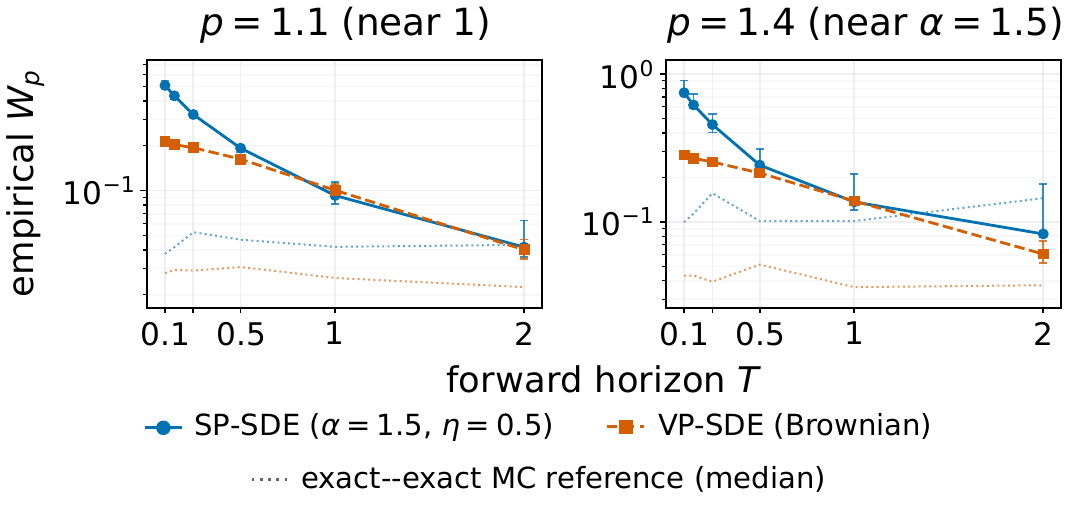}
  \caption{
  Empirical $W_{1.1}$ and $W_{1.4}$ discrepancies between samplers initialized from the stationary reference and from the exact forward marginal $\pi_T$, for stable and Brownian dynamics on a Student-$t_4$ target. Both discrepancies decrease with $T$ toward the finite-sample baseline (dotted), estimated from independent samples from $\pi_\varepsilon$. Points and bars report medians and $90\%$ bootstrap intervals over $12$ runs; see Appendix~\ref{app:exp:init}.}
  \label{fig:exp2}
\end{figure}

\paragraph{Early stopping and discretization.}
Having found no qualitative separation in the initialization error, we finally compare the scaling of the remaining two terms. For the stable sampler, both directly reflect the self-similarity of the driving process:
$\gamma(\varepsilon)\asymp(\beta_0\varepsilon)^{1/\alpha}$ and
$L_h^\alpha\overset{\mathcal L}=h^{1/\alpha}L_1^\alpha$ yield the respective upper bounds
$\varepsilon^{1/\alpha}$ and $h^{1/\alpha}$.  
For $\alpha<2$, $L^\alpha$ is a pure-jump process and $1/\alpha>1/2$ is its self-similarity index. It expresses superdiffusive scaling, in a Lévy-flight sense: heavy-tailed jumps permit occasional long-range moves and thereby provide a mechanism for broader exploration of a large space \citep{viswanathan2008levy}. At $\alpha=2$, both exponents reduce to the Brownian scaling $1/2$.

A rigorous comparison of these bounds with Brownian dynamics is not currently
available under the heavy-tailed assumptions considered here. In particular,
\citet{Strasman2024} obtain an $h^{1/2}$ term for the same
exponential-integrator construction, but under global score regularity and
strong log-concavity and without the early-stopping error. The shared
construction nevertheless suggests $\varepsilon^{1/2}$ and $h^{1/2}$ as the
corresponding Brownian scalings. If these rates extend to the present
heavy-tailed setting, the stable bounds would exhibit more favorable exponents
for $\alpha<2$. This comparison is therefore suggestive rather than
conclusive: it identifies a possible scaling mechanism consistent with the
empirical sampling gains reported for Lévy-based generative models
\citep{Yoon2023,Paquet2024LevyProtein}.

\subsection{Tail behavior}
\label{sec:tce}

Global convergence in $W_p$ does not by itself guarantee that the generated distribution reproduces the extreme behavior of the target. In particular, it does not determine whether the generated law has the correct tail index or assigns the correct mass to extreme events. We therefore complement the Wasserstein analysis with two tail-sensitive quantities motivated by classical Extreme Value Theory \citep[EVT,][]{deHaanFerreira2006,GhoshResnick2009,Resnick2007}. The first, based on the mean-excess function, identifies the asymptotic tail regime; the second compares generated and target quantiles at increasingly extreme probability levels.

\paragraph{Tail regime.}
\label{sec:tail_regime}
A target-independent criterion, called limiting relative mean excess, distinguishes power-law tails from lighter tails, such as Gaussian tails.
\begin{definition}[Limiting relative mean excess]\label{def:tce}
For a probability measure $\mu$ on $\mathbb{R}^d$, let $F_\mu(r)=\mu(\{x:\|x\|\le r\})$  denote its radial CDF.  For a level $\kappa\in(0,1)$, let $q_\kappa(\mu)=\inf\{r:F_\mu(r)\ge\kappa\}$ be its radial quantile of level $\kappa$. Then, for $u$ such that $1-F_\mu(u)>0$, introduce the mean-excess function $e_\mu(u)=\mathbb{E}_\mu[\left\lVert X\right\rVert-u\mid\left\lVert X\right\rVert>u]$ so that the \emph{limiting relative mean excess} is given by
\begin{equation*}
  \Gamma(\mu):=\lim_{u\to\infty}\frac{e_\mu(u)}{u},
\end{equation*}
whenever the limit exists.
\end{definition}

Note that if $q_\kappa(\mu)\to\infty$ as $\kappa\to1$, then the limiting relative mean excess can be re-expressed as
$\Gamma(\mu)=\lim_{\kappa\to1}e_\mu(q_\kappa(\mu))/q_\kappa(\mu)$. This condition is satisfied by the distributions considered here.

We also introduce $\overline G_\mu(\cdot):= 1-F_\mu(\cdot)$, the radial survival function of $\mu$.  In particular, Assumption~\ref{ass:heavy} implies that $\overline G_{\pi_{\mathrm{data}}}(r)\sim
c'_{\pi_{\mathrm{data}}}r^{-\nu}$ as $r\to\infty$, with
$c'_{\pi_{\mathrm{data}}}=\omega_{d-1}c_{\pi_{\mathrm{data}}}/\nu$, where $\omega_{d-1}$ is the surface area of the unit sphere. 
\begin{proposition}[Tail-regime dichotomy]\label{prop:tce} \hfill\\
\textit{(i)} If $\overline G_\mu(r)\sim c_\mu r^{-\vartheta}$ for some
$c_\mu>0$ and $\vartheta>1$, then $e_\mu(u)\sim u/(\vartheta-1)$ and
\begin{equation*}
  \Gamma(\mu)=\frac{1}{\vartheta-1}>0.
\end{equation*}
\textit{(ii)} If $X\sim\mu$ decomposes as $X=G+R$ with $G$ a nondegenerate isotropic Gaussian vector and $\left\lVert R\right\rVert\le K$ almost surely, then $e_\mu(u)=\mathcal O(1)$ and $\Gamma(\mu)=0$.
\end{proposition}
Thus, $\Gamma(\mu)$ identifies a qualitative difference in tail regime, but does not measure whether the generated distribution reproduces the target at a given extreme probability level. To capture this second aspect, we compare their extreme quantiles directly.

\paragraph{Extreme-quantile matching.}
\label{sec:emr}

For a level $\kappa\in(0,1)$, define the \emph{extreme-quantile ratio} of $\mu$ with respect to a target reference $\pi$ by
\begin{equation*}
  \mathcal M_\mu(\kappa):=\frac{q_\kappa(\mu)}{q_\kappa(\pi)}.
\end{equation*}
Thus, $\mathcal M_\mu(\kappa)\approx1$ means that generated and target distributions reach comparable radial scales at probability level $\kappa$. For power-law tails, its asymptotic behavior follows directly from the corresponding tail indices and constants.

\begin{lemma}[Extreme-quantile ratio under power-law tails]\label{lem:emr_asymp}
Consider a distribution $\mu$ such that $\overline G_\mu(r)\sim c_\mu r^{-\alpha}$ for $c_\mu,\alpha>0$. Then, its quantile function satisfies $q_\kappa(\mu)\sim c_\mu^{1/\alpha}(1-\kappa)^{-1/\alpha}$ as $\kappa\uparrow1$.
Consider in addition a target distribution $\pi$ such that $\overline G_\pi(r)\sim c_\pi r^{-\nu}$ for
$c_\pi,\nu>0$. Then, the extreme-quantile ratio of $\mu$ w.r.t.\ $\pi$ satisfies
\begin{equation}\label{eq:emr_asymp}
  \mathcal M_\mu(\kappa)\sim c_\mu^{1/\alpha}\,c_\pi^{-1/\nu}\,(1-\kappa)^{\frac1\nu-\frac1\alpha}\qquad(\kappa\uparrow1).
\end{equation}
\end{lemma}
Lemma~\ref{lem:emr_asymp} separates tail mismatch as follows. A discrepancy between the tail indices produces a polynomial divergence or collapse of $\mathcal M_\mu(\kappa)$ as $\kappa\uparrow1$. When the indices match ($\alpha=\nu$) this polynomial discrepancy disappears and
\(
\mathcal M_\mu(\kappa)
\longrightarrow
\left({c_\mu}/{c_\pi}\right)^{1/\alpha}.
\)

\paragraph{Stable/Brownian separation in the tails.}
Fix $0<\varepsilon<T<\infty$. Write
$\mu_{\tau}^{(\alpha)}:=\mathrm{Law}(\overleftarrow Y_{\tau})$ for the stable
exact-score reverse law initialized at $\pi_\infty^\alpha$, and define
$\mu_{\tau}^{(2)}$ analogously for the Brownian reverse SDE initialized at
$\pi_\infty^{(2)}$. The symbol $\mu$ denotes a generated law at reverse time
$\tau\in(0,T-\varepsilon]$, whereas $\pi$ denotes a forward marginal.
In particular, $\mu_{T-\varepsilon}^{(2)}=\mu_{T,\varepsilon}^{(2)}$
in Theorem~\ref{thm:gauss_init}.
\begin{proposition}[Stable/Brownian tail separation]
\label{prop:tail_separation}
Suppose Assumption~\ref{ass:heavy} holds, so that
\(
\overline G_{\pi_{\mathrm{data}}}(r)
\sim
c'_{\pi_{\mathrm{data}}}r^{-\nu}.
\)
At every generation time $\tau\in(0,T-\varepsilon]$, the Brownian output $\mu_{\tau}^{(2)}$ has sub-Gaussian radial tails, whereas the $\alpha$-stable output satisfies
\[
\overline G_{\mu_{\tau}^{(\alpha)}}(r)
\sim
c_{\tau} r^{-\alpha},
\qquad
c_{\tau}\in(0,\infty).
\]
In terms of relative mean excess, this gives
\begin{equation*}
\Gamma(\mu_{\tau}^{(\alpha)})
=
\frac{1}{\alpha-1} \,
> \, 0
=
\Gamma(\mu_{\tau}^{(2)}).
\end{equation*}
For the extreme-quantile ratio,
\(
\mathcal M_{\mu_{\tau}^{(2)}}(\kappa)
\xrightarrow[\kappa \to 1]{} 0,
\)
while
\begin{equation*}
\mathcal M_{\mu_{\tau}^{(\alpha)}}(\kappa)
\underset{\kappa\to 1}{\sim}
c_{\tau}^{1/\alpha}
(c'_{\pi_{\mathrm{data}}})^{-1/\nu}
(1-\kappa)^{\frac{1}{\nu}-\frac{1}{\alpha}}.
\end{equation*}
\end{proposition}
The proof decomposes the stable output into an exactly $\alpha$-stable leading
term and a uniformly sublinear remainder, which is negligible in the tail. The
tail constant is $c_{\tau}=C_{d,\alpha}\rho_{\tau}^\alpha$, where $C_{d,\alpha}$ is the
radial tail constant of $\mathrm{S}\alpha\mathrm{S}_d(1)$ and $\rho_{\tau}$ is the scale of
the leading stable component (see Appendix~\ref{app:tce} for details).

The two diagnostics reveal complementary aspects of tail reconstruction:
$\Gamma$ distinguishes light from power-law tails, whereas $\mathcal M_\mu$
measures the scale of extreme quantiles. Their joint behavior is summarized
in Table~\ref{tab:tail_regimes}.
\begin{table}[!htbp]
\centering
\caption{Tail regimes under Assumption~\ref{ass:heavy}, at fixed
$0<\varepsilon<T<\infty$, as in Proposition~\ref{prop:tail_separation}.}
\label{tab:tail_regimes}
\renewcommand{\arraystretch}{1.35}
\begin{tabular}{lcc}
\toprule
 & $\Gamma(\mu_{\tau})$
 & $\mathcal M_{\mu_{\tau}}(\kappa)$ as $\kappa\uparrow1$ \\
\midrule
Brownian
& $0$
& $0$ \\
Stable, $\alpha>\nu$
& $\frac{1}{\alpha-1}$
& $0$ \\
Stable, $\alpha=\nu$
& $\frac{1}{\nu-1}$
& $\left(c_{\tau}/c'_{\pi_{\mathrm{data}}}\right)^{1/\nu}$ \\
Stable, $\alpha<\nu$
& $\frac{1}{\alpha-1}$
& $+\infty$ \\
\bottomrule
\end{tabular}
\end{table}

Thus, producing heavy tails is not sufficient for correctly calibrated
extremes. 
Indeed, if $\alpha>\nu$, the generated tail is lighter than the target one and the extreme-quantile ratio vanishes. 
If $\alpha<\nu$, the generated tail is heavier so that the extreme-quantile ratio diverges. 
Tail-index matching, $\alpha=\nu$, is therefore the unique stable regime
allowing asymptotically correct extreme-quantile scaling; exact matching
further requires the tail constants to agree.

Figure~\ref{fig:numerical_tails} illustrates the three tail-reconstruction regimes predicted by Proposition~\ref{prop:tail_separation} for Student-$t_\nu$ targets. With stable noise ($\alpha=1.5$), as $\kappa\uparrow1$ the extreme-quantile ratio decreases for $\nu<\alpha$, remains approximately constant for $\nu=\alpha$, and increases for $\nu>\alpha$, consistently with the predicted scaling
$(1-\kappa)^{1/\nu-1/\alpha}$. In the index-matched case $\nu=\alpha=1.5$, the Brownian ratio instead decreases toward zero, 
illustrating its failure to reproduce the target tail scale. 
The off-match experiments reproduce the
predicted qualitative trends, although the available sample size does not
reliably resolve their asymptotic exponents.

\begin{figure}[!htbp]
\centering
\includegraphics[width=0.35\linewidth]{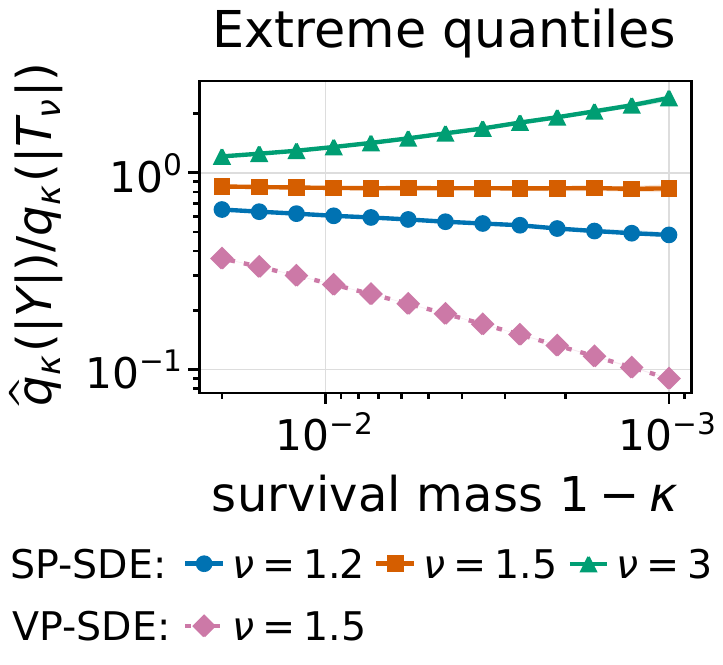}
 \caption{Generated-to-target extreme-quantile ratios for
  Student-$t_\nu$ targets. Stable samplers use
  $(\alpha,\eta)=(1.5,0.5)$; the VP baseline uses $\nu=1.5$.
  Curves show medians over $12$ runs; bands give the inter-seed $16$--$84\%$ range.
  See Appendix~\ref{app:exp:tails}.} 
\label{fig:numerical_tails}
\end{figure}

Overall, the exact-score dynamics initialized at the reference laws differ
in their tails: Brownian outputs remain light-tailed at finite generation
time, while stable outputs inherit the noise tail index. This distinction
is not determined by global $W_p$ convergence. Tail-index matching
therefore provides a possible mechanism behind the improved tail coverage
observed empirically with L\'evy-based generative models
\citep{Yoon2023,Pandey2025,Paquet2024LevyProtein,Shariatian2025}.

\FloatBarrier
\section*{CONCLUSION}
Our results distinguish global accuracy from tail reconstruction in $\alpha$-stable generative models. Our tail analysis includes the index-matched regime $\nu=\alpha$, whereas our global convergence guarantees require stronger tail assumptions. Tail-index matching explains when stable dynamics can reproduce the correct asymptotic scale of extreme events, yet the role of tail constants, anisotropy, and multivariate dependence remains largely unexplored. It is also unclear whether Lévy noise provides advantages beyond tail reconstruction, for instance through faster exploration of multimodal landscapes or improved finite-step sampling.  Finally, extending the tail analysis to learned scores is essential: estimation errors in low-density regions may interact strongly with the jump structure precisely where the stable dynamics is meant to improve tail coverage. 

\section*{AI Use Statement}

In this work, we used generative AI tools, including ChatGPT, to assist with English writing and editing, literature checks, mathematical proofreading, and the implementation and verification of some numerical experiments. Generative AI tools were not used to define the research questions, formulate the main methodological ideas, or determine the scientific conclusions of the work.

All AI-assisted mathematical arguments were independently checked by the authors, and no proof or theoretical claim was accepted without verification. AI-assisted code and numerical experiments were reviewed, tested, and validated by the authors, including by checking their consistency with the underlying mathematical formulation and the reported results. Literature-related suggestions were verified against the original sources before being used or cited.

We take full responsibility for the final content of this work, including its text, mathematical claims, numerical results, and conclusions.

\bibliographystyle{unsrtnat}
\bibliography{ref,software_refs}

\begin{thebibliography}{41}
\providecommand{\natexlab}[1]{#1}
\providecommand{\url}[1]{\texttt{#1}}
\expandafter\ifx\csname urlstyle\endcsname\relax
  \providecommand{\doi}[1]{doi: #1}\else
  \providecommand{\doi}{doi: \begingroup \urlstyle{rm}\Url}\fi

\bibitem[Song et~al.(2021)Song, Sohl-Dickstein, Kingma, Kumar, Ermon, and
  Poole]{Song2021}
Yang Song, Jascha Sohl-Dickstein, Diederik~Pieter Kingma, Abhishek Kumar,
  Stefano Ermon, and Ben Poole.
\newblock Score-based generative modeling through stochastic differential
  equations.
\newblock In \emph{International Conference on Learning Representations}, 2021.
\newblock URL \url{https://openreview.net/forum?id=PxTIG12RRHS}.

\bibitem[Yoon et~al.(2023)Yoon, Park, Kim, and Lim]{Yoon2023}
Eunbi Yoon, Keehun Park, Sungwoong Kim, and Sungbin Lim.
\newblock Score-based generative models with {L}{\'e}vy processes.
\newblock In \emph{Advances in Neural Information Processing Systems},
  volume~36, pages 40694--40707, 2023.
\newblock \doi{10.52202/075280-1772}.
\newblock URL \url{https://doi.org/10.52202/075280-1772}.

\bibitem[Paquet et~al.(2024)Paquet, Soleymani, Viktor, and
  Michalowski]{Paquet2024LevyProtein}
Eric Paquet, Farzan Soleymani, Herna~Lydia Viktor, and Wojtek Michalowski.
\newblock Annealed fractional {L}{\'e}vy--{I}t{\^o} diffusion models for
  protein generation.
\newblock \emph{Computational and Structural Biotechnology Journal},
  23:\penalty0 1641--1653, 2024.
\newblock \doi{10.1016/j.csbj.2024.04.009}.

\bibitem[Pandey et~al.(2025)Pandey, Pathak, Xu, Mandt, Pritchard, Vahdat, and
  Mardani]{Pandey2025}
Kushagra Pandey, Jaideep Pathak, Yilun Xu, Stephan Mandt, Michael Pritchard,
  Arash Vahdat, and Morteza Mardani.
\newblock Heavy-tailed diffusion models.
\newblock In \emph{International Conference on Learning Representations}, 2025.
\newblock URL \url{https://arxiv.org/abs/2410.14171}.

\bibitem[Fassina et~al.(2026)Fassina, Cardoso, Le~Corff, and
  Romary]{Fassina2026}
Tiziano Fassina, Gabriel Cardoso, Sylvain Le~Corff, and Thomas Romary.
\newblock Generalizing score-based generative models for heavy-tailed
  distributions, 2026.
\newblock URL \url{https://arxiv.org/abs/2603.00772}.

\bibitem[Nachmani et~al.(2021)Nachmani, San~Roman, and Wolf]{Nachmani2021}
Eliya Nachmani, Robin San~Roman, and Lior Wolf.
\newblock Non-gaussian denoising diffusion models, 2021.
\newblock URL \url{https://arxiv.org/abs/2106.07582}.

\bibitem[Deasy et~al.(2021)Deasy, Simidjievski, and Li{\`o}]{Deasy2021}
Jacob Deasy, Nikola Simidjievski, and Pietro Li{\`o}.
\newblock Heavy-tailed denoising score matching, 2021.
\newblock URL \url{https://arxiv.org/abs/2112.09788}.

\bibitem[Shariatian et~al.(2025)Shariatian, {{\c{S}}im{\c{s}}ekli}, and
  Durmus]{Shariatian2025}
Dario Shariatian, Umut {{\c{S}}im{\c{s}}ekli}, and Alain Durmus.
\newblock Heavy-tailed diffusion with denoising {L}{\'e}vy probabilistic
  models.
\newblock In \emph{International Conference on Learning Representations}, 2025.
\newblock URL \url{https://arxiv.org/abs/2407.18609}.

\bibitem[Popov et~al.(2025)Popov, Yermekova, Sadekova, Khrapov, and
  Kudinov]{Popov2025}
Vadim Popov, Assel Yermekova, Tasnima Sadekova, Artem Khrapov, and Mikhail
  Kudinov.
\newblock Improved sampling algorithms for {L}{\'e}vy--{I}t{\^o} diffusion
  models.
\newblock In \emph{International Conference on Learning Representations}, 2025.
\newblock URL \url{https://openreview.net/forum?id=XxCgeWSTNp}.

\bibitem[Cherkaoui et~al.(2026)Cherkaoui, Halconruy, and
  Ocello]{cherkaoui2026heavytails}
Hamza Cherkaoui, H{\'e}l{\`e}ne Halconruy, and Antonio Ocello.
\newblock Do heavy tails help diffusion? on the subtle trade-off between
  initialization and training, 2026.
\newblock URL \url{https://arxiv.org/abs/2605.13175}.

\bibitem[Conforti and L{\'e}onard(2022)]{ConfortiLeonard2022}
Giovanni Conforti and Christian L{\'e}onard.
\newblock Time reversal of markov processes with jumps under a finite entropy
  condition.
\newblock \emph{Stochastic Processes and their Applications}, 144:\penalty0
  85--124, 2022.
\newblock \doi{10.1016/j.spa.2021.10.002}.

\bibitem[Qi et~al.(2026)Qi, Wang, and Zhang]{Qi2026}
Tianfu Qi, Jun Wang, and Jun Zhang.
\newblock Generator-guided inverse sampling for {L}{\'e}vy-driven generative
  models, 2026.
\newblock URL \url{https://arxiv.org/abs/2608.10384}.

\bibitem[Conforti et~al.(2025)Conforti, Durmus, and
  Gentiloni~Silveri]{Conforti2025KL}
Giovanni Conforti, Alain Durmus, and Marta Gentiloni~Silveri.
\newblock {KL} convergence guarantees for score diffusion models under minimal
  data assumptions.
\newblock \emph{SIAM Journal on Mathematics of Data Science}, 7\penalty0
  (1):\penalty0 86--109, 2025.
\newblock \doi{10.1137/23M1613670}.

\bibitem[Strasman et~al.(2025)Strasman, Ocello, Boyer, Le~Corff, and
  Lemaire]{Strasman2024}
Stanislas Strasman, Antonio Ocello, Claire Boyer, Sylvain Le~Corff, and Vincent
  Lemaire.
\newblock An analysis of the noise schedule for score-based generative models.
\newblock \emph{Transactions on Machine Learning Research}, 2025.
\newblock URL \url{https://openreview.net/forum?id=BlYIPa0Fx1}.

\bibitem[Gao et~al.(2025)Gao, Nguyen, and Zhu]{Gao2025}
Xuefeng Gao, Hoang~M. Nguyen, and Lingjiong Zhu.
\newblock Wasserstein convergence guarantees for a general class of score-based
  generative models.
\newblock \emph{Journal of Machine Learning Research}, 26\penalty0
  (43):\penalty0 1--54, 2025.
\newblock URL \url{https://www.jmlr.org/papers/v26/24-0902.html}.

\bibitem[Ren et~al.(2026)Ren, Rotskoff, and Ying]{Ren2025}
Yinuo Ren, Grant~M. Rotskoff, and Lexing Ying.
\newblock A unified approach to analysis and design of denoising markov models.
\newblock \emph{Journal of Machine Learning Research}, 27\penalty0
  (105):\penalty0 1--69, 2026.
\newblock URL \url{https://www.jmlr.org/papers/v27/25-0693.html}.

\bibitem[Applebaum(2009)]{Applebaum2009}
David Applebaum.
\newblock \emph{{L}{\'e}vy Processes and Stochastic Calculus}, volume 116 of
  \emph{Cambridge Studies in Advanced Mathematics}.
\newblock Cambridge University Press, 2 edition, 2009.
\newblock \doi{10.1017/CBO9780511809781}.

\bibitem[Samorodnitsky and Taqqu(1994)]{SamorodnitskyTaqqu1994}
Gennady Samorodnitsky and Murad~S. Taqqu.
\newblock \emph{Stable Non-Gaussian Random Processes: Stochastic Models with
  Infinite Variance}.
\newblock Chapman \& Hall, 1994.

\bibitem[Efron(2011)]{Efron2011}
Bradley Efron.
\newblock Tweedie's formula and selection bias.
\newblock \emph{Journal of the American Statistical Association}, 106\penalty0
  (496):\penalty0 1602--1614, 2011.
\newblock \doi{10.1198/jasa.2011.tm11181}.

\bibitem[Resnick(2007)]{Resnick2007}
Sidney~I. Resnick.
\newblock \emph{Heavy-Tail Phenomena: Probabilistic and Statistical Modeling}.
\newblock Springer, New York, 2007.
\newblock \doi{10.1007/978-0-387-45024-7}.

\bibitem[Viswanathan et~al.(2008)Viswanathan, Raposo, and
  da~Luz]{viswanathan2008levy}
Gandhimohan~M. Viswanathan, Ernesto~Pessoa Raposo, and Marcos
  Gomes~Eleut{\'e}rio da~Luz.
\newblock {L\'e}vy flights and superdiffusion in the context of biological
  encounters and random searches.
\newblock \emph{Physics of Life Reviews}, 5\penalty0 (3):\penalty0 133--150,
  2008.
\newblock \doi{10.1016/j.plrev.2008.03.002}.
\newblock URL \url{https://doi.org/10.1016/j.plrev.2008.03.002}.

\bibitem[de~Haan and Ferreira(2006)]{deHaanFerreira2006}
Laurens de~Haan and Ana Ferreira.
\newblock \emph{Extreme Value Theory: An Introduction}.
\newblock Springer, New York, 2006.
\newblock \doi{10.1007/0-387-34471-3}.

\bibitem[Ghosh and Resnick(2010)]{GhoshResnick2009}
Souvik Ghosh and Sidney~I. Resnick.
\newblock A discussion on mean excess plots.
\newblock \emph{Stochastic Processes and their Applications}, 120\penalty0
  (8):\penalty0 1492--1517, 2010.
\newblock \doi{10.1016/j.spa.2010.04.002}.

\bibitem[R{\"o}ckner et~al.(2020)R{\"o}ckner, Xie, and
  Zhang]{RocknerXieZhang2020}
Michael R{\"o}ckner, Longjie Xie, and Xicheng Zhang.
\newblock Superposition principle for non-local
  {F}okker--{P}lanck--{K}olmogorov operators.
\newblock \emph{Probability Theory and Related Fields}, 178\penalty0
  (3--4):\penalty0 699--733, 2020.
\newblock \doi{10.1007/s00440-020-00985-8}.
\newblock URL \url{https://arxiv.org/abs/1910.11022v1}.

\bibitem[Bogdan et~al.(2014)Bogdan, Grzywny, and
  Ryznar]{BogdanGrzywnyRyznar2014}
Krzysztof Bogdan, Tomasz Grzywny, and Micha{\l} Ryznar.
\newblock Density and tails of unimodal convolution semigroups.
\newblock \emph{Journal of Functional Analysis}, 266\penalty0 (6):\penalty0
  3543--3571, 2014.
\newblock URL \url{https://arxiv.org/abs/1305.0976v2}.

\bibitem[Kaleta and Sztonyk(2015)]{KaletaSztonyk2015}
Kamil Kaleta and Pawe{\l} Sztonyk.
\newblock Estimates of transition densities and their derivatives for jump
  {L\'evy} processes.
\newblock \emph{Journal of Mathematical Analysis and Applications},
  431\penalty0 (1):\penalty0 260--282, 2015.
\newblock \doi{10.1016/j.jmaa.2015.05.061}.
\newblock URL \url{https://arxiv.org/abs/1307.1302v3}.

\bibitem[Hult and Lindskog(2007)]{HultLindskog2007}
Henrik Hult and Filip Lindskog.
\newblock Extremal behavior of stochastic integrals driven by regularly varying
  {L\'evy} processes.
\newblock \emph{The Annals of Probability}, 35\penalty0 (1):\penalty0 309--339,
  2007.
\newblock \doi{10.1214/009117906000000548}.
\newblock URL \url{https://arxiv.org/abs/math/0703802}.

\bibitem[Bingham et~al.(2006)Bingham, Goldie, and Omey]{BinghamGoldieOmey2006}
Nicholas~Hugh Bingham, Charles~M. Goldie, and Edward Omey.
\newblock Regularly varying probability densities.
\newblock \emph{Publications de l'Institut Math{\'e}matique}, 80\penalty0
  (94):\penalty0 47--57, 2006.
\newblock \doi{10.2298/PIM0694047B}.
\newblock URL \url{https://www.ma.imperial.ac.uk/~bin06/Papers/bgo.pdf}.

\bibitem[Denuit et~al.(2025)Denuit, Ortega-Jim{\'e}nez, and
  Robert]{DenuitOrtegaRobert2025}
Michel Denuit, Patricia Ortega-Jim{\'e}nez, and Christian-Yann Robert.
\newblock Conditional expectations given the sum of independent random
  variables with regularly varying densities.
\newblock \emph{ASTIN Bulletin}, 55\penalty0 (2):\penalty0 449--485, 2025.
\newblock \doi{10.1017/asb.2025.11}.

\bibitem[Dytso et~al.(2021)Dytso, Poor, and Shamai]{DytsoPoorShamai2021}
Alex Dytso, Harold~Vincent Poor, and Shlomo Shamai.
\newblock A general derivative identity for the conditional mean estimator in
  {Gaussian} noise and some applications, 2021.
\newblock URL \url{https://arxiv.org/abs/2104.01883v1}.

\bibitem[Cattiaux et~al.(2023)Cattiaux, Conforti, Gentil, and
  L{\'e}onard]{CattiauxConfortiGentilLeonard}
Patrick Cattiaux, Giovanni Conforti, Ivan Gentil, and Christian L{\'e}onard.
\newblock Time reversal of diffusion processes under a finite entropy
  condition.
\newblock \emph{Annales de l'Institut Henri Poincar{\'e}, Probabilit{\'e}s et
  Statistiques}, 59\penalty0 (4):\penalty0 1844--1881, 2023.
\newblock \doi{10.1214/22-AIHP1320}.
\newblock URL \url{https://arxiv.org/abs/2104.07708v3}.

\bibitem[Bakry et~al.(2014)Bakry, Gentil, and Ledoux]{BakryGentilLedoux2014}
Dominique Bakry, Ivan Gentil, and Michel Ledoux.
\newblock \emph{Analysis and Geometry of Markov Diffusion Operators}.
\newblock Springer, Cham, 2014.
\newblock \doi{10.1007/978-3-319-00227-9}.

\bibitem[Kevei(n.d.)]{KeveiRegularVariation}
P{\'e}ter Kevei.
\newblock Regularly varying functions.
\newblock Lecture notes, University of Szeged, n.d.
\newblock URL
  \url{https://www.math.u-szeged.hu/~kevei/tanitas/jegyzet/RegVar_notes.pdf}.

\bibitem[{National Institute of Standards and Technology}(n.d.)]{NISTDLMF}
{National Institute of Standards and Technology}.
\newblock {NIST Digital Library of Mathematical Functions}, n.d.
\newblock URL \url{https://dlmf.nist.gov/8.11#E2}.
\newblock Section 8.11(i), equation 8.11.2; consulted on 10 September 2026.

\bibitem[Schilling(2016)]{Schilling2016}
Ren{\'e}~L. Schilling.
\newblock An introduction to {L\'evy} and {Feller} processes.
\newblock Lecture notes, 2016.
\newblock URL \url{https://arxiv.org/abs/1603.00251}.

\bibitem[Hill(1975)]{Hill1975}
Bruce~M. Hill.
\newblock A simple general approach to inference about the tail of a
  distribution.
\newblock \emph{The Annals of Statistics}, 3\penalty0 (5):\penalty0 1163--1174,
  1975.
\newblock \doi{10.1214/aos/1176343247}.

\bibitem[Foss et~al.(2009)Foss, Korshunov, and
  Zachary]{FossKorshunovZachary2009}
Sergey Foss, Dmitry Korshunov, and Stan Zachary.
\newblock Convolutions of long-tailed and subexponential distributions.
\newblock \emph{Journal of Applied Probability}, 46\penalty0 (3):\penalty0
  756--767, 2009.
\newblock \doi{10.1239/jap/1253279850}.

\bibitem[Chambers et~al.(1976)Chambers, Mallows, and
  Stuck]{ChambersMallowsStuck1976}
John~M. Chambers, Colin~Lingwood Mallows, and Bart~W. Stuck.
\newblock A method for simulating stable random variables.
\newblock \emph{Journal of the American Statistical Association}, 71\penalty0
  (354):\penalty0 340--344, 1976.
\newblock \doi{10.1080/01621459.1976.10480344}.

\bibitem[Bradbury et~al.(2018)Bradbury, Frostig, Hawkins, Johnson, Leary,
  Maclaurin, Necula, Paszke, Vander{P}las, Wanderman-{M}ilne, and
  Zhang]{SoftwareBradbury2018JAX}
James Bradbury, Roy Frostig, Peter Hawkins, Matthew~James Johnson, Chris Leary,
  Dougal Maclaurin, George Necula, Adam Paszke, Jake Vander{P}las, Skye
  Wanderman-{M}ilne, and Qiao Zhang.
\newblock {JAX}: composable transformations of {P}ython+{N}um{P}y programs,
  2018.
\newblock URL \url{https://github.com/jax-ml/jax}.
\newblock Version 0.7.2.

\bibitem[Harris et~al.(2020)Harris, Millman, {van der Walt}, Gommers, Virtanen,
  Cournapeau, Wieser, Taylor, Berg, Smith, Kern, Picus, Hoyer, {van Kerkwijk},
  Brett, Haldane, {Fernández del Río}, Wiebe, Peterson, Gérard-Marchant,
  Sheppard, Reddy, Weckesser, Abbasi, Gohlke, and
  Oliphant]{SoftwareHarris2020NumPy}
Charles~R. Harris, K.~Jarrod Millman, Stéfan~J {van der Walt}, Ralf Gommers,
  Pauli Virtanen, David Cournapeau, Eric Wieser, Julian Taylor, Sebastian Berg,
  Nathaniel~J. Smith, Robert Kern, Matti Picus, Stephan Hoyer, Marten~H. {van
  Kerkwijk}, Matthew Brett, Allan Haldane, Jaime {Fernández del Río}, Mark
  Wiebe, Pearu Peterson, Pierre Gérard-Marchant, Kevin Sheppard, Tyler Reddy,
  Warren Weckesser, Hameer Abbasi, Christoph Gohlke, and Travis~E. Oliphant.
\newblock Array programming with {NumPy}.
\newblock \emph{Nature}, 585:\penalty0 357--362, 2020.
\newblock \doi{10.1038/s41586-020-2649-2}.

\bibitem[Virtanen et~al.(2020)Virtanen, Gommers, Oliphant, Haberland, Reddy,
  Cournapeau, Burovski, Peterson, Weckesser, Bright, {van der Walt}, Brett,
  Wilson, Millman, Mayorov, Nelson, Jones, Kern, Larson, Carey, Polat, Feng,
  Moore, {VanderPlas}, Laxalde, Perktold, Cimrman, Henriksen, Quintero, Harris,
  Archibald, Ribeiro, Pedregosa, {van Mulbregt}, and {SciPy 1.0
  Contributors}]{SoftwareVirtanen2020SciPy}
Pauli Virtanen, Ralf Gommers, Travis~E. Oliphant, Matt Haberland, Tyler Reddy,
  David Cournapeau, Evgeni Burovski, Pearu Peterson, Warren Weckesser, Jonathan
  Bright, St{\'e}fan~J. {van der Walt}, Matthew Brett, Joshua Wilson, K.~Jarrod
  Millman, Nikolay Mayorov, Andrew R.~J. Nelson, Eric Jones, Robert Kern, Eric
  Larson, C~J Carey, {\.I}lhan Polat, Yu~Feng, Eric~W. Moore, Jake
  {VanderPlas}, Denis Laxalde, Josef Perktold, Robert Cimrman, Ian Henriksen,
  E.~A. Quintero, Charles~R. Harris, Anne~M. Archibald, Ant{\^o}nio~H. Ribeiro,
  Fabian Pedregosa, Paul {van Mulbregt}, and {SciPy 1.0 Contributors}.
\newblock {{SciPy} 1.0: Fundamental Algorithms for Scientific Computing in
  Python}.
\newblock \emph{Nature Methods}, 17:\penalty0 261--272, 2020.
\newblock \doi{10.1038/s41592-019-0686-2}.
\newblock URL \url{https://doi.org/10.1038/s41592-019-0686-2}.

\end{thebibliography}

\clearpage
\appendix
\thispagestyle{plain}

\addcontentsline{toc}{section}{Appendix}
\part{Appendix}
\parttoc
\clearpage

\section*{Notation}
\addcontentsline{toc}{section}{Notation}

The following notation is used throughout the supplement. Vectors live in
$\mathbb{R}^d$ and all stable laws are isotropic; $\alpha\in(1,2)$ and $\nu>1$
are fixed, and subscripts $t$ (resp.\ $\tau$) denote forward (resp.\ reverse)
time.

{\footnotesize
\renewcommand{\arraystretch}{1.12}
\setlength{\LTleft}{0pt}\setlength{\LTright}{0pt plus 1fill}
\newcolumntype{Y}{>{$\relax}p{0.19\linewidth}<{$}}

\begin{longtable}{@{} Y @{\hspace{0.6em}}>{\raggedright\arraybackslash}p{0.27\linewidth}@{\hspace{1.2em}} Y @{\hspace{0.6em}}>{\raggedright\arraybackslash}p{0.27\linewidth}@{}}
\caption{Notation used throughout the supplement.}\label{tab:notation}\\
\toprule
\multicolumn{1}{@{}l@{\hspace{0.6em}}}{\normalfont\bfseries Symbol}
& \normalfont\bfseries Meaning
& \multicolumn{1}{l@{\hspace{0.6em}}}{\normalfont\bfseries Symbol}
& \normalfont\bfseries Meaning \\
\midrule
\endfirsthead
\toprule
\multicolumn{1}{@{}l@{\hspace{0.6em}}}{\normalfont\bfseries Symbol}
& \normalfont\bfseries Meaning
& \multicolumn{1}{l@{\hspace{0.6em}}}{\normalfont\bfseries Symbol}
& \normalfont\bfseries Meaning \\
\midrule
\endhead
\bottomrule
\endfoot

\multicolumn{4}{@{}l@{}}{\textit{Indices, horizon, and schedule}}\\[0.2em]
\alpha\in(1,2) & stability index of the driving noise
& \nu & target tail index, $\nu>1$ (stable $W_p$ bounds use $\nu\ge1+\alpha$)\\
d,\ p & dimension; Wasserstein / moment order, $p\in[1,\alpha)$
& T,\ \varepsilon & terminal time and early-stopping threshold\\
\beta(t) & noise schedule; $\beta_{\min}=\beta(\varepsilon)$ and $\beta(T)$ is its maximum on $[\varepsilon,T]$
& B(t) & integrated schedule, $\int_0^t\beta(s)\,\mathrm ds$; $B_\varepsilon=B(\varepsilon)$, $B_T=B(T)$\\
\bar\beta(\tau) & reversed schedule, $\beta(T-\tau)$
& a(t) & signal scale, $e^{-B(t)/\alpha}$ (in VP, $a_t=e^{-B(t)/2}$)\\
\gamma(t) & noise scale, $\gamma(t)^\alpha=1-a(t)^\alpha$ (in VP, $\gamma_t^2=1-a_t^2$)
& \eta/\alpha & coefficient of the linear dissipative term\\
\widetilde m(\tau) & exponential attenuator, $\exp(-\frac{\eta}{\alpha}\int_0^\tau\bar\beta)$
& h,\ N,\ \tau_k & step size, number of steps, and grid points; $h=(T-\varepsilon)/N$\\
\Delta B & per-step injected noise, $\Delta B\le h\beta(T)$
& & \\
\addlinespace[0.45em]

\multicolumn{4}{@{}l@{}}{\textit{Processes and laws}}\\[0.2em]
L^\alpha_t,\ \overleftarrow L{}^\alpha_\tau & forward and backward isotropic $\alpha$-stable driving processes
& \nu_\alpha & Lévy measure, $\nu_\alpha(\|z\|>\rho)=C_{d,\alpha}\rho^{-\alpha}$\\
X_t & forward process, $X_0\sim\pi_{\mathrm{data}}$
& \overleftarrow Y_\tau & exact-score reverse process initialized at $\pi_\infty^\alpha$\\
\bar X^{\theta,h}_{\tau_k} & exponential-integrator (EI) scheme
& \pi_t,\ \pi_t^{(2)} & stable and Brownian forward marginals and densities\\
\mu_\tau^{(\alpha)},\ \mu_\tau^{(2)} & stable and Brownian exact-score outputs under reference initialization; $\mu_\tau$ locally
& \widehat\pi_N^{(\beta,\theta,\eta,\varepsilon)} & sampled law after $N$ steps\\
\pi_\infty^\alpha & stationary stable law, $\mathrm S\alpha\mathrm S_d(1)$
& \pi_\infty^{(2)} & stationary VP law, $\mathcal N(0,I_d)$\\
\mathrm{Law}(\cdot),\ \overset{\mathrm d}{=} & law of a variable; equality in law
& \overleftarrow Z_\tau & exact-score reverse process initialized at $\pi_T$, with law $\pi_{T-\tau}$\\
\addlinespace[0.45em]

\multicolumn{4}{@{}l@{}}{\textit{Fractional score, denoiser, and drift}}\\[0.2em]
\Delta^{(\alpha-2)/2}\nabla & fractional gradient, multiplier $-i\xi\|\xi\|^{\alpha-2}$
& S_t^{(\alpha)} & marginal fractional score\\
S_t^{(\alpha)}(\cdot\mid x_0) & conditional fractional score
& m_t(x) & posterior barycenter; MSE-optimal denoiser when $X_0\in L^2$\\
C_t & global denoiser modulus, $\sup_x\|\nabla m_t(x)\|_{\mathrm{op}}$
& \bar C_\varepsilon & uniform modulus, $\sup_{t\ge\varepsilon}C_t$, in Assumption~\ref{ass:denoiser}\\
C_t^{(\alpha)} & score dissipativity constant, $(1-a(t)C_t)/(\alpha\gamma(t)^\alpha)$; $C_t^{(2)}$ is its VP analogue
& \bar b(\tau,x),\ \bar L(\tau) & backward drift and one-sided Lipschitz constant\\
S_t^{(\alpha),\mathrm{res}} & stationary-score residual, $x+\alpha S_t^{(\alpha)}(x)=\alpha(S_t^{(\alpha)}-S_\infty^{(\alpha)})$
& L_t^{\mathrm{res}},\ C_t^{\mathrm{res}} & Lipschitz and one-sided Lipschitz constants of the residual\\
s=\nabla\log \varphi_\alpha & score of the standard stable density
& s_\theta,\ h_\theta & learned score and learned residual map $x+\alpha s_\theta$\\
\chi_t & direct-time contraction coefficient, $(1+\eta)C_t^{(\alpha)}-1/\alpha$
& \mathfrak A_\varepsilon & cumulative non-contractivity after early stopping\\
\mathfrak D_\varepsilon^{\mathrm{mix}} & cumulative deficit from the asymptotic mixing rate
& \mathcal G_{\varepsilon,h} & uniform bound on products of EI one-step factors\\
\addlinespace[0.45em]

\multicolumn{4}{@{}l@{}}{\textit{Constants of the standard stable law}}\\[0.2em]
\varphi_\alpha & density of $\mathrm S\alpha\mathrm S_d(1)$
& k_{d,\alpha} & density tail constant, $\varphi_\alpha(z)\sim k_{d,\alpha}\|z\|^{-(d+\alpha)}$\\
C_{d,\alpha} & radial tail constant, $\omega_{d-1}k_{d,\alpha}/\alpha$
& c_{d,\alpha,p} & $p$-moment constant, $\mathbb E[\|Z\|^p]^{1/p}$, finite iff $p<\alpha$\\
c_\varphi,\ C_\varphi & two-sided envelope constants of $\varphi_\alpha$
& \Lambda_1,\ \Lambda_2 & $\sup\|s\|$ and $\sup\|\nabla s\|$\\
\Lambda & $\sup\|\nabla^2\log \varphi_\alpha\|$
& & \\
\addlinespace[0.45em]

\multicolumn{4}{@{}l@{}}{\textit{Target distribution and its tail}}\\[0.2em]
\pi_{\mathrm{data}}\ (\pi\text{ locally}) & target law and density, represented by the same symbol
& \overline G_{\pi_{\mathrm{data}}}(r) & radial survival, $\mathbb P(\|X_0\|>r)$\\
c_{\pi,-},\ c_{\pi,+},\ R_\pi & two-sided density-envelope constants and threshold
& M_1,\ M_2 & $\mathbb E\|X_0\|$ and $\mathbb E\|X_0\|^2$\\
S_{\varepsilon,T} & uniform bound on the smoothed Brownian score for $s\in[\varepsilon,T]$
& M_*^{(p)} & uniform conditional $p$-th moment\\
B_\varepsilon^{(p)} & uniform moment bound for the ideal reverse process
& R & support radius in the compact case\\
\addlinespace[0.45em]

\multicolumn{4}{@{}l@{}}{\textit{Effective tail scales of the generated law}}\\[0.2em]
\rho_\tau\ (\rho_{\nu,\tau}\text{ locally}) & scale of the exactly stable leading term $\mathcal W_\tau$
& E(\tau,s),\ E_\nu(\tau,s) & gains of the corresponding linear flows from $s$ to $\tau$\\
c_{\tau} & tail constant, $\overline G_{\mu_\tau^{(\alpha)}}(r)\sim c_{\tau}r^{-\alpha}$
& \sigma_\tau^2,\ K_\tau & Brownian variance and bounded-remainder radius\\
\addlinespace[0.45em]

\multicolumn{4}{@{}l@{}}{\textit{Tail-regime and extreme-quantile functionals}}\\[0.2em]
e_\mu(u) & radial mean excess, $\mathbb E_\mu[\|X\|-u\mid\|X\|>u]$
& \Gamma(\mu) & limiting relative mean excess, $\lim_{u\to\infty}e_\mu(u)/u$\\
q_\kappa(\mu) & radial $\kappa$-quantile
& \mathcal M_\mu(\kappa) & generated-to-target extreme-quantile ratio\\
\operatorname{osc}_S(f) & oscillation, $\sup_Sf-\inf_Sf$
& & \\
\addlinespace[0.45em]

\multicolumn{4}{@{}l@{}}{\textit{Error terms and metrics}}\\[0.2em]
\mathcal E_0,\ \mathcal E_1,\ \mathcal E_2 & early-stopping, mixing, and score-plus-discretization errors
& \mathrm{err} & score-approximation error in Assumption~\ref{ass:score}\\
L_\varepsilon^{\mathrm{time}} & time-regularity constant in Assumption~\ref{ass:time}
& W_p & order-$p$ Wasserstein distance\\
\|\cdot\|_{\mathrm{TV}} & total-variation distance
& \mathrm{KL}(\cdot\,\|\,\cdot) & Kullback--Leibler divergence\\
\addlinespace[0.45em]

\multicolumn{4}{@{}l@{}}{\textit{Miscellaneous}}\\[0.2em]
I_d & $d\times d$ identity matrix
& \varphi_\gamma,\ \varphi & Brownian density of scale $\gamma$; standard Brownian density\\
\omega_{d-1} & surface area of $\mathbb S^{d-1}$
& \bar B(0,R) & closed Euclidean ball of radius $R$\\
\langle\cdot,\cdot\rangle,\ \|\cdot\| & Euclidean inner product and norm
& (x)_+ & positive part, $\max(x,0)$\\
\end{longtable}
}
The constants $k_{d,\alpha}$, $C_{d,\alpha}$, and $c_{d,\alpha,p}$ are,
respectively, density-tail, radial-tail, and moment constants; despite their
similar notation, they are not interchangeable.

\section{Proofs for Section~\ref{sec:setup}: Stable Dynamics}\label{app:setup}

The explicit stable Ornstein--Uhlenbeck transition law gives both scale
preservation and the conditional score. The Fourier convention stated below
\eqref{eq:score_def} then identifies the marginal score and the reverse
marginals.

\subsection{Strong solution and scale preservation}

\begin{lemma}[Strong solution]\label{app:lem_strong}
The SP-SDE \eqref{eq:spsde} admits the unique strong solution
\begin{equation}\label{eq:app_strong}
X_t=a(t)X_0+\int_0^t \frac{a(t)}{a(s)}\,\beta(s)^{1/\alpha}\,\mathrm{d} L^\alpha_s ,
\end{equation}
and $X_t\overset{\mathcal L}=a(t)X_0+\gamma(t)Z$ with
$Z\sim\mathrm{S}\alpha\mathrm{S}_d(1)$ independent of $X_0$. If
$X_0\sim\mathrm{S}\alpha\mathrm{S}_d(1)$, then $X_t\sim\mathrm{S}\alpha\mathrm{S}_d(1)$ for
all $t$.
\end{lemma}

\begin{proof}
Variation of constants gives \eqref{eq:app_strong} and uniqueness.
For a deterministic scalar $k\in L^\alpha$, independence of stable increments
gives
$\mathbb{E} e^{i\xi\cdot\int k\,\mathrm{d} L^\alpha}
=\exp(-\|\xi\|^\alpha\int|k|^\alpha)$,
first for simple $k$ and then by approximation.
Here the scale to the power $\alpha$ is
\[
 \int_0^t\left(\frac{a(t)}{a(s)}\right)^\alpha\beta(s)\,\mathrm{d} s
 =\int_0^t\beta(s)e^{-[B(t)-B(s)]}\,\mathrm{d} s
 =1-e^{-B(t)}=\gamma(t)^\alpha.
\]
The integral is independent of $X_0$, and scale preservation follows from
$a(t)^\alpha+\gamma(t)^\alpha=1$.
\end{proof}

\subsection{Proof of Proposition~\ref{prop:tweedie}: Fractional Tweedie formula}

The following conditional-score identity is standard for the
L\'evy--It\^o model \citep{Yoon2023}. We include its short derivation to fix
our Fourier and stable-scale conventions.

\begin{lemma}[Conditional score]\label{app:lem_cond}
The conditional fractional score is exactly linear:
\begin{equation}\label{eq:app_cond}
S^{(\alpha)}_t(x\mid x_0)=-\frac{1}{\alpha\gamma(t)^\alpha}\bigl(x-a(t)x_0\bigr).
\end{equation}
\end{lemma}

\begin{proof}
This is the conditional fractional-score formula of
\citet[Theorem~4.4]{Yoon2023}, with our stable-scale convention.
Indeed, for $f(x)=\pi_t(x\mid x_0)$, $a=a(t)$ and $\gamma=\gamma(t)$,
\[
 \widehat f(\xi)=e^{iax_0\cdot\xi-\gamma^\alpha\|\xi\|^\alpha},
 \qquad
 \mathcal F[(x-ax_0)f](\xi)
 =-\alpha\gamma^\alpha(-i\xi\|\xi\|^{\alpha-2})\widehat f(\xi).
\]
Fourier inversion and division by $f>0$ give \eqref{eq:app_cond}.
\end{proof}

\begin{proof}[\normalfont\bfseries Proof of Proposition~\ref{prop:tweedie}]
For $t>0$, the envelope
$(1+\|\xi\|^{\alpha-1})e^{-\gamma(t)^\alpha\|\xi\|^\alpha}$ is integrable
and independent of $x_0$. Fourier inversion and Fubini therefore allow
$\Delta^{(\alpha-2)/2}\nabla$ to commute with the mixing integral.
For fixed $(t,x)$, the map
$x_0\mapsto\|x_0\|\varphi_\alpha((x-a(t)x_0)/\gamma(t))$ is bounded by the
stable density envelope. Thus the conditional first moment is finite even
without an unconditional moment assumption. Integrating
\eqref{eq:app_cond} against the conditional law yields
\[
 S_t^{(\alpha)}(x)
 =\int S_t^{(\alpha)}(x\mid x_0)\,
       \mathbb{P}(X_0\in\mathrm{d} x_0\mid X_t=x)
 =\frac{a(t)m_t(x)-x}{\alpha\gamma(t)^\alpha}.
\]
\end{proof}

\subsection{Proof of Proposition~\ref{prop:family}: Matching marginals}\label{app:matching_marginals}

The exact time reversal of \eqref{eq:spsde} is a jump process whose jump
kernel is the stable L\'evy measure tilted by a ratio of marginals,
$\bar\beta(\tau)\nu_\alpha(\mathrm{d}z)\,\pi_{T-\tau}(x+z)/\pi_{T-\tau}(x)$
\citep{ConfortiLeonard2022}. It is not driven by a pure L\'evy process and is not
used here. We justify instead the matching-marginals family
\eqref{eq:backward} in two steps.

Set $H:=T-\varepsilon$ and $\bar\beta(\tau):=\beta(T-\tau)$. For $r\ge1$, write
$\mathcal P_r(\mathbb R^d)$ for the probability measures with finite $r$-th
moment. For a compact interval $I$, let $\mathcal C_1(I)$ denote the narrowly continuous curves
$(\mu_t)_{t\in I}$ in $\mathcal P(\mathbb R^d)$ such that $\int_I\!\int_{\mathbb
R^d}\|x\|\,\mu_t(\mathrm{d}x)\,\mathrm{d}t<\infty$.
Here $\Delta^{\alpha/2}=-(-\Delta)^{\alpha/2}$ denotes the generator with
Fourier multiplier $-\|\xi\|^\alpha$. In the weak sense, the forward
Fokker--Planck Cauchy problem on
$[\varepsilon,T]$ is
\begin{equation}\label{eq:app_forward_fp}
\begin{aligned}
 \int\varphi\,\mathrm{d}\nu_t-
 \int\varphi\,\mathrm{d}\pi_\varepsilon
 =\int_\varepsilon^t\!\int
 \Bigl[-\frac{\beta(s)}{\alpha}x\!\cdot\!\nabla\varphi(x)
 +\beta(s)\Delta^{\alpha/2}\varphi(x)\Bigr]
 \,\nu_s(\mathrm{d}x)\,\mathrm{d}s,
\end{aligned}
\end{equation}
for $t\in[\varepsilon,T]$ and every $\varphi\in C_c^2(\mathbb R^d)$. The
corresponding reverse problem on $[0,H]$ is
\begin{equation}\label{eq:app_backward_fp}
\begin{aligned}
 \int\varphi\,\mathrm{d}\mu_\tau-
 \int\varphi\,\mathrm{d}\pi_T
 ={}&\int_0^\tau\!\int
 \bar b(u,x)\!\cdot\!\nabla\varphi(x)
 \,\mu_u(\mathrm{d}x)\,\mathrm{d}u\\
 &+\eta\int_0^\tau\bar\beta(u)
 \int\Delta^{\alpha/2}\varphi(x)\,\mu_u(\mathrm{d}x)\,\mathrm{d}u.
\end{aligned}
\end{equation}

\begin{proposition}[Specialization of Popov et al.'s theorem]
\label{app:prop_popov_adaptation}
Assume $\pi_{\mathrm{data}}\in\mathcal P_1(\mathbb R^d)$ and the following
sufficient weak formulation of the well-posedness requirements underlying
\citet[Theorem~1 and Appendix~A]{Popov2025}:
\begin{enumerate}[label=\textup{(P\arabic*)},leftmargin=2.8em,itemsep=1pt,topsep=2pt]
 \item the forward SDE \eqref{eq:spsde}, restricted to $[\varepsilon,T]$
 with initial law $\pi_\varepsilon$, and the reverse SDE
 \eqref{eq:backward} admit weak solutions whose marginal curves belong,
 respectively, to $\mathcal C_1([\varepsilon,T])$ and
 $\mathcal C_1([0,H])$;
 \item $(\pi_t)_{t\in[\varepsilon,T]}$ is the unique solution of
 \eqref{eq:app_forward_fp} in $\mathcal C_1([\varepsilon,T])$, and
 \eqref{eq:app_backward_fp} has a unique solution in $\mathcal C_1([0,H])$.
\end{enumerate}
Then every weak solution of \eqref{eq:backward} whose marginal curve belongs
to $\mathcal C_1([0,H])$ satisfies
$\mathrm{Law}(\overleftarrow Z_\tau)=\pi_{T-\tau}, \qquad 0\le\tau\le H$.
\end{proposition}

\begin{proof}
Specialize the Fokker--Planck calculation of
\citet[Theorem~1 and Appendix~A]{Popov2025} to
$\mu(x,t)=-\beta(t)x/\alpha$, $\sigma_t=\beta(t)^{1/\alpha}$ and
$\eta_t=\eta$. For $\mu_\tau=\pi_{T-\tau}$, the identity
$\operatorname{div}(\mu_\tau S_{T-\tau}^{(\alpha)})
=\Delta^{\alpha/2}\mu_\tau$ gives, in the weak sense,
\[
 -\operatorname{div}(\bar b\,\mu_\tau)
 +\eta\bar\beta(\tau)\Delta^{\alpha/2}\mu_\tau
 =-\frac{\bar\beta(\tau)}{\alpha}\operatorname{div}(x \mu_\tau)
  -\bar\beta(\tau)\Delta^{\alpha/2}\mu_\tau
 =\partial_\tau \mu_\tau.
\]
This curve belongs to $\mathcal C_1([0,H])$ and starts from $\pi_T$.
Under \textup{(P1)} the reverse marginals solve the same weak problem;
uniqueness in \textup{(P2)} identifies them with $\mu_\tau$.
\end{proof}

\begin{proposition}[Verification of the Popov--et--al. well-posedness hypothesis]
\label{app:prop_popov_wellposedness}
Assume $\pi_{\mathrm{data}}\in\mathcal P_1(\mathbb R^d)$ and that the drift
\eqref{eq:backward_drift} is Borel measurable, globally Lipschitz in $x$
uniformly in $\tau\in[0,H]$, and satisfies
$\sup_{\tau\in[0,H]}\|\bar b(\tau,0)\|<\infty$. Then conditions
\textup{(P1)}--\textup{(P2)} of
Proposition~\ref{app:prop_popov_adaptation} hold. Consequently,
Proposition~\ref{prop:family} holds.

Moreover, the stated regularity conditions on $\bar b$ follow from
\begin{equation*}
 \bar C_{\varepsilon,T}:=\sup_{t\in[\varepsilon,T]}C_t<\infty.
\end{equation*}
In particular, since $\bar C_{\varepsilon,T}\le\bar C_\varepsilon$, Assumption~\ref{ass:denoiser} together with $\pi_{\mathrm{data}}\in\mathcal P_1(\mathbb R^d)$ verifies the conditions of Proposition~\ref{prop:family}.
In that case,
\begin{equation}\label{eq:app_bLip}
 \sup_{\tau\in[0,H]}\mathrm{Lip}_x\bigl(\bar b(\tau,\cdot)\bigr)
 \le\|\beta\|_{L^\infty([\varepsilon,T])}
 \Bigl(\tfrac1\alpha+(1+\eta)
 \tfrac{1+\bar C_{\varepsilon,T}}
 {\alpha\gamma(\varepsilon)^\alpha}\Bigr).
\end{equation}
\end{proposition}

\begin{proof}
Subtracting the additive stochastic integral reduces either SDE to a
pathwise Carath\'eodory equation with a Lipschitz vector field of linear
growth. This gives a non-explosive strong solution and uniqueness from
every starting time. The additive stable integral is a martingale with finite $r$-th
moments for $1<r<\alpha$; Doob's $L^r$ inequality gives a finite maximal
first moment on $[0,H]$. Gr\"onwall therefore gives \textup{(P1)}.
The jump measures are $\beta(t)\nu_\alpha$ and
$\eta\bar\beta(\tau)\nu_\alpha$, with bounded multipliers. Their required
characteristics satisfy
\[
 \int_{\|z\|\le1}\|z\|^2\nu_\alpha(\mathrm{d} z)<\infty,
 \qquad
 \int_{\|z\|>1}\log(1+\|z\|)\nu_\alpha(\mathrm{d} z)<\infty.
\]
Together with linear drift growth, these bounds verify condition~(1.17)
of \citet[arXiv version~1]{RocknerXieZhang2020}. Extend the coefficients past the fixed
interval with the same bounds. Strong well-posedness gives well-posed
martingale problems; their Corollary~1.8 then gives \textup{(P2)}.
Finally, \eqref{eq:tweedie} gives
$\mathrm{Lip}(S_s^{(\alpha)})\le
(1+a(s)C_s)/(\alpha\gamma(s)^\alpha)$, hence \eqref{eq:app_bLip}.
In the integral formula for $m_s(0)$ the numerator is uniformly bounded
by $C\mathbb{E}\|X_0\|$ and the positive continuous denominator is bounded
away from zero on $[\varepsilon,T]$. This also bounds $\bar b(\tau,0)$;
the same formula gives its measurability.
\end{proof}

\subsection{Proof of Proposition~\ref{prop:dissipativity}: Dissipativity}

\begin{proof}[\normalfont\bfseries Proof of Proposition~\ref{prop:dissipativity}]
By \eqref{eq:tweedie} and the definition of $C_t$,
\[
 \langle S_t^{(\alpha)}(x)-S_t^{(\alpha)}(y),x-y\rangle
 \le-\frac{1-a(t)C_t}{\alpha\gamma(t)^\alpha}\|x-y\|^2.
\]
This proves the score estimate. Substitution into \eqref{eq:backward_drift}
gives \eqref{eq:osl}, with
$\bar L(\tau)=\bar\beta(\tau)(1/\alpha-(1+\eta)C_{T-\tau}^{(\alpha)})$.
\end{proof}

\subsection{Measurability and transient anti-contractivity}

\begin{lemma}[Measurability of the denoiser modulus]\label{app:lem_ct_measurable}
Assume $\pi_{\mathrm{data}}\in\mathcal P_1(\mathbb R^d)$. Then
$(t,x)\mapsto\nabla m_t(x)$ is continuous on $(0,\infty)\times\mathbb R^d$ and
$t\mapsto C_t$ is Borel measurable, with values in $[0,\infty]$.
\end{lemma}

\begin{proof}
Write $m=N/D$ using the conditional numerator and normalizer.
Boundedness of $\varphi_\alpha,\nabla \varphi_\alpha$, the first moment of $\pi$, and
$\inf_{[\varepsilon,T]}\gamma>0$ justify differentiation and joint
dominated convergence. Since $D>0$, $\nabla m$ is jointly continuous.
Consequently $C_t=\sup_{x\in\mathbb Q^d}\|\nabla m_t(x)\|_{\mathrm{op}}$
is a countable supremum of continuous functions, hence Borel measurable.
\end{proof}

Define the direct-time contraction coefficient
\begin{equation}\label{eq:app_chi_direct}
 \chi_t:=(1+\eta)C_t^{(\alpha)}-\frac1\alpha,
 \qquad
 \bar L(\tau)=-\bar\beta(\tau)\chi_{T-\tau},
\end{equation}
and, for a fixed early-stopping time $\varepsilon>0$, the cumulative non-contractivity $\mathfrak
A_\varepsilon :=\int_\varepsilon^\infty\beta(t)(\chi_t)_-\,\mathrm d t$.
By Lemma~\ref{app:lem_ct_measurable}, this integral is well defined.

\begin{lemma}[Finite transient amplification]\label{app:lem_transient}
Under Assumptions~\ref{ass:schedule} and~\ref{ass:denoiser}, and $\pi_{\mathrm{data}}\in\mathcal P_1(\mathbb R^d)$, one has $\chi_t\to\eta/\alpha$. In particular, there exists
$t_+\ge\varepsilon$, independent of the horizon $T$, such that
\begin{equation}\label{eq:app_eventual_half}
 \chi_t\ge\frac{\eta}{2\alpha},\qquad t\ge t_+,
\end{equation}
and $\mathfrak A_\varepsilon<\infty$. More precisely,
\begin{equation}\label{eq:app_debt_bound}
 \mathfrak A_\varepsilon
 \le K_\varepsilon\bigl(B(t_+)-B(\varepsilon)\bigr),
 \qquad
 K_\varepsilon:=\frac1\alpha+
 \frac{1+\eta}{\alpha\gamma(\varepsilon)^\alpha}
 \bigl(a(\varepsilon)\bar C_\varepsilon-1\bigr)_+.
\end{equation}
\end{lemma}

\begin{proof}
The bound $a(t)C_t\le a(t)\bar C_\varepsilon\to0$ gives
$C_t^{(\alpha)}\to1/\alpha$ and $\chi_t\to\eta/\alpha$.
Thus \eqref{eq:app_eventual_half} holds after a finite $t_+$ independent
of $T$. Moreover, for $t\ge\varepsilon$,
\[
 (\chi_t)_-\le\frac1\alpha+
 \frac{1+\eta}{\alpha\gamma(\varepsilon)^\alpha}
       (a(\varepsilon)\bar C_\varepsilon-1)_+=K_\varepsilon.
\]
The negative part vanishes after $t_+$; integration gives
\eqref{eq:app_debt_bound}.
\end{proof}

\section{Proof of Theorem~\ref{cor:rate}: The Non-Asymptotic Bound}\label{app:bound}

We follow the error-propagation structure of
\citet[Appendix~C.2]{Strasman2024}, adapting it to stable noise in $L^p$,
$p<\alpha$, and allowing a finite non-contractive phase.

We prove the bound under Assumptions~\ref{ass:schedule}--\ref{ass:time} and~\ref{ass:denoiser} for a
target supported in $\bar B(0,R)$. Appendix~\ref{app:separation} treats the
heavy-tailed case. Propositions~\ref{prop:ct_decay} and~\ref{prop:heavy_denoiser}, respectively, verify Assumption~\ref{ass:denoiser} for these two target classes.
Write $\bar\beta(\tau)=\beta(T-\tau)$.

\subsection{Dissipative reformulation and the EI scheme}\label{app:ei_scheme}

Since $S^{(\alpha)}_\infty(x)=-x/\alpha$, define the stationary-score residual map
$S^{(\alpha),\mathrm{res}}_t(x):=x+\alpha S^{(\alpha)}_t(x)
=\alpha\bigl(S^{(\alpha)}_t(x)-S^{(\alpha)}_\infty(x)\bigr)$. Then $\bar
b(\tau,x)=-\frac{\eta}{\alpha}\,\bar\beta(\tau)\,x+\frac{1+\eta}{\alpha}\,\bar\beta(\tau)\,S^{(\alpha),\mathrm{res}}_{T-\tau}(x)$,
and $S^{(\alpha),\mathrm{res}}_\infty\equiv0$. On the uniform grid
$0=\tau_0<\cdots<\tau_N=T-\varepsilon$, the EI scheme is
\begin{equation}\label{eq:app_ei}
\bar X^{\theta,h}_{\tau_{k+1}}=\widetilde m_k\,\bar X^{\theta,h}_{\tau_k}
+\frac{1+\eta}{\alpha}\int_{\tau_k}^{\tau_{k+1}}\frac{\widetilde m(\tau_{k+1})}{\widetilde
m(u)}\,\bar\beta(u)\,h_\theta(\tau_k,\bar X^{\theta,h}_{\tau_k})\,\mathrm{d} u
+\int_{\tau_k}^{\tau_{k+1}}\frac{\widetilde m(\tau_{k+1})}{\widetilde
m(u)}\,\bigl(\eta\,\bar\beta(u)\bigr)^{1/\alpha}\,\mathrm{d}\overleftarrow{L}{}^\alpha_u ,
\end{equation}
where $h_\theta(t,x):=x+\alpha s_\theta(t,x)$ and the exponential attenuator and its
one-step ratio are $\widetilde
m(\tau):=\exp\Bigl(-\tfrac{\eta}{\alpha}\int_0^\tau\bar\beta(v)\,\mathrm{d} v\Bigr)$ and $\widetilde m_k:=\frac{\widetilde m(\tau_{k+1})}{\widetilde m(\tau_k)}$. The attenuator is
decreasing. Define $L^{\mathrm{res}}_{t}:=1+\frac{1+a(t)C_t}{\gamma(t)^\alpha}$ and $C^{\mathrm{res}}_{t}:=1-\alpha C_t^{(\alpha)}$, which respectively bound the Lipschitz and
one-sided Lipschitz constants of
$S_t^{(\alpha),\mathrm{res}}$.

\subsection{The early-stopping bias}

\begin{lemma}[Early-stopping bias]\label{app:lem_bias}
For every $\varepsilon\in(0,T)$ and $p\in[1,\alpha)$,
\begin{equation*}
\mathcal{E}_0=W_p(\pi_{\mathrm{data}},\pi_\varepsilon)\;\le\;\bigl(1-a(\varepsilon)\bigr)\,\mathbb{E}[\left\lVert X_0\right\rVert^p]^{1/p}+\gamma(\varepsilon)\,c_{d,\alpha,p}.
\end{equation*}
\end{lemma}

\begin{proof}
By Lemma~\ref{app:lem_strong},
$X_\varepsilon\overset{\mathcal L}=a(\varepsilon)X_0+\gamma(\varepsilon)Z_\alpha$ with
$Z_\alpha\sim\mathrm{S}\alpha\mathrm{S}_d(1)$ independent of $X_0$. Coupling
$\pi_{\mathrm{data}}$ and $\pi_\varepsilon$ through
$(X_0,\,a(\varepsilon)X_0+\gamma(\varepsilon)Z_\alpha)$ and applying Minkowski's inequality,
$W_p(\pi_{\mathrm{data}},\pi_\varepsilon)\le\left\lVert
X_0-a(\varepsilon)X_0-\gamma(\varepsilon)Z_\alpha\right\rVert_{L^p}\le(1-a(\varepsilon))\left\lVert
X_0\right\rVert_{L^p}+\gamma(\varepsilon)\left\lVert Z_\alpha\right\rVert_{L^p}$.
\end{proof}

\subsection{Exponential contraction of the initialization error}

\begin{lemma}[Forward contraction]\label{app:lem_forward}
For every $p\in[1,\alpha)$, $W_p(\pi_T,\pi_\infty^\alpha)\le
a(T)\,W_p(\pi_{\mathrm{data}},\pi_\infty^\alpha)$.
\end{lemma}

\begin{proof}
For any coupling of the two initial laws, use the same forward noise.
Equation~\eqref{eq:app_strong} gives $X_T-X'_T=a(T)(X_0-X'_0)$.
Taking $L^p$ norms and the infimum over initial couplings proves the claim,
since the second marginal is invariant.
\end{proof}

\begin{lemma}[Backward propagation of the mixing error]\label{app:lem_mixing}
For every $p\in[1,\alpha)$ and $\varepsilon\in(0,T)$,
\begin{equation}\label{eq:app_mixing}
\mathcal{E}_1\;\le\;W_p(\pi_{\mathrm{data}},\pi_\infty^\alpha)\,\exp\Bigl(-\frac{1+\eta}{\alpha}\int_\varepsilon^T\beta(t)\,\frac{1-a(t)C_t}{\gamma(t)^\alpha}\,\mathrm{d} t\Bigr).
\end{equation}
If $\mathfrak A_\varepsilon<\infty$, then
\begin{equation}\label{eq:app_mixing_debt}
 \mathcal E_1\le
 e^{\mathfrak A_\varepsilon}
 W_p(\pi_{\mathrm{data}},\pi_\infty^\alpha)
 e^{-\frac1\alpha(B(T)-B(\varepsilon))}.
\end{equation}
Moreover, define $\mathfrak D_\varepsilon^{\mathrm{mix}} :=\int_\varepsilon^\infty\beta(t)
\left(\frac\eta\alpha-\chi_t\right)_+\,\mathrm d t$. If $\mathfrak
D_\varepsilon^{\mathrm{mix}}<\infty$, then
\begin{equation}\label{eq:app_mixing_optimal}
 \mathcal E_1\le
 e^{\mathfrak D_\varepsilon^{\mathrm{mix}}}
 W_p(\pi_{\mathrm{data}},\pi_\infty^\alpha)
 e^{-\frac{1+\eta}{\alpha}(B(T)-B(\varepsilon))}.
\end{equation}
\end{lemma}

\begin{proof}
Realize the two well-posed backward SDEs with common noise and any coupling
of $(\pi_\infty^\alpha,\pi_T)$, independent of that noise.
Their difference $D$ is absolutely continuous, and \eqref{eq:osl} gives
$\frac{\mathrm d}{\mathrm{d} \tau}\|D_\tau\|^2
\le2\bar L(\tau)\|D_\tau\|^2$.
Gr\"onwall, the infimum over couplings, the marginal identity, and
Lemma~\ref{app:lem_forward} yield
\[
 \mathcal E_1\le W_p(\pi_{\mathrm{data}},\pi_\infty^\alpha)
 a(T)\exp\left(-\int_\varepsilon^T\beta(t)\chi_t\,\mathrm{d} t\right).
\]
The logarithm of this last factor equals
$-B(\varepsilon)/\alpha-(1+\eta)
\int_\varepsilon^T\beta(t)C_t^{(\alpha)}\,\mathrm{d} t$.
Dropping $-B(\varepsilon)/\alpha$ gives \eqref{eq:app_mixing}.
Writing its exponent as
$-[B(T)-B(\varepsilon)]/\alpha-\int_\varepsilon^T\beta\chi$
or as
$-(1+\eta)[B(T)-B(\varepsilon)]/\alpha
-\int_\varepsilon^T\beta(\chi-\eta/\alpha)$
then gives \eqref{eq:app_mixing_debt} and
\eqref{eq:app_mixing_optimal}, respectively.
\end{proof}

\subsection{Score and EI discretization error}

We first record a uniform moment bound for the ideal process, used to close the recursion.

\begin{proposition}[Uniform moment bound]\label{app:prop_moments}
Under Assumptions~\ref{ass:schedule} and~\ref{ass:denoiser}, for
$p\in[1,\alpha)$ and $\pi_{\mathrm{data}}\in\mathcal P_p(\mathbb R^d)$,
\begin{equation}\label{eq:app_bp}
\sup_{\tau\in[0,T-\varepsilon]}\mathbb{E}\bigl[\left\lVert \overleftarrow Y_\tau\right\rVert^p\bigr]^{1/p}\;\le\;\bigl(1+e^{\mathfrak A_\varepsilon}\bigr)\bigl(\mathbb{E}[\left\lVert X_0\right\rVert^p]^{1/p}+c_{d,\alpha,p}\bigr)=:B_\varepsilon^{(p)} .
\end{equation}
\end{proposition}

\begin{proof}
Lemma~\ref{app:lem_transient} gives $\mathfrak A_\varepsilon<\infty$.
Take any $\zeta\in\Pi(\pi_\infty^\alpha,\pi_T)$ and drive
$\overleftarrow Y$ and the matching-marginals process $\overleftarrow Z$
by the same stable noise, independent of their initial pair of law $\zeta$.
Proposition~\ref{app:prop_popov_wellposedness} identifies
$\mathrm{Law}(\overleftarrow Z_\tau)=\pi_{T-\tau}$, so Minkowski and the
forward representation give
$\|\overleftarrow Z_\tau\|_{L^p}\le\|X_0\|_{L^p}+c_{d,\alpha,p}$.
The synchronous gap bound and the triangle inequality yield
\[
 \|\overleftarrow Y_\tau\|_{L^p}
 \le \|X_0\|_{L^p}+c_{d,\alpha,p}
 +e^{\mathfrak A_\varepsilon}
 \left(\int\|y-z\|^p\,\zeta(\mathrm{d}y,\mathrm{d}z)\right)^{1/p}.
\]
The left-hand side is independent of $\zeta$. Taking the infimum gives
$W_p(\pi_\infty^\alpha,\pi_T)$ in the last term. Forward contraction
(Lemma~\ref{app:lem_forward}) and an independent coupling then bound it by
$W_p(\pi_\infty^\alpha,\pi_{\mathrm{data}})
\le\|X_0\|_{L^p}+c_{d,\alpha,p}$, proving \eqref{eq:app_bp}.
\end{proof}

We now bound $\mathcal{E}_2$, tracking the discretization gap $D_k:=\bar
X^{\theta,h}_{\tau_k}-\overleftarrow Y_{\tau_k}$ through a one-step contraction analysis and
a telescoping sum.

\begin{proposition}[Cumulative bound]\label{app:prop_cumul}
Under Assumptions~\ref{ass:schedule}--\ref{ass:time} and~\ref{ass:denoiser}, for a target either supported in $\bar B(0,R)$ or satisfying
Assumption~\ref{ass:heavy} with $\nu\ge1+\alpha$ (in which case $R$ is replaced by $M_1$), fix
arbitrary $0<\varepsilon<T$ and let $h\le h_0(\varepsilon,T)$ be small enough. Then
\begin{equation}\label{eq:app_e2}
\mathcal{E}_2\;\le\;\mathcal G_{\varepsilon,h}\Bigg\{(1+\eta)\bigl[\mathrm{err}+L^{\mathrm{time}}_\varepsilon h\,(1+B_\varepsilon^{(p)})\bigr]\,\bigl(B(T)-B(\varepsilon)\bigr)
+\frac{1+\eta}{\alpha}\sum_{k=0}^{N-1}\Bigl(\int_{\tau_k}^{\tau_{k+1}}\bar\beta(u)\,L^{\mathrm{res}}_{T-u}\,\mathrm{d}
u\Bigr)\,\Theta_k
\Bigg\}
\end{equation}
with $\displaystyle\Theta_k:=(B_\varepsilon^{(p)}+R)\Bigl[\frac{\Delta
B_k}{\alpha}+\frac{1+\eta}{\alpha}\int_{\tau_k}^{\tau_{k+1}}\bar\beta(v)\,\frac{1+a(T-v)C_{T-v}}{\gamma(T-v)^\alpha}\,\mathrm{d}
v\Bigr]+c_{d,\alpha,p}\,\eta^{1/\alpha}\,\Delta B_k^{1/\alpha}$, where $\Delta
B_k:=\int_{\tau_k}^{\tau_{k+1}}\bar\beta\le h\beta(T)$ and $\mathcal G_{\varepsilon,h}$ is
defined in \eqref{eq:app_grid_product} below.
\end{proposition}

\begin{proof}
Both target classes belong to $\mathcal P_1(\mathbb R^d)$, so Assumption~\ref{ass:denoiser} and Proposition~\ref{app:prop_popov_wellposedness} verify Proposition~\ref{prop:family}. Lemma~\ref{app:lem_transient} also provides $t_+$ and $\mathfrak A_\varepsilon<\infty$.
Set $r_k(u):=\widetilde m(\tau_{k+1})/\widetilde m(u)\le1$ for $u\in[\tau_k,\tau_{k+1}]$.
Write $D_k:=\bar X^{\theta,h}_{\tau_k}-\overleftarrow Y_{\tau_k}$ for the discretization
gap. Subtracting the integral form of the exact SDE for $\overleftarrow Y$ from the scheme
\eqref{eq:app_ei} (the shared jump integral cancels) and inserting $\pm
S^{(\alpha),\mathrm{res}}_{T-u}(\bar X^{\theta,h}_{\tau_k})$ and $\pm
S^{(\alpha),\mathrm{res}}_{T-u}(\overleftarrow Y_{\tau_k})$ splits the one-step gap as
$D_{k+1}=A_k+R^{(1)}_k+R^{(2)}_k$, with
\begin{align}
A_k&:=\widetilde m_kD_k+\tfrac{1+\eta}{\alpha}\int_{\tau_k}^{\tau_{k+1}}r_k(u)\bar\beta(u)\bigl[S^{(\alpha),\mathrm{res}}_{T-u}(\bar X^{\theta,h}_{\tau_k})-S^{(\alpha),\mathrm{res}}_{T-u}(\overleftarrow Y_{\tau_k})\bigr]\,\mathrm{d} u,\label{eq:app_Ak}\\
R^{(1)}_k&:=\tfrac{1+\eta}{\alpha}\int_{\tau_k}^{\tau_{k+1}}r_k(u)\bar\beta(u)\bigl[S^{(\alpha),\mathrm{res}}_{T-u}(\overleftarrow Y_{\tau_k})-S^{(\alpha),\mathrm{res}}_{T-u}(\overleftarrow Y_u)\bigr]\,\mathrm{d} u,\label{eq:app_R1}\\
R^{(2)}_k&:=\tfrac{1+\eta}{\alpha}\int_{\tau_k}^{\tau_{k+1}}r_k(u)\bar\beta(u)\bigl[h_\theta(\tau_k,\bar X^{\theta,h}_{\tau_k})-S^{(\alpha),\mathrm{res}}_{T-u}(\bar X^{\theta,h}_{\tau_k})\bigr]\,\mathrm{d} u,\label{eq:app_R2}
\end{align}
The three terms are, respectively, a contraction of the previous gap ($A_k$), the
local fluctuation of the ideal process ($R^{(1)}_k$), and the modeling-plus-time-freezing
residual ($R^{(2)}_k$). We bound each in $L^p$.

\par\smallskip\noindent\textit{Step 1 (one-step contraction of $A_k$).}\, 
Expanding $\left\lVert A_k\right\rVert^2\le\delta_k^2\left\lVert D_k\right\rVert^2$ pathwise
from \eqref{eq:app_Ak}, the deterministic factor $\delta_k$ is
\begin{equation}\label{eq:app_deltak}
\delta_k^2=\underbrace{\widetilde m_k^2}_{(a)}
+\underbrace{2\widetilde
m_k\,\tfrac{1+\eta}{\alpha}\!\int_{\tau_k}^{\tau_{k+1}}\!r_k(u)\bar\beta(u)C^{\mathrm{res}}_{T-u}\,\mathrm{d} u}_{(b)}
+\underbrace{\Bigl(\tfrac{1+\eta}{\alpha}\int_{\tau_k}^{\tau_{k+1}}\!r_k(u)\bar\beta(u)L^{\mathrm{res}}_{T-u}\,\mathrm{d} u\Bigr)^2}_{(c)} .
\end{equation}
We bound the three terms without expanding $\bar\beta$ or the score constants at the left
endpoint, neither of which is assumed Lipschitz in time. Write
$I_k:=\int_{\tau_k}^{\tau_{k+1}}\bar\beta$, so that $\widetilde
m_k=e^{-\frac{\eta}{\alpha}I_k}$ and $r_k(u)=\exp\bigl(-\tfrac\eta\alpha\int_u^{\tau_{k+1}}\bar\beta\bigr)\in(0,1]$. The substitution
$v=\int_u^{\tau_{k+1}}\bar\beta$, for which $\,\mathrm{d} v=-\bar\beta(u)\,\mathrm{d} u$,
gives the \emph{exact} identity
\begin{equation}\label{eq:app_exactint}
\int_{\tau_k}^{\tau_{k+1}}r_k(u)\,\bar\beta(u)\,\mathrm{d} u=\int_0^{I_k}e^{-\frac\eta\alpha v}\,\mathrm{d} v=\frac\alpha\eta\bigl(1-\widetilde m_k\bigr).
\end{equation}
Set $J_k:=\int_{\tau_k}^{\tau_{k+1}}r_k(u)\bar\beta(u) \chi_{T-u}\,\mathrm d u$, and
$Q_k:=\frac{1+\eta}{\alpha} \int_{\tau_k}^{\tau_{k+1}}r_k(u)\bar\beta(u)
L^{\mathrm{res}}_{T-u}\,\mathrm d u$. Since $\frac{1+\eta}{\alpha}C_t^{\mathrm{res}}
=\frac\eta\alpha-\chi_t$ and $\frac\eta\alpha\int_{\tau_k}^{\tau_{k+1}}r_k(u)\bar\beta(u)\,\mathrm d u =1-\widetilde m_k$,
the factor in \eqref{eq:app_deltak} satisfies the exact identity
\begin{equation}\label{eq:app_delta_identity}
 \delta_k^2
 =1-(1-\widetilde m_k)^2-2\widetilde m_kJ_k+Q_k^2.
\end{equation}
Furthermore,
\[
 L_t^{\mathrm{res}}
 \le L_\varepsilon^{\mathrm{res}}
 :=1+\frac{1+a(\varepsilon)\bar C_\varepsilon}
 {\gamma(\varepsilon)^\alpha},
 \qquad
 Q_k\le C_LI_k,
 \qquad
 C_L:=\frac{1+\eta}{\alpha}L_\varepsilon^{\mathrm{res}}.
\]
For every step, \eqref{eq:app_delta_identity} gives
\begin{equation}\label{eq:app_delta_transient}
 \delta_k\le1+(J_k)_-+\frac12C_L^2I_k^2,
\end{equation}
because $\delta_k^2\le1+2(J_k)_-+C_L^2I_k^2$ and
$\sqrt{1+y}\le1+y/2$ for $y\ge0$.

If the direct-time cell
$\mathcal I_k:=[T-\tau_{k+1},T-\tau_k]$ is contained in $[t_+,\infty)$,
then \eqref{eq:app_eventual_half} and \eqref{eq:app_exactint} give
$J_k\ge(1-\widetilde m_k)/2$. Hence $\delta_k^2\le \widetilde m_k+C_L^2I_k^2$. If
$(\eta/\alpha)I_k\le1$, then
$1-\widetilde m_k\ge(\eta/(2\alpha))I_k$, so
$\delta_k^2\le1-(\eta/(2\alpha))I_k+C_L^2I_k^2$. Since the right-hand
side is at least $1/2$ under the stated condition,
$\sqrt{1+y}\le1+y/2$ for $y\ge-1$ gives
\begin{equation}\label{eq:app_delta_late_linear}
 \delta_k\le1-\frac{\eta}{4\alpha}I_k+\frac12C_L^2I_k^2.
\end{equation}
As $\delta_k$ is deterministic, taking $L^p$ norms in \eqref{eq:app_Ak} gives $\left\lVert
A_k\right\rVert_{L^p}\le\delta_k\left\lVert D_k\right\rVert_{L^p}$.

\par\smallskip\noindent\textit{Step 2 (the residual $R^{(2)}_k$ and closing the recursion).}\, 
At $\bar X^{\theta,h}_{\tau_k}$, split the bracket in \eqref{eq:app_R2}
into a modeling part and a time-freezing part,
\begin{equation*}
h_\theta(\tau_k,\cdot)-S^{(\alpha),\mathrm{res}}_{T-u}(\cdot)=\underbrace{\alpha\bigl(s_\theta(\tau_k,\cdot)-S^{(\alpha)}_{T-\tau_k}(\cdot)\bigr)}_{\text{Assumption~\ref{ass:score}: }
\|\cdot\|_{L^p}\le\,\alpha\,\mathrm{err}}
+\underbrace{\alpha\bigl(S^{(\alpha)}_{T-\tau_k}-S^{(\alpha)}_{T-u}\bigr)(\cdot)}_{\text{Assumption~\ref{ass:time}}}
.
\end{equation*}
Assumption~\ref{ass:time} bounds the second part by $\alpha
L^{\mathrm{time}}_\varepsilon(u-\tau_k)(1+\left\lVert \bar
X^{\theta,h}_{\tau_k}\right\rVert)$, and $\left\lVert \bar
X^{\theta,h}_{\tau_k}\right\rVert_{L^p}\le\left\lVert
D_k\right\rVert_{L^p}+B_\varepsilon^{(p)}$ by Proposition~\ref{app:prop_moments}. Since
$u-\tau_k\le h\le I_k/\beta_{\min}$ and $r_k(u)\le1$, the
part proportional to $\left\lVert D_k\right\rVert_{L^p}$ is at most $C_2I_k^2\left\lVert
D_k\right\rVert_{L^p}$ with $C_2:=(1+\eta)L^{\mathrm{time}}_\varepsilon/\beta_{\min}$ (the
outer weight $\tfrac{1+\eta}{\alpha}$ and the factor $\alpha$ of $h_\theta$ combining into
$1+\eta$); the remaining, $D_k$-independent term is
$\mathrm{Source}_k:=(1+\eta)\int_{\tau_k}^{\tau_{k+1}}\bar\beta(u)\bigl[\mathrm{err}+L^{\mathrm{time}}_\varepsilon(u-\tau_k)(1+B_\varepsilon^{(p)})\bigr]\,\mathrm{d}
u$. Set $\widetilde\delta_k:=\delta_k+C_2I_k^2$ and $K_1:=\frac12C_L^2+C_2$. For every
step,
\eqref{eq:app_delta_transient} gives $\widetilde\delta_k \le\exp\bigl((J_k)_-+K_1I_k^2\bigr)$. For a cell contained in $[t_+,\infty)$, \eqref{eq:app_delta_late_linear}
instead gives $\widetilde\delta_k\le1$ provided $I_k\le I_\star
:=\min\left\{\frac\alpha\eta,\frac{\eta}{4\alpha K_1}\right\}$.

Let $\mathcal K_{\mathrm{tr}}$ contain the cells not entirely included in
$[t_+,\infty)$. Since the grid is uniform in direct time, for $h\le1$ their
union is contained in $[\varepsilon,t_++1]$. Consequently, $\sum_{k\in\mathcal
K_{\mathrm{tr}}}(J_k)_- \le\mathfrak A_\varepsilon$, and $\sum_{k\in\mathcal
K_{\mathrm{tr}}}I_k^2 \le h\,\Xi_\varepsilon$ with $\Xi_\varepsilon:=\beta(t_++1)
\bigl(B(t_++1)-B(\varepsilon)\bigr)$. Here we used $(J_k)_-\le\int_{\mathcal
I_k}\beta(t)(\chi_t)_-\,\mathrm d t$,
$I_k\le h\beta(t_++1)$ on transient cells, and disjointness of the cells.
Thus every partial product of the one-step factors satisfies
\begin{equation}\label{eq:app_grid_product}
 \prod_j\widetilde\delta_j
 \le\mathcal G_{\varepsilon,h}
 :=\exp\bigl(\mathfrak A_\varepsilon+K_1h\Xi_\varepsilon\bigr).
\end{equation}
Since $I_k\le h\beta(T)$, it is enough to impose
\begin{equation}\label{eq:app_h0}
 h\le h_0(\varepsilon,T)
 :=\min\left\{1,\frac{I_\star}{\beta(T)},
 \frac{1}{K_1\Xi_\varepsilon}\right\},
\end{equation}
with the convention that the last term is $+\infty$ if
$K_1\Xi_\varepsilon=0$. Then
$\mathcal G_{\varepsilon,h}\le e^{\mathfrak A_\varepsilon+1}$, independently
of $T$.

\par\smallskip\noindent\textit{Step 3 (the fluctuation $R^{(1)}_k$).}\, 
Integrating the exact backward SDE over one step,
\begin{equation}\label{eq:app_Yincrement}
\overleftarrow Y_u-\overleftarrow Y_{\tau_k}=\int_{\tau_k}^{u}\bar\beta(v)\Bigl[\tfrac1\alpha\overleftarrow Y_v+(1+\eta) S^{(\alpha)}_{T-v}(\overleftarrow Y_v)\Bigr]\,\mathrm{d} v+\int_{\tau_k}^{u}\bigl(\eta\bar\beta(v)\bigr)^{1/\alpha}\,\mathrm{d}\overleftarrow L{}^\alpha_v .
\end{equation}
The score--denoiser identity bounds the score at the origin, $\left\lVert
S^{(\alpha)}_t(0)\right\rVert=\tfrac{a}{\alpha\gamma^\alpha}\left\lVert
m_t(0)\right\rVert\le\tfrac{a}{\alpha\gamma^\alpha}R\le\tfrac{1+aC_t}{\alpha\gamma^\alpha}R$
(as $a\le1$), so its at-most-linear growth gives $\left\lVert
S^{(\alpha)}_{T-v}(\overleftarrow
Y_v)\right\rVert_{L^p}\le\tfrac{1+a(T-v)C_{T-v}}{\alpha\gamma(T-v)^\alpha}(B_\varepsilon^{(p)}+R)$.
The drift term of \eqref{eq:app_Yincrement} is therefore bounded in $L^p$ by
$\int_{\tau_k}^{u}\bar\beta(v)\bigl[\tfrac{B_\varepsilon^{(p)}}{\alpha}+\tfrac{1+\eta}{\alpha}\tfrac{1+a(T-v)C_{T-v}}{\gamma(T-v)^\alpha}(B_\varepsilon^{(p)}+R)\bigr]\,\mathrm{d}
v\le(B_\varepsilon^{(p)}+R)\bigl[\tfrac{\Delta
B_k}{\alpha}+\tfrac{1+\eta}{\alpha}\int_{\tau_k}^{\tau_{k+1}}\bar\beta(v)\tfrac{1+a(T-v)C_{T-v}}{\gamma(T-v)^\alpha}\,\mathrm{d}
v\bigr]$; the stochastic integral in \eqref{eq:app_Yincrement} is
$\mathrm{S}\alpha\mathrm{S}_d$ with scale
$(\eta\int_{\tau_k}^u\bar\beta)^{1/\alpha}\le(\eta\,\Delta B_k)^{1/\alpha}$, contributing
$c_{d,\alpha,p}\,\eta^{1/\alpha}\Delta B_k^{1/\alpha}$. Hence $\left\lVert \overleftarrow
Y_u-\overleftarrow Y_{\tau_k}\right\rVert_{L^p}\le\Theta_k$, and by \eqref{eq:app_R1},
\begin{equation}\label{eq:app_R1bound}
\left\lVert R^{(1)}_k\right\rVert_{L^p}\le\frac{1+\eta}{\alpha}\Bigl(\int_{\tau_k}^{\tau_{k+1}}\bar\beta(u)\,L^{\mathrm{res}}_{T-u}\,\mathrm{d} u\Bigr)\Theta_k .
\end{equation}

\par\smallskip\noindent\textit{Step 4 (telescoping).}\, 
Combining Steps 1--3,
$\left\lVert D_{k+1}\right\rVert_{L^p}\le\widetilde\delta_k\left\lVert
D_k\right\rVert_{L^p}+\left\lVert R^{(1)}_k\right\rVert_{L^p}+\mathrm{Source}_k$. Starting
from $\left\lVert D_0\right\rVert_{L^p}=0$ (both processes start at $\pi_\infty^\alpha$,
synchronously coupled) and using \eqref{eq:app_grid_product},
\[
\left\lVert D_N\right\rVert_{L^p}\le \mathcal G_{\varepsilon,h}\Bigl(
\sum_{k=0}^{N-1}\mathrm{Source}_k+\sum_{k=0}^{N-1}\left\lVert R^{(1)}_k\right\rVert_{L^p}\Bigr).
\] In the source sum, factor
$\mathrm{err}$ and
$L^{\mathrm{time}}_\varepsilon h(1+B_\varepsilon^{(p)})$ out of
$\sum_k\int_{\tau_k}^{\tau_{k+1}}\bar\beta
=\int_0^{T-\varepsilon}\bar\beta=B(T)-B(\varepsilon)$; this is the first
term inside braces in \eqref{eq:app_e2}. The second sum, bounded by
\eqref{eq:app_R1bound}, is the second term. Finally
$\mathcal{E}_2\le\left\lVert D_N\right\rVert_{L^p}$ since $W_p$ is an
infimum over couplings.
\end{proof}

\subsection{Assembly: the explicit bound}

\begin{theorem}[Explicit non-asymptotic bound]\label{app:thm_main}
Under Assumptions~\ref{ass:schedule}--\ref{ass:time} and~\ref{ass:denoiser}, for a target supported in $\bar B(0,R)$, $p\in[1,\alpha)$,
arbitrary $0<\varepsilon<T$, and $h\le h_0(\varepsilon,T)$,
\begin{align}
W_p\bigl(\pi_{\mathrm{data}},\hat\pi^{(\beta,\theta,\eta,\varepsilon)}_N\bigr)
&\le \mathcal E_0+\mathcal E_1+\mathcal E_2,\label{eq:app_main}\\
\mathcal E_0
&\le \bigl(1-a(\varepsilon)\bigr)
\mathbb{E}[\left\lVert X_0\right\rVert^p]^{1/p}
+\gamma(\varepsilon)c_{d,\alpha,p},\notag\\
\mathcal E_1
&\le W_p(\pi_{\mathrm{data}},\pi_\infty^\alpha)\notag\\[-2pt]
&\quad\times\exp\Bigl(-\tfrac{1+\eta}{\alpha}
\int_\varepsilon^T\beta(t)
\tfrac{1-a(t)C_t}{\gamma(t)^\alpha}\,\mathrm{d} t\Bigr),\notag\\
\mathcal E_2
&\le \mathcal G_{\varepsilon,h}\Bigg\{
(1+\eta)\bigl[\mathrm{err}+L^{\mathrm{time}}_\varepsilon
h(1+B_\varepsilon^{(p)})\bigr]\bigl(B(T)-B(\varepsilon)\bigr)\notag\\
&\qquad+\frac{1+\eta}{\alpha}\sum_{k=0}^{N-1}
\Bigl(\int_{\tau_k}^{\tau_{k+1}}\bar\beta(u)
L^{\mathrm{res}}_{T-u}\,\mathrm{d} u\Bigr)\Theta_k
\Bigg\}.\notag
\end{align}
Here $\Theta_k$, $B_\varepsilon^{(p)}$ and
$\mathcal G_{\varepsilon,h}$ are as in Proposition~\ref{app:prop_cumul}.
\end{theorem}

\begin{proof}
Combine the triangle decomposition \eqref{eq:decomp} with Lemmas~\ref{app:lem_bias}
and~\ref{app:lem_mixing} and Proposition~\ref{app:prop_cumul}.
\end{proof}

\subsection{Schedule corollaries and proof of Theorem~\ref{cor:rate}}\label{app:constant_rate}

We express \eqref{eq:app_main} through $B(t)=\int_0^t\beta$, with
$B_\varepsilon:=B(\varepsilon)$, $B_T:=B(T)$, the per-step injected noise $\Delta
B:=\max_k\int_{\tau_k}^{\tau_{k+1}}\bar\beta\le h\beta(T)$, and
$\mu_p:=\mathbb{E}[\left\lVert X_0\right\rVert^p]^{1/p}$. The change of variable
$\,\mathrm{d} B=\beta(t)\,\mathrm{d} t$, $\gamma(t)^\alpha=1-e^{-B(t)}$, gives
\begin{equation}\label{eq:app_changevar}
\frac{\beta(t)}{\gamma(t)^\alpha}\,\mathrm{d} t=\frac{\,\mathrm{d} B}{1-e^{-B}} .
\end{equation}

\begin{corollary}[General schedule, intrinsic scale]\label{app:cor_general}
Under the assumptions of Theorem~\ref{app:thm_main} and the resolution $\Delta B\le\min\{B_\varepsilon,1\}$, \begin{equation}\label{eq:app_general}
\begin{aligned}
W_p\bigl(\pi_{\mathrm{data}},\hat\pi^{(\beta,\theta,\eta,\varepsilon)}_N\bigr)
\le{}&c_0'\,B_\varepsilon^{1/\alpha}\\
&+c_1\,e^{-\rho_\star(B_T-B_\varepsilon)}
+c_2(B_T-B_\varepsilon)\mathrm{err}\\
&+c_3\,\frac{\Delta B}{B_\varepsilon}\\
&+c_4\left(B_T-B_\varepsilon
+\log\left(1+\frac1{B_\varepsilon}\right)\right)\Delta B^{1/\alpha},
\end{aligned}
\end{equation}
where, writing $w:=\frac{1+\eta}{\alpha}$,
$\bar w:=\max(1,w)$ and
$\mathcal G_\varepsilon:=e^{\mathfrak A_\varepsilon+1}$,
\[
\rho_\star=\frac{1+\eta}{\alpha},\qquad
c_0'=\mu_p+c_{d,\alpha,p},\qquad
c_1=e^{\mathfrak D_\varepsilon^{\mathrm{mix}}}
(\mu_p+c_{d,\alpha,p}),\qquad
c_2=\mathcal G_\varepsilon(1+\eta),
\] $c_3=8\,\mathcal G_\varepsilon w\,\bar w\, (B_\varepsilon^{(p)}+R)(2+\bar C_\varepsilon)^2$, \[
\begin{aligned}
c_4=\mathcal G_\varepsilon\Bigl[
&w\,\eta^{1/\alpha}c_{d,\alpha,p}(2+\bar C_\varepsilon)
+2w\,\bar w\,(B_\varepsilon^{(p)}+R)(2+\bar C_\varepsilon)^2\\
&+(1+\eta)L^{\mathrm{time}}_\varepsilon
(1+B_\varepsilon^{(p)})/\beta_{\min}
\Bigr].
\end{aligned}
\]
\end{corollary}

\begin{proof}
\par\smallskip\noindent\textit{Bias.}\, The global inequalities
\[
1-a(\varepsilon)\le\min\{B_\varepsilon/\alpha,1\}
\le B_\varepsilon^{1/\alpha},
\qquad
\gamma(\varepsilon)\le\min\{B_\varepsilon,1\}^{1/\alpha}
\le B_\varepsilon^{1/\alpha}
\]
and Lemma~\ref{app:lem_bias} yield $\mathcal{E}_0\le c_0'B_\varepsilon^{1/\alpha}$.

\par\smallskip\noindent\textit{Mixing.}\,
By \eqref{eq:app_chi_direct}, $\left(\frac\eta\alpha-\chi_t\right)_+ =\frac{1+\eta}{\alpha}
\left(\frac{a(t)C_t-a(t)^\alpha}{\gamma(t)^\alpha}\right)_+ \le\frac{1+\eta}{\alpha}
\frac{a(t)C_t}{\gamma(t)^\alpha}$. Assumption~\ref{ass:denoiser}, the monotonicity of $\gamma$, and $\int_\varepsilon^\infty\beta(t)a(t)\,\mathrm d t=\alpha a(\varepsilon)$ yield
\begin{equation}\label{eq:app_compact_mix_deficit}
 \mathfrak D_\varepsilon^{\mathrm{mix}}
 \le(1+\eta)\bar C_\varepsilon
 \frac{a(\varepsilon)}{\gamma(\varepsilon)^\alpha}
 <\infty.
\end{equation}
Lemma~\ref{app:lem_mixing} gives the mixing term in
\eqref{eq:app_general}, while the independent coupling gives
$W_p(\pi_{\mathrm{data}},\pi_\infty^\alpha)
\le\mu_p+c_{d,\alpha,p}$.

\par\smallskip\noindent\textit{Score.}\,
The first term inside braces in \eqref{eq:app_e2} is
$(1+\eta)[\mathrm{err}+L^{\mathrm{time}}_\varepsilon
h(1+B_\varepsilon^{(p)})](B_T-B_\varepsilon)$.
After multiplication by $\mathcal G_{\varepsilon,h}\le\mathcal G_\varepsilon$,
the modeling part gives the coefficient $c_2$. The temporal-regularity part,
bounded via $h\le\Delta B/\beta_{\min}$ and
$\Delta B\le\Delta B^{1/\alpha}$ because $\Delta B\le1$, is absorbed into
the $c_4(B_T-B_\varepsilon)\Delta B^{1/\alpha}$ term below.

\par\smallskip\noindent\textit{Discretization.}\, Under compact support, $a(t)C_t\le\bar C_\varepsilon$, so, using $\gamma(t)^\alpha\le1$ and $\alpha>1$, both $L^{\mathrm{res}}_{t}=1+\frac{1+a(t)C_t}{\gamma(t)^\alpha}$ and $\frac1\alpha+\frac{1+a(t)C_t}{\gamma(t)^\alpha}$ are at most $\frac{2+\bar C_\varepsilon}{\gamma(t)^\alpha}$. Passing to direct time $t=T-u$ and using \eqref{eq:app_changevar}, each one-step regularity factor is bounded by $(2+\bar C_\varepsilon)H_k$, where $H_k:=\int_{B(t_{k+1})}^{B(t_k)}\frac{\,\mathrm{d} B}{1-e^{-B}}$, $t_k:=T-\tau_k$; the integration intervals partition $[B_\varepsilon,B_T]$. Since $e^{B_\varepsilon}\ge1+B_\varepsilon$ implies $\frac1{1-e^{-B_\varepsilon}}\le1+\frac1{B_\varepsilon}$, the antiderivative $B+\log(1-e^{-B})$ gives the global estimate $\sum_kH_k=\int_{B_\varepsilon}^{B_T}\frac{\,\mathrm{d} B}{1-e^{-B}} \le B_T-B_\varepsilon+\log\left(1+\frac1{B_\varepsilon}\right)$. Moreover $\Theta_k\le\bar w\,(B_\varepsilon^{(p)}+R)(2+\bar C_\varepsilon)H_k+c_{d,\alpha,p}\eta^{1/\alpha}\Delta B_k^{1/\alpha}$, since $\frac{\Delta B_k}{\alpha}+w\int_{\tau_k}^{\tau_{k+1}}\bar\beta\,\frac{1+aC}{\gamma^\alpha}\le\Delta B_k+w(1+\bar C_\varepsilon)H_k\le\bar w(2+\bar C_\varepsilon)H_k$ ($\Delta B_k\le H_k$, the integrand of $H_k$ being $\ge1$, and $\frac1\alpha\le1$); hence
\[
w\sum_k\Bigl(\int\bar\beta L^{\mathrm{res}}_{}\Bigr)\Theta_k\le w\,\bar
w\,(B_\varepsilon^{(p)}+R)(2+\bar
C_\varepsilon)^2\sum_kH_k^2+w\,\eta^{1/\alpha}c_{d,\alpha,p}(2+\bar
C_\varepsilon)\sum_kH_k\Delta B_k^{1/\alpha}.
\]
\emph{Drift part.} The inequality $\frac{1}{1-e^{-B}}\le1+\frac1B$ gives $H_k\le\Delta B_k(1+\frac{1}{B(t_{k+1})})$. The resolution $\Delta B\le B_\varepsilon$ entails $B(t_k)\le2B(t_{k+1})$, so $\frac{1}{B(t_{k+1})^2}\le\frac{4}{B(t_k)^2}$, and the monotonicity of $B\mapsto B^{-2}$ gives $\frac{\Delta B_k}{B(t_k)^2}\le\int_{B(t_{k+1})}^{B(t_k)}\frac{\,\mathrm{d} B}{B^2}$. Thus $\sum_kH_k^2\le2\Delta B\Bigl[(B_T-B_\varepsilon)+4\int_{B_\varepsilon}^{B_T}\frac{\,\mathrm{d} B}{B^2}\Bigr]\le2\Delta B\,(B_T-B_\varepsilon)+\frac{8\Delta B}{B_\varepsilon}$. The first piece, at most $2\Delta B^{1/\alpha}(B_T-B_\varepsilon)$ because $\Delta B\le1$,
joins the noise term; the second gives $c_3\Delta B/B_\varepsilon$.
\emph{Noise part.} The global integral estimate above gives $\sum_kH_k\Delta B_k^{1/\alpha} \le\Delta B^{1/\alpha} \left(B_T-B_\varepsilon+\log\left(1+\frac1{B_\varepsilon}\right)\right)$. Summing all five contributions and multiplying the score--discretization terms
by $\mathcal G_{\varepsilon,h}\le\mathcal G_\varepsilon$ gives
\eqref{eq:app_general} with the stated constants.
\end{proof}

\begin{proof}[\normalfont\bfseries Proof of Theorem~\ref{cor:rate}]
Set $\rho_\star=\beta_0\frac{1+\eta}{\alpha},\qquad
c_0=\beta_0^{1/\alpha}\bigl(\mu_p+c_{d,\alpha,p}\bigr)$. Since
Theorem~\ref{cor:rate} only requires $h$ to be small enough, decrease its
admissible threshold so that
$h\le\min\{h_0(\varepsilon,T),\varepsilon,\beta_0^{-1}\}$. Then
$\Delta B=\beta_0h\le\min\{B_\varepsilon,1\}, \qquad
B_\varepsilon=\beta_0\varepsilon$, so Corollary~\ref{app:cor_general} applies with
$B_T=\beta_0T$ and
$\beta_{\min}=\beta_0$. Its bias term is
$c_0\varepsilon^{1/\alpha}$. The mixing exponent is
$\rho_\star(T-\varepsilon)$; the score coefficient is
$\beta_0c_2$, and the drift term is
$c_3\beta_0h/(\beta_0\varepsilon)=c_3h/\varepsilon$. Finally, $\left(\beta_0(T-\varepsilon)
+\log\left(1+\frac1{\beta_0\varepsilon}\right)\right) (\beta_0h)^{1/\alpha} \le c_{\beta_0}
\left(1+T+\log\left(1+\frac1\varepsilon\right)\right)h^{1/\alpha}$ for a constant
$c_{\beta_0}<\infty$ depending only on $\beta_0$. After absorbing this factor
into $c_4$, we obtain \eqref{eq:rate}.
The constants $c_3,c_4$ inherited from Corollary~\ref{app:cor_general}
depend on $\varepsilon$ through $\bar C_\varepsilon$ and
$L^{\mathrm{time}}_\varepsilon$, and $h_0$ of \eqref{eq:app_h0} depends on
$(\varepsilon,T)$, including any horizon dependence of
$L^{\mathrm{time}}_\varepsilon$. At fixed $(\varepsilon,T)$, the
discretization contribution is $\mathcal O(h^{1/\alpha})$.
A joint accuracy or step-complexity bound also requires balancing
early stopping, initialization, and score error.
\end{proof}

\section{Proof of Proposition~\ref{prop:ct_decay}: Denoiser Regularity under Compact Support}\label{app:compact_decay}

The following properties of the isotropic stable kernel will also be used
in the heavy-tailed analysis. Writing $\varphi_{\alpha,d}=\varphi_\alpha$, the stable
specialization of \citet[Corollary~23]{BogdanGrzywnyRyznar2014} gives
$\varphi_{\alpha,d}(x)\asymp(1+\|x\|)^{-d-\alpha}$.
Moreover, \citet[Theorem~3]{KaletaSztonyk2015}, applied to the stable
L\'evy measure, gives
$|\partial^j \varphi_{\alpha,d}(x)|\le C(1+\|x\|)^{-d-\alpha}$ for $|j|\le2$.
Thus, with operator norm for matrices,
\[
 s:=\nabla\log \varphi_\alpha,\qquad
 \Lambda_1:=\sup_x\|s(x)\|<\infty,\qquad
 \Lambda_2:=\sup_x\|\nabla s(x)\|<\infty,\qquad \Lambda:=\Lambda_2.
\]
For the sharper derivative bound needed below, let $S$ be positive
$\alpha/2$-stable with $\mathbb{E}e^{-vS}=e^{-v^{\alpha/2}}$.
Gaussian subordination yields
\[
 \varphi_{\alpha,d}(r)=\mathbb{E}\bigl[(4\pi S)^{-d/2}e^{-r^2/(4S)}\bigr],
 \qquad \varphi_{\alpha,d}'(r)=-2\pi r \varphi_{\alpha,d+2}(r),\quad r>0.
\]
Differentiation is justified by domination on compact subsets of
$(0,\infty)$. The density bound in dimension $d+2$ gives
$|\varphi_{\alpha,d}'(r)|\le Cr(1+r)^{-d-\alpha-2}$ and radial monotonicity.
Finally, \citet[Lemma~2.1]{HultLindskog2007} gives
$\mathbb{P}(\|L_1^\alpha\|>r)\sim C_{d,\alpha}r^{-\alpha}$, where
$C_{d,\alpha}:=\nu_\alpha(\|z\|>1)$.
Radial monotonicity and comparison of annuli with radii $r$ and
$\lambda r$, followed by $\lambda\downarrow1$, give
\[
 \varphi_\alpha(re)\sim k_{d,\alpha}r^{-d-\alpha},\qquad
 k_{d,\alpha}=\frac{\alpha C_{d,\alpha}}{\omega_{d-1}},\qquad
 e\in\mathbb S^{d-1},\quad\omega_{d-1}=|\mathbb S^{d-1}|.
\]
All kernel constants depend only on $d$ and $\alpha$.

Assume now $R:=\sup_{x_0\in\mathrm{supp}\,\pi_{\mathrm{data}}}\|x_0\|<\infty$.

\begin{proof}[\normalfont\bfseries Proof of Proposition~\ref{prop:ct_decay}]
Let $\mu_{t,x}$ be the conditional law of $X_0$ given $X_t=x$ and set
$a=a(t)$, $\gamma=\gamma(t)$, $U=(x-aX_0)/\gamma$ and
$s=\nabla\log \varphi_\alpha$.
Differentiating the conditional-mean quotient, justified by compact support, gives
\begin{equation}\label{eq:app_ct_covid}
 \nabla m_t(x)=\gamma^{-1}
      \mathrm{Cov}_{\mu_{t,x}}(X_0,s(U)).
\end{equation}
Subtract $s((x-am_t(x))/\gamma)$ from the second covariance argument.
Since $s$ is $\Lambda$-Lipschitz,
\[
 \|\nabla m_t(x)\|_{\mathrm{op}}
 \le\frac{\Lambda a}{\gamma^2}
       \mathbb{E}_{\mu_{t,x}}\|X_0-m_t(x)\|^2
 \le\frac{\Lambda aR^2}{\gamma^2}
 \le\frac{d\Lambda aR^2}{\gamma^2}.
\]
This implies \eqref{eq:ct_decay}. Since $a\to0$ and $\gamma\to1$,
$C_t\to0$, $C_t^{(\alpha)}\to1/\alpha$, and
Assumption~\ref{ass:denoiser} holds. Lemma~\ref{app:lem_transient}
gives the remaining conclusion.
\end{proof}

\section{Proofs for Sections~\ref{sec:heavy} and~\ref{sec:global_comparison}: Heavy-Tailed Targets}\label{app:separation}

Assumption~\ref{ass:heavy} implies that, for suitable constants
$0<c_{\pi,-}\le c_{\pi,+}$ and $R_\pi\ge1$,
\begin{equation}\label{eq:a5prime}
c_{\pi,-}\lVert x\rVert^{-(d+\nu)}
\;\le\;\pi(x)\;\le\;
c_{\pi,+}\lVert x\rVert^{-(d+\nu)},
\qquad \lVert x\rVert\ge R_\pi,
\tag{\ref*{ass:heavy}$'$}
\end{equation}
because the asymptotic relation \eqref{eq:a5} is uniform in direction. We use
this density envelope throughout the appendix; the full asymptotic relation is
invoked only where an exact limit is required.

Write $\pi$ for the target density, $\mu_x$ for the conditional law of
$X_0$ given $X_t=x$, and $\nu_x$ for its image under
$u=(x-ax_0)/\gamma$. We use the notation $s,\Lambda_1,\Lambda_2$ and the
stable-kernel estimates from Appendix~\ref{app:compact_decay}, in particular
\begin{equation*}
c_\varphi(1+\left\lVert u\right\rVert)^{-(d+\alpha)}\;\le\;\varphi_\alpha(u)\;\le\;C_\varphi(1+\left\lVert
u\right\rVert)^{-(d+\alpha)} .
\end{equation*}

\subsection{Proof of Theorem~\ref{thm:main_ht}: Heavy-tailed bound}

The heavy-tailed transfer requires a bound on $\|m_t(0)\|$ together with the
uniform regularity and decay estimates of Proposition~\ref{prop:heavy_denoiser}.

\begin{lemma}[Uniform bound of $m_t$ at zero]\label{app:lem_mtzero}
Under Assumption~\ref{ass:heavy}, $\left\lVert
m_t(0)\right\rVert\le\mathbb{E}_\pi[\left\lVert X_0\right\rVert]=:M_1<\infty$ for all $t$.
\end{lemma}

\begin{proof}
At $x=0$ the conditional law is weighted by
$w(R)=\varphi_\alpha(aR/\gamma)$, $R=\|X_0\|$, a nonincreasing function.
For an independent copy $R'$, the identity
$2\mathrm{Cov}(R,w(R))
=\mathbb{E}[(R-R')(w(R)-w(R'))]\le0$
gives
$\|m_t(0)\|\le\mathbb{E}[Rw(R)]/\mathbb{E}[w(R)]\le\mathbb{E}R$.
\end{proof}

\begin{proof}[\normalfont\bfseries Proof of Theorem~\ref{thm:main_ht}]
Use Lemma~\ref{app:lem_mtzero} in Step~3 of
Proposition~\ref{app:prop_cumul}. Since $p<\alpha<\nu$ also at $\nu=1+\alpha$,
$\mu_p<\infty$ and controls the moment terms in the bound.
Proposition~\ref{prop:heavy_denoiser} verifies
Assumption~\ref{ass:denoiser}, including at this boundary, and Lemma~\ref{app:lem_transient} gives
$\mathfrak A_\varepsilon<\infty$ and an eventual-contractivity time $t_+$
independent of $T$, while \eqref{eq:app_compact_mix_deficit} gives $\mathfrak D_\varepsilon^{\mathrm{mix}}<\infty$.
Therefore Proposition~\ref{app:prop_moments}, the cumulative bound and the
refined mixing estimate \eqref{eq:app_mixing_optimal} all apply, with $R$
replaced by $M_1$.
\end{proof}

\subsection{Proofs of Propositions~\ref{prop:heavy_denoiser} and~\ref{prop:brown_ct_divergence}: Denoiser regularity}

Spatial asymptotics give lower bounds on $C_t$; global upper bounds require separate estimates of the Jacobian.

\paragraph{Asymptotics of the denoiser $m_s$.} For $s\in[\varepsilon,T]$ let $f_s(u):=a(s)^{-d}\pi\bigl(u/a(s)\bigr)$ and $g_s(v):=\gamma(s)^{-d}\varphi_\alpha\bigl(v/\gamma(s)\bigr)$ denote the densities of $a(s)X_0$ and of $\gamma(s)Z$, and set
\begin{equation}\label{eq:app_DH}
\mathsf D_s:=f_s*g_s,
\qquad
\mathsf H_s(x):=\int_{\mathbb{R}^d}u\,f_s(u)\,g_s(x-u)\,\mathrm{d} u ,
\end{equation}
so that $\mathsf D_s$ is the density of $X_s$ and, by Bayes' formula,
\begin{equation}\label{eq:app_denoiser_ratio}
a(s)\,m_s(x)=\mathbb{E}\bigl[a(s)X_0\mid X_s=x\bigr]=\frac{\mathsf H_s(x)}{\mathsf D_s(x)} .
\end{equation}
Since $a(s)$ and $\gamma(s)$ range in compact subsets of $(0,1]$ as $s$ ranges over
$[\varepsilon,T]$, \eqref{eq:a5} and $\varphi_\alpha(z)\sim k_{d,\alpha}\left\lVert
z\right\rVert^{-(d+\alpha)}$ give, uniformly in $s\in[\varepsilon,T]$ and $e\in\mathbb
S^{d-1}$,
\begin{equation}\label{eq:app_fg_tail}
f_s(re)=\bigl(c_\pi a(s)^{\nu}+o(1)\bigr)r^{-(d+\nu)},
\qquad
g_s(re)=\bigl(k_{d,\alpha}\gamma(s)^{\alpha}+o(1)\bigr)r^{-(d+\alpha)},
\end{equation}
together with the uniform global envelopes
\begin{equation}\label{eq:app_env}
f_s(y)\le C\bigl(1+\left\lVert y\right\rVert\bigr)^{-(d+\nu)},
\qquad
g_s(y)\le C\bigl(1+\left\lVert y\right\rVert\bigr)^{-(d+\alpha)},
\end{equation}
and, because $\nu>1$, the uniform first moment
$\sup_{s\in[\varepsilon,T]}\int_{\mathbb{R}^d}\left\lVert u\right\rVert f_s(u)\,\mathrm{d}
u<\infty$.

The unweighted convolution principle is classical
\citep[Theorem~1.1]{BinghamGoldieOmey2006}; related conditional-mean
regimes for positive scalar variables are studied by
\citet[Section~3]{DenuitOrtegaRobert2025}. We give the weighted
$\mathbb R^d$ argument, including the uniformity in $s$ needed below.

\begin{lemma}[Two-centre convolution asymptotics]\label{app:lem_conv}
Uniformly in $s\in[\varepsilon,T]$ and $e\in\mathbb S^{d-1}$, as $r\to\infty$,
\begin{equation}\label{eq:app_D_asymp}
\mathsf D_s(re)=f_s(re)+g_s(re)+o\bigl(f_s(re)+g_s(re)\bigr),
\end{equation}
and, for the numerator $\mathsf H_s$ of the denoiser, the following two alternatives hold,
again uniformly in $s$ and $e$:
\begin{align}
\mathsf H_s(re)&=re\,f_s(re)+o\bigl(r f_s(re)\bigr), && 1<\nu<1+\alpha,\label{eq:app_H_asymp}\\
\left\lVert \mathsf H_s(re)\right\rVert&=\mathcal O\bigl(g_s(re)\bigr), && \nu\ge1+\alpha.\label{eq:app_H_bound}
\end{align}
\end{lemma}

\begin{proof}
Fix $x=re$ and split the convolution into
$A=\{\|u\|\le r/2\}$, $B=\{\|x-u\|\le r/2\}$ and their complement
$C$; $A\cap B$ is Lebesgue-null.
Uniformly in the time parameter and direction, the tail equivalents and
envelopes in \eqref{eq:app_fg_tail}--\eqref{eq:app_env} imply
\[
 \int_A f_s(u)g_s(x-u)\,\mathrm{d} u\sim g_s(x),\qquad
 \int_B f_s(u)g_s(x-u)\,\mathrm{d} u\sim f_s(x).
\]
Indeed, on $A$ the ratio $g_s(x-u)/g_s(x)$ is uniformly bounded and
tends to one for bounded $u$; the common integrable envelope for $f_s$
gives uniform dominated convergence. The argument on $B$ is identical
after setting $v=x-u$. On $C$, scaling $u=rw$ gives the bound
$Cr^{-d-\nu-\alpha}$, proving \eqref{eq:app_D_asymp}.
For the weighted numerator the same argument on $B$ gives
$\int_Bu f_s(u)g_s(x-u)\,\mathrm{d} u=xf_s(x)+o(rf_s(x))$,
since $\|(x-v)/r\|\le3/2$ there.
The contribution from $A$ is $O(g_s(x))$ by the common first-moment
bound, and that from $C$ is $O(r^{1-d-\nu-\alpha})$.
For $\nu<1+\alpha$ these two terms are $o(rf_s(x))$; for
$\nu\ge1+\alpha$, all three terms are $O(g_s(x))$.
This proves \eqref{eq:app_H_asymp}--\eqref{eq:app_H_bound}, including
the boundary $\nu=1+\alpha$, with the required uniformity.
\end{proof}

\begin{proposition}[Spatial asymptotic regimes of the denoiser]\label{app:prop_denoiser}
Under Assumption~\ref{ass:heavy} with $\nu>1$, the following hold uniformly in $s\in[\varepsilon,T]$ and
$e\in\mathbb S^{d-1}$.
\begin{itemize}\itemsep2pt
\item[(i)] If $1<\nu<\alpha$, then $m_s(re)=a(s)^{-1}re+o(r)$.
\item[(ii)] If $\nu=\alpha$, then $m_s(re)=\varkappa_{\alpha,s}\,re+o(r)$ with
$\varkappa_{\alpha,s}=\dfrac{c_\pi a(s)^{\alpha-1}}{c_\pi
a(s)^{\alpha}+k_{d,\alpha}\gamma(s)^{\alpha}}$.
\item[(iii)] If $\alpha<\nu<1+\alpha$, then
$m_s(re)=\Bigl(\dfrac{c_\pi
a(s)^{\nu-1}}{k_{d,\alpha}\gamma(s)^{\alpha}}+o(1)\Bigr)r^{\alpha+1-\nu}e$;
in particular $m_s$ is unbounded, but sublinear since $\alpha+1-\nu\in(0,1)$.
\item[(iv)] If $\nu\ge1+\alpha$, then $\sup_{s\in[\varepsilon,T]}\sup_{x}\left\lVert m_s(x)\right\rVert<\infty$; the boundary $\nu=1+\alpha$ is included.
\end{itemize}
Consequently, setting
\begin{equation}\label{eq:app_varkappa}
\varkappa_{\nu,s}:=
\begin{cases}
a(s)^{-1}, & 1<\nu<\alpha,\\[1mm]
\dfrac{c_\pi a(s)^{\alpha-1}}{c_\pi a(s)^{\alpha}+k_{d,\alpha}\gamma(s)^{\alpha}}, & \nu=\alpha,\\[3mm]
0, & \nu>\alpha,
\end{cases}
\end{equation}
one has the unified decomposition $m_s(x)=\varkappa_{\nu,s}\,x+\mathsf r_s(x)$ with a
\emph{uniformly sublinear} remainder,
\begin{equation}\label{eq:app_sublin}
\lim_{R\to\infty}\ \sup_{s\in[\varepsilon,T]}\ \sup_{\left\lVert x\right\rVert\ge R}\frac{\left\lVert \mathsf r_s(x)\right\rVert}{\left\lVert x\right\rVert}=0 .
\end{equation}
Moreover $m_s$ is $C^1$ in $x$, locally Lipschitz uniformly on compact space--time
cylinders, and of uniform linear growth, $\sup_{s\in[\varepsilon,T]}\left\lVert
m_s(x)\right\rVert\le C(1+\left\lVert x\right\rVert)$.
\end{proposition}

\begin{proof}
By \eqref{eq:app_denoiser_ratio}, $m_s=a(s)^{-1}\mathsf H_s/\mathsf D_s$.
In Lemma~\ref{app:lem_conv}, the denominator is asymptotic to $f_s$ if
$\nu<\alpha$, to $f_s+g_s$ if $\nu=\alpha$, and to $g_s$ if
$\nu>\alpha$. Substitution of \eqref{eq:app_fg_tail} gives (i)--(iii);
for $\nu\ge1+\alpha$ the numerator bound gives (iv) outside a fixed ball.
Lemma~\ref{app:lem_ct_measurable} gives joint continuity of $m_s$ and
$\nabla m_s$, hence boundedness on compact cylinders and the remaining
part of (iv). The uniform equivalents imply \eqref{eq:app_sublin};
together with continuity on the complementary compact set, they also
give uniform linear growth. Joint continuity of the Jacobian gives the
stated local Lipschitz property.
\end{proof}

\begin{lemma}[Covariance identity]\label{app:lem_covid}
Under Assumption~\ref{ass:heavy}, for all $x$ and $t>0$, $m_t\in C^1$ and $\nabla
m_t(x)=-\frac1a\,\mathrm{Cov}_{\nu_x}\bigl(U,\,s(U)\bigr)$.
\end{lemma}

\begin{proof}
The differentiation in \eqref{eq:app_ct_covid} remains valid under the
finite first moment, by boundedness of $\nabla \varphi_\alpha$.
The affine relation $X_0=(x-\gamma U)/a$ therefore gives
$\nabla m_t(x)=-a^{-1}\mathrm{Cov}_{\nu_x}(U,s(U))$.
\end{proof}

\begin{proposition}[Uniform conditional moment]\label{app:prop_condmoment}
Under Assumption~\ref{ass:heavy}, for $1\le p\le\nu-\alpha$ and $0<\gamma_0\le\gamma_1$, $M^{(p)}_*:=\sup\bigl\{\mathbb{E}_{\mu_x}\left\lVert X_0\right\rVert^p:\ x\in\mathbb{R}^d,\ a\in(0,1),\ \gamma\in[\gamma_0,\gamma_1]\bigr\}<\infty$. \end{proposition}

\begin{proof}
Let $K=d+\alpha$ and omit the common likelihood factor $\gamma^{-d}$.
Uniformly in the stated parameters, the conditional normalizer satisfies
\[
 Z(x)=\int\pi(y)\varphi_\alpha((x-ay)/\gamma)\,\mathrm{d} y
       \ge c(1+\|x\|)^{-K},
\]
by restricting to $\|y\|\le1$. The numerator is at most $C M_p$ on
bounded sets of $x$, where $M_p=\mathbb{E}_\pi\|X_0\|^p<\infty$.
For $r=\|x\|\ge\max(2R_\pi,2)$, split it at $\|ay\|=r/2$.
The near contribution is at most $Cr^{-K}M_p$.
On the far region the power-law envelope gives
$\|y\|^p\pi(y)\le C(r/a)^{p-d-\nu}$; consequently
\begin{equation}\label{eq:app_pm_far}
 N_{p,\mathrm{far}}(x)
 \le C(r/a)^{p-d-\nu}
      \int_{\mathbb{R}^d}\varphi_\alpha((x-ay)/\gamma)\,\mathrm{d} y
 \le C a^{\nu-p}r^{p-d-\nu}.
\end{equation}
Dividing by $Z(x)$ bounds the conditional moment by
$C(M_p+a^{\nu-p}r^{p+\alpha-\nu})$, uniformly bounded because
$p+\alpha-\nu\le0$. This includes $p=\nu-\alpha$.
\end{proof}

\begin{proof}[\normalfont\bfseries Proof of Proposition~\ref{prop:heavy_denoiser}]
\par\smallskip\noindent\textit{Step 1 (spatial lower bounds).}\,
For $\nu<\alpha$, Proposition~\ref{app:prop_denoiser} gives
\[
 C_t\ge\frac{\|m_t(re)-m_t(0)\|}{r}\longrightarrow a(t)^{-1},
\]
which proves (i) without differentiating a spatial equivalent.
The same argument gives $C_t\ge\varkappa_{\alpha,t}$ at $\nu=\alpha$.

\par\smallskip\noindent\textit{Step 2 (lack of global regularity when $\alpha\le\nu<1+\alpha$).}\,
Fix such a $\nu$, choose
\[
\max\{0,2-\nu\}<\vartheta<1+\alpha-\nu,
\qquad 0<\lambda<\tfrac12,
\]
and, in dimension one, let
\begin{equation}\label{eq:app_osc_target}
\pi(y)=c(1+y^2)^{-(1+\nu)/2}
\left[1+\lambda(1+y^2)^{-\vartheta/2}\cos y\right],
\end{equation}
where $c>0$ normalizes the density. This density is smooth, positive, even, and satisfies Assumption~\ref{ass:heavy} with $c_\pi=c$.

We first record an elementary convolution estimate. Let $g$ be a one-dimensional symmetric stable density of index $\alpha$, and let $h$ satisfy
\[
\int h=0,\qquad \int |u h(u)|\,\mathrm d u<\infty,
\qquad |h(u)|\le C(1+|u|)^{-k},\qquad 2<k<2+\alpha.
\]
If $h(x)=b x^{-k}\sin(\omega x)+o(x^{-k})$ as $x\to+\infty$, with $\omega>0$, then
\begin{equation}\label{eq:app_osc_convolution}
(h*g)(x)=b\widehat g(\omega)x^{-k}\sin(\omega x)+o(x^{-k}).
\end{equation}
Indeed, on $|u|\le x/2$, subtract $g(x)$ and use $|g'(z)|\le C(1+|z|)^{-2-\alpha}$, established for the stable kernel in Appendix~\ref{app:compact_decay}. The resulting integral is $\mathcal O(x^{-2-\alpha})$; the subtracted term is $\mathcal O(x^{-k-\alpha})$ because $\int h=0$. On $|x-u|\le x/2$, set $v=x-u$. After multiplication by $x^k$, the integrand is bounded by $C g(v)$, and the expansion of $h(x-v)$ gives
$b\int\sin(\omega(x-v))g(v)\,\mathrm d v
=b\widehat g(\omega)\sin(\omega x)$,
with an error tending to zero by dominated convergence. On the remaining region $|u|,|x-u|>x/2$, the two envelopes give $\mathcal O(x^{-k-\alpha})$. These bounds prove \eqref{eq:app_osc_convolution}.

Now fix any $t>0$, abbreviate $a=a(t)$, $\gamma=\gamma(t)$, and use $f=f_t$, $g=g_t$, $\mathsf D=\mathsf D_t$, $\mathsf H=\mathsf H_t$ from \eqref{eq:app_DH}. Set
\[
A=ca^\nu,\qquad B=c\lambda a^{\nu+\vartheta},
\qquad G=k_{1,\alpha}\gamma^\alpha,
\qquad \omega=a^{-1},\qquad
\phi=\widehat g(\omega)=e^{-\gamma^\alpha\omega^\alpha}>0.
\]
Direct differentiation of \eqref{eq:app_osc_target} after rescaling gives
\[
f'(x)=-\omega B x^{-1-\nu-\vartheta}\sin(\omega x)
       +\mathcal O(x^{-2-\nu}),\qquad
(xf(x))'=-\omega B x^{-\nu-\vartheta}\sin(\omega x)
       +\mathcal O(x^{-1-\nu}).
\]
Both derivatives have integral zero and a finite absolute first moment, since $\nu+\vartheta>2$. Moreover, $\vartheta<1$ and $\nu+\vartheta<1+\alpha$, so \eqref{eq:app_osc_convolution} applies with $k=1+\nu+\vartheta$ and $k=\nu+\vartheta$, respectively. Differentiating the convolutions therefore yields
\begin{align*}
\mathsf D'(x)&=-\omega B\phi\,x^{-1-\nu-\vartheta}\sin(\omega x)
               +o(x^{-1-\nu-\vartheta}),\\
\mathsf H'(x)&=-\omega B\phi\,x^{-\nu-\vartheta}\sin(\omega x)
               +o(x^{-\nu-\vartheta}).
\end{align*}
Meanwhile, Lemma~\ref{app:lem_conv} gives
\[
\mathsf H(x)\sim A x^{-\nu},\qquad
\mathsf D(x)\sim
\begin{cases}
(A+G)x^{-1-\alpha},&\nu=\alpha,\\
Gx^{-1-\alpha},&\nu>\alpha.
\end{cases}
\]
Substitution into $m_t'=a^{-1}(\mathsf H'\mathsf D-\mathsf H\mathsf D')/\mathsf D^2$ gives
\begin{equation*}
m_t'(x)=-K_t x^{1+\alpha-\nu-\vartheta}\sin(\omega x)
         +o(x^{1+\alpha-\nu-\vartheta}),
\qquad
K_t=\frac{\omega B\phi}{a}
\begin{cases}
G/(A+G)^2,&\nu=\alpha,\\
1/G,&\nu>\alpha.
\end{cases}
\end{equation*}
Since $K_t>0$ and $1+\alpha-\nu-\vartheta>0$, evaluating along $\omega x=\pi/2+2\pi n$ proves $C_t=\infty$. The density \eqref{eq:app_osc_target} is independent of $t$, so this holds for every $t>0$, proving (ii).

This failure is not forced by the matched index itself: if $X_0\sim\mathrm S\alpha\mathrm S_1(\sigma)$, the fractional Tweedie formula gives
\[
m_t(x)=\frac{\sigma^\alpha a(t)^{\alpha-1}}
{\sigma^\alpha a(t)^\alpha+\gamma(t)^\alpha}\,x,
\qquad
C_t=\frac{\sigma^\alpha a(t)^{\alpha-1}}
{\sigma^\alpha a(t)^\alpha+\gamma(t)^\alpha}\longrightarrow0.
\]

\par\smallskip\noindent\textit{Step 3 (uniform decay for $\nu\ge1+\alpha$).}\,
Fix $\varepsilon>0$. Proposition~\ref{app:prop_condmoment} at $p=1$
gives $\sup_{t\ge\varepsilon,x}\mathbb{E}_{\mu_x}\|X_0\|\le M_*<\infty$.
Also $M_2<\infty$, since $\nu\ge1+\alpha>2$.
With $a=a(t)$, $\gamma=\gamma(t)$ and conditional density $w_x$, the
covariance identity gives
\begin{equation}\label{eq:app_cov_split}
 \nabla m_t(x)=\gamma^{-1}\int(y-m_t(x))
 [s((x-ay)/\gamma)-s(x/\gamma)]^\top w_x(y)\,\mathrm{d} y.
\end{equation}
For bounded $x$, and also on $\|ay\|\le\|x\|/2$ when $\|x\|$ is large,
the stable envelopes imply $w_x(y)\le C_\varepsilon\pi(y)$.
The Lipschitz bound on $s$ therefore bounds this contribution by
$C_\varepsilon a(M_2+M_*M_1)$.
On the remaining region, $r=\|x\|\ge\max(2R_\pi,2)$ and $\|ay\|>r/2$.
Equation~\eqref{eq:app_pm_far} with $p=1$, divided by the normalizer,
gives
\[
 \int_{\|ay\|>r/2}\|y\|w_x(y)\,\mathrm{d} y
 \le C_\varepsilon a^{\nu-1}r^{1+\alpha-\nu}.
\]
Since $\|y\|>1$ there, boundedness of $s$ bounds the remaining
contribution to \eqref{eq:app_cov_split} by the same expression up to a
constant. It is at most $C_\varepsilon a$ because $\nu-1>1$ and
$1+\alpha-\nu\le0$. Taking the supremum over $x$ proves (iii), including
$\nu=1+\alpha$.
\end{proof}

\begin{proof}[\normalfont\bfseries Proof of Proposition~\ref{prop:brown_ct_divergence}]
The Gaussian conditional-mean derivative identity
\citep[Proposition~1]{DytsoPoorShamai2021}, applied to $aX_0+\gamma Z$,
gives
$\nabla m_t(x)=(a/\gamma^2)\mathrm{Cov}(X_0\mid X_t=x)$.
The conditional second moments are finite by Gaussian damping.
For compact support, the covariance operator norm is at most $R^2$,
giving $C_t\le aR^2/\gamma^2\to0$.
Under Assumption~\ref{ass:heavy}, set $U=(aX_0-x)/\gamma$.
Its conditional density is proportional to
\[
 \rho_x(u)\varphi(u),\qquad
 \rho_x(u)=\frac{\pi((x+\gamma u)/a)}{\pi(x/a)},
\]
where $\varphi$ is the standard Gaussian density.
As $\|x\|\to\infty$, $\rho_x(u)\to1$ for bounded $u$.
On $\|u\|\le\|x\|/(2\gamma)$ it is uniformly bounded by the power-law
envelopes; on the complement it is at most $C\|x\|^{d+\nu}$.
Dominated convergence on the first region and Gaussian tail decay on
the second therefore give, for $w=1,u,uu^\top$,
$\int w(u)\rho_x(u)\varphi(u)\,\mathrm{d} u
\to\int w(u)\varphi(u)\,\mathrm{d} u$.
Thus $\mathrm{Cov}(U\mid X_t=x)\to I_d$ and
$\nabla m_t(x)\to a^{-1}I_d$.
It follows that $C_t\ge a(t)^{-1}\to\infty$ and
$C_t^{(2)}\le0$. All these spatial limits are uniform on compact
positive-time intervals. Gaussian domination also gives joint
continuity of the covariance there, hence its global spatial boundedness
uniformly on each such interval.
\end{proof}

\subsection{Proof of Corollary~\ref{cor:mixing}: Stable mixing}

\begin{proof}[\normalfont\bfseries Proof of Corollary~\ref{cor:mixing}]
Since $\nu\ge1+\alpha>\alpha$, both initial laws have finite $p$-th
moments. Proposition~\ref{prop:heavy_denoiser} gives
$C_t\le c_\varepsilon a(t)$ for $t\ge\varepsilon$, and hence
\[
 \mathfrak D_\varepsilon^{\mathrm{mix}}
 \le\frac{(1+\eta)c_\varepsilon}{\alpha\gamma(\varepsilon)^\alpha}
       \int_\varepsilon^\infty\beta(t)a(t)^2\,\mathrm{d} t
 =\frac{(1+\eta)c_\varepsilon a(\varepsilon)^2}
       {2\gamma(\varepsilon)^\alpha}<\infty,
\]
using $\mathrm d(a^2)=-(2/\alpha)\beta a^2\,\mathrm{d} t$.
Equation~\eqref{eq:app_mixing_optimal} proves the claim with
$C=e^{\mathfrak D_\varepsilon^{\mathrm{mix}}}$, independent of $T$.
\end{proof}

\paragraph{Brownian synchronous coupling.}
The same argument for the VP SDE produces
$\Phi^{(2)}(T):=\int_\varepsilon^T \beta(t)(1-a_tC_t)/\gamma_t^2\,\mathrm{d}t$.
Proposition~\ref{prop:brown_ct_divergence} gives $a_tC_t\ge1$, so
$\Phi^{(2)}(T)\le0$ and $e^{-\Phi^{(2)}(T)}\ge1$. Thus this synchronous bound is
non-contractive; it does not imply that Brownian mixing fails.

\subsection{Proof of Theorem~\ref{thm:gauss_init}: Brownian initialization error}\label{app:gauss_init}

We prove the exponential $W_p$ convergence of the Brownian model
(Theorem~\ref{thm:gauss_init}), for the \emph{reverse SDE}, so that the object compared with
the stable sampler of Corollary~\ref{cor:mixing} is of the same nature. We work in the
constant-schedule VP model $\,\mathrm{d} X_t=-\frac{\beta_0}{2}X_t\,\mathrm{d}
t+\sqrt{\beta_0}\,\mathrm{d} W_t$, with $a_t=e^{-\beta_0t/2}$, $\gamma_t^2=1-a_t^2$ and
stationary law $\pi_\infty^{(2)}=\mathcal N(0,I_d)$; the general schedule follows by a time
change. The standing hypotheses are those of Theorem~\ref{thm:gauss_init}:
\begin{equation}\label{eq:app_HB}
\text{(H)}\qquad \pi_{\mathrm{data}}\ \text{satisfies Assumption~\ref{ass:heavy} with }\nu>2 .
\end{equation}
In particular $M_2:=\mathbb{E}\left\lVert X_0\right\rVert^2<\infty$.

The sampler analyzed is the Brownian reverse SDE
\begin{equation}\label{eq:app_revvp}
\,\mathrm{d}\overleftarrow Y_\tau=\Bigl[\tfrac{\bar\beta(\tau)}{2}\overleftarrow Y_\tau+\bar\beta(\tau)\nabla\log \pi_{T-\tau}(\overleftarrow Y_\tau)\Bigr]\,\mathrm{d}\tau+\sqrt{\bar\beta(\tau)}\,\,\mathrm{d}\overleftarrow W_\tau ,
\end{equation}
run over $\tau\in[0,T-\varepsilon]$. Lemmas~\ref{app:lem_revwp} and~\ref{app:lem_revkernel}
show that \eqref{eq:app_revvp} is well posed from any initial law and that its transition
kernel $\mathsf K_{T,\varepsilon}$ is the forward conditional kernel,
\begin{equation}\label{eq:app_revkernel}
\mathsf K_{T,\varepsilon}(x,\cdot)=\mathrm{Law}\bigl(X_\varepsilon\mid X_T=x\bigr),\qquad\text{so that}\qquad \pi_T\mathsf K_{T,\varepsilon}=\pi_\varepsilon .
\end{equation}
The generated law is then $\mu_{T,\varepsilon}^{(2)}:=\pi_\infty^{(2)}\mathsf K_{T,\varepsilon}$; to
lighten notation, we abbreviate it as $\mu_T$ throughout this subsection and write
$\mathcal{E}_1=W_p(\mu_T,\pi_\varepsilon)$. Everything else follows from three ingredients:
a total-variation bound of order $a_T$ for $\left\lVert
\pi_T-\pi_\infty^{(2)}\right\rVert_{\mathrm{TV}}$, a uniform moment bound for $\mu_T$
obtained from the likelihood ratio $\pi_\infty^{(2)}/\pi_T$, and an interpolation between
total variation and $W_p$.

\begin{lemma}[Score regularity and well-posedness]\label{app:lem_revwp}
Under Assumption~\ref{ass:heavy} and for $\varepsilon>0$, write
$\Sigma_t(y):=\mathrm{Cov}_{\mu_y}(X_0)$ for the conditional covariance entering the
Jacobian of the denoiser $m_t$. Then $\pi_t\in C^\infty$, $\pi_t>0$, and
\begin{equation*}
\nabla\log \pi_t(y)=\frac{a_t\,m_t(y)-y}{\gamma_t^2},\qquad
\nabla^2\log \pi_t(y)=-\frac{1}{\gamma_t^2}I_d+\frac{a_t^2}{\gamma_t^4}\Sigma_t(y),
\end{equation*}
with $\Lambda_{\varepsilon,T}:=\sup_{t\in[\varepsilon,T]}\sup_y\left\lVert \nabla^2\log
\pi_t(y)\right\rVert<\infty$. Consequently the drift of \eqref{eq:app_revvp} is globally
Lipschitz in $y$, uniformly in $\tau\in[0,T-\varepsilon]$, so \eqref{eq:app_revvp} has a
unique strong non-explosive solution from every initial law, and its transition kernel
$\mathsf K_{T,\varepsilon}$ is well defined.
\end{lemma}

\begin{proof}
Gaussian Tweedie calculus and
\citet[Proposition~1]{DytsoPoorShamai2021} give
\[
 \nabla\log\pi_t(y)=\frac{a_tm_t(y)-y}{\gamma_t^2},
 \qquad
 \nabla^2\log\pi_t(y)=-\gamma_t^{-2}I_d
       +a_t^2\gamma_t^{-4}\Sigma_t(y).
\]
The proof of Proposition~\ref{prop:brown_ct_divergence} bounds
$\Sigma_t$ uniformly on $[\varepsilon,T]\times\mathbb{R}^d$.
Joint continuity also bounds $\nabla\log\pi_t(0)$ on that interval.
The backward drift is therefore uniformly Lipschitz with linear growth;
subtracting the additive Gaussian integral gives strong existence,
pathwise uniqueness and non-explosion.
\end{proof}

\begin{lemma}[The reverse kernel is the forward conditional kernel]\label{app:lem_revkernel}
Under \eqref{eq:app_HB}, \eqref{eq:app_revkernel} holds for $\pi_T$-almost every
$x$; in particular $\pi_T\mathsf K_{T,\varepsilon}=\pi_\varepsilon$, and
$\mu_T=\pi_\infty^{(2)}\mathsf K_{T,\varepsilon}$ is the law of the output of
\eqref{eq:app_revvp} initialized at $\pi_\infty^{(2)}$.
\end{lemma}

\begin{proof}
Let $P$ be the forward Ornstein--Uhlenbeck path law and $R$ the same
process started from its Gaussian invariant law. Their conditional path
laws given $X_0$ agree, so
$H(P\mid R)=\mathrm{KL}(\pi_{\mathrm{data}}\mid\pi_\infty^{(2)})<\infty$,
by boundedness of $\pi$ and its finite second moment.
The stationary Gaussian reference satisfies Hypotheses~1.10 of
\citet[arXiv version~3]{CattiauxConfortiGentilLeonard};
their Theorem~1.12 identifies
the reverse drift as $\beta_0 y/2+\beta_0\nabla\log\pi_{T-\tau}(y)$.
Restricting to $[0,T-\varepsilon]$ and using
Lemma~\ref{app:lem_revwp} identifies the reverse transition kernel with
$\mathrm{Law}(X_\varepsilon\mid X_T=x)$, $\pi_T$-almost everywhere.
Since $\pi_T$ is positive, the identity also holds
$\pi_\infty^{(2)}$-almost everywhere. Propagating the two initial laws
gives the two asserted kernel identities.
\end{proof}

\begin{lemma}[TV bound]\label{app:lem_tvquad}
Under \eqref{eq:app_HB}, let
$K_0:=\mathrm{KL}(\pi_{\mathrm{data}}\,\|\,\pi_\infty^{(2)})<\infty$.
Then
\[
 \|\mu_T-\pi_\varepsilon\|_{\mathrm{TV}}
 \le\|\pi_T-\pi_\infty^{(2)}\|_{\mathrm{TV}}
 \le\sqrt{K_0/2}\,a_T.
\]
\end{lemma}

\begin{proof}
The first inequality is contraction of total variation under
$\mathsf K_{T,\varepsilon}$, by Lemma~\ref{app:lem_revkernel}.
Gaussian log-Sobolev entropy decay and Pinsker
\citep[Theorem~5.2.1, Eq.~(5.2.2), and Proposition~5.5.1]{BakryGentilLedoux2014}
give the second inequality, with the VP normalization
$\mathrm{KL}(\pi_T\,\|\,\pi_\infty^{(2)})\le a_T^2K_0$.
Here $K_0\le\log\|\pi\|_\infty+M_2/2+(d/2)\log(2\pi)<\infty$.
\end{proof}

\begin{lemma}[Density comparison]\label{app:lem_density}
If $M_2<\infty$, then $\pi_\infty^{(2)}(x)\le
K_Te^{\lambda_T\left\lVert x\right\rVert^2}\pi_T(x)$ with
$\lambda_T=\frac{a_T}{2(1-a_T)}\to0$ and
$K_T=\gamma_T^d\,e^{\lambda_T M_2}\to1$.
\end{lemma}

\begin{proof}
For $a=a_T\in(0,1)$, Young's inequality gives
$\|x-aX_0\|^2\le(1+a)(\|x\|^2+a\|X_0\|^2)$.
Jensen's inequality and $\gamma_T^2=(1-a)(1+a)$ therefore yield
\[
 \pi_T(x)\ge(2\pi\gamma_T^2)^{-d/2}
 \exp\!\left[-\frac{\|x\|^2+aM_2}{2(1-a)}\right].
\]
Dividing the standard Gaussian density by this lower bound gives
$K_Te^{\lambda_T\|x\|^2}$ without a centering assumption.
\end{proof}

\begin{lemma}[Uniform moments of the generated law]\label{app:lem_unifmom}
Assume \eqref{eq:app_HB} and let $q\in[1,\nu)$. Choose $r>1$ with $qr<\nu$ and set
$r'=r/(r-1)$. Then for every $T$ with $r'a_T<1$,
\[
\mathbb{E}_{\mu_T}\left\lVert
Y\right\rVert^q\;\le\;\bigl(\mathbb{E}\left\lVert
X_\varepsilon\right\rVert^{qr}\bigr)^{1/r}\,
\Bigl[K_T^{\,r'-1}\bigl(1-2(r'-1)\lambda_T\bigr)^{-d/2}\Bigr]^{1/r'},
\]
with
$K_T,\lambda_T$ as in Lemma~\ref{app:lem_density}. In particular $\sup_{T\ge
T_q,\,0<\varepsilon<T}\mathbb{E}_{\mu_T}\left\lVert Y\right\rVert^q<\infty$ for a threshold $T_q$ determined
by, for instance, $r'a_{T_q}\le\frac12$, and the bound holds \emph{for every early-stopping time $0<\varepsilon<T$}, since
$a_\varepsilon,\gamma_\varepsilon\le1$ makes $\mathbb{E}\left\lVert
X_\varepsilon\right\rVert^{qr}$ bounded independently of $\varepsilon$.
\end{lemma}

\begin{proof}
Let $R_T=\pi_\infty^{(2)}/\pi_T$. The reverse-kernel identity and
H\"older give
\[
 \mathbb{E}_{\mu_T}\|Y\|^q
 =\mathbb{E}[\|X_\varepsilon\|^qR_T(X_T)]
 \le(\mathbb{E}\|X_\varepsilon\|^{qr})^{1/r}
       (\mathbb{E} R_T(X_T)^{r'})^{1/r'}.
\]
Lemma~\ref{app:lem_density} and Gaussian integration give
\[
 \mathbb{E} R_T(X_T)^{r'}
 =\int R_T^{r'-1}\pi_\infty^{(2)}
 \le K_T^{r'-1}
       \left(1-2(r'-1)\lambda_T\right)^{-d/2},
\]
provided $2(r'-1)\lambda_T<1$, equivalently $r'a_T<1$.
The representation $X_\varepsilon=a_\varepsilon X_0+\gamma_\varepsilon Z$
and $qr<\nu$ bound the first factor independently of $\varepsilon$.
Taking $r'a_{T_q}\le1/2$ bounds the second uniformly for $T\ge T_q$.
\end{proof}

\begin{proposition}[Constant-schedule Brownian $W_p$ convergence]\label{app:prop_gaussinit_const}
Assume \eqref{eq:app_HB}. In the constant-schedule VP model, for $p\in[1,\nu)$ and any $q\in(p,\nu)$ there is $K_{p,q}<\infty$, independent of the early-stopping time $0<\varepsilon<T$, such that for $T$ large enough $\mathcal{E}_1=W_p(\mu_T,\pi_\varepsilon)\le K_{p,q}\,e^{-\frac{q-p}{2pq}\beta_0T}$. \end{proposition}

\begin{proof}
Couple the common part of two probability measures identically and their
residual parts arbitrarily. Since
$\|x-y\|^p\le2^{p-1}(\|x\|^p+\|y\|^p)$, this gives
$W_p(\mu,\nu)^p\le2^{p-1}\int\|x\|^p\,\mathrm d|\mu-\nu|$.
H\"older with respect to $|\mu-\nu|\le\mu+\nu$ then gives
\[
 W_p(\mu,\nu)^p
 \le2^{p-1}
    \left(\int\|x\|^q\,\mathrm d(\mu+\nu)\right)^{p/q}
    \left(2\|\mu-\nu\|_{\mathrm{TV}}\right)^{1-p/q}.
\]
Apply this with $(\mu,\nu)=(\mu_T,\pi_\varepsilon)$.
Lemmas~\ref{app:lem_tvquad} and~\ref{app:lem_unifmom} give
$W_p\le K_{p,q}a_T^{1/p-1/q}
=K_{p,q}e^{-(q-p)\beta_0T/(2pq)}$, uniformly in $\varepsilon$.
\end{proof}

\begin{lemma}[Reduction by time change]\label{app:lem_timechange}
In the VP model the joint law of $(X_\varepsilon,X_T)$ --- hence the kernel $\mathsf
K_{T,\varepsilon}$, the marginals $\pi_\varepsilon$, $\pi_T$ and the generated law $\mu_T$
--- depends on $(\varepsilon,T)$ only through $\bigl(B(\varepsilon),B(T)\bigr)$.
Consequently the pair $(\mu_T,\pi_\varepsilon)$ for a schedule $\beta$ is that of the unit
schedule at horizon $B(T)$ and early-stopping time $B(\varepsilon)$.
\end{lemma}

\begin{proof}
The Ornstein--Uhlenbeck transition from $s$ to $t$ has mean multiplier
$e^{-[B(t)-B(s)]/2}$ and variance $1-e^{-[B(t)-B(s)]}$.
Consequently the joint law of $(X_\varepsilon,X_T)$ is that of the unit
schedule at $(B(\varepsilon),B(T))$. The same is true of its conditional
kernel and the two propagated laws.
\end{proof}

\begin{proof}[\normalfont\bfseries Proof of Theorem~\ref{thm:gauss_init}]
Apply Proposition~\ref{app:prop_gaussinit_const} with unit schedule at
$(B(\varepsilon),B(T))$, using Lemma~\ref{app:lem_timechange}.
Its constants and threshold are uniform in the early-stopping time, and
$B(T)\to\infty$.
\end{proof}

\section{Proofs for Section~\ref{sec:tce}: Tail Regimes and Extreme Quantiles}\label{app:tce}

The two consequences of Assumption~\ref{ass:heavy} used in this appendix are
the density envelope \eqref{eq:a5prime} and $\overline G_{\pi_{\mathrm{data}}}(r) \sim
c'_\pi r^{-\nu}, \qquad c'_\pi=\frac{c_\pi\omega_{d-1}}{\nu}$, where locally
$c_\pi:=c_{\pi_{\mathrm{data}}}$, obtained by integrating the
uniform asymptotic relation \eqref{eq:a5} in polar
coordinates. The former controls the denoiser estimates, while the latter
identifies the radial tail index and constant. Both hold here with $\nu>1$.
Let
$\mu_{\tau}=\mathrm{Law}(\overleftarrow Y_{\tau})$ denote the output at
$\tau\in(0,T-\varepsilon]$ from the appropriate reference law, and write
$\mu_{\tau}^{(\alpha)}$ and $\mu_{\tau}^{(2)}$ when the stable and Brownian models must be distinguished. Both
models are treated under Assumption~\ref{ass:heavy}, with $\nu>1$ and
fixed $0<\varepsilon<T<\infty$.

\subsection{Proof of Proposition~\ref{prop:tce}: Tail-regime dichotomy}

\begin{proof}[\normalfont\bfseries Proof of Proposition~\ref{prop:tce}]
For (i), the survival representation and Karamata's theorem
\citep[Theorem~12(ii)]{KeveiRegularVariation} give
\[
 e_\mu(u)=\frac{\int_u^\infty\overline G_\mu(r)\,\mathrm{d} r}
                 {\overline G_\mu(u)}
 \sim\frac{u}{\vartheta-1}.
\]
For (ii), put $V=\|G+R\|$ and $Q=\|G\|$.
The bound $|V-Q|\le K$ implies
\[
 (V-u)^+\le2K\mathbf1_{\{V>u\}}+(Q-u-K)^+,
 \qquad \{Q>u+K\}\subseteq\{V>u\}.
\]
Consequently $e_\mu(u)\le2K+e_Q(u+K)$, without any independence
assumption. For $G\sim\mathcal{N}(0,\sigma^2I_d)$, the incomplete-gamma
asymptotic \citep[Eq.~8.11.2]{NISTDLMF} gives
$\mathbb{P}(Q>v)\sim c v^{d-2}e^{-v^2/(2\sigma^2)}$ and
$e_Q(v)\sim\sigma^2/v$.
Thus $e_\mu(u)=\mathcal O(1)$ and $\Gamma(\mu)=0$.
\end{proof}

\subsection{Brownian tail bounds for Proposition~\ref{prop:tail_separation}}

\begin{lemma}[Bounded smoothed Brownian score]\label{app:lem_flow}
Under Assumption~\ref{ass:heavy}, for every $0<\varepsilon<T<\infty$,
\[
 S_{\varepsilon,T}:=
 \sup_{s\in[\varepsilon,T],\,y\in\mathbb R^d}
 \|\nabla\log\pi_s^{(2)}(y)\|<\infty.
\]
\end{lemma}

\begin{proof}
Set $U=(a_sX_0-y)/\gamma_s$. The weighted Gaussian-convolution argument
in the proof of Proposition~\ref{prop:brown_ct_divergence} gives
$\mathbb E[U\mid X_s=y]\to0$ as $\|y\|\to\infty$, uniformly for
$s\in[\varepsilon,T]$. Differentiating the Gaussian kernel yields
\[
 \nabla\log\pi_s^{(2)}(y)
 =\gamma_s^{-1}\mathbb E[U\mid X_s=y].
\]
Since $\gamma_s$ is bounded away from zero on this interval, the score
tends to zero uniformly at spatial infinity. Gaussian convolution and
positivity give joint continuity, hence boundedness on the remaining
compact set.
\end{proof}

\begin{lemma}[Brownian-plus-bounded decomposition]\label{app:lem_gaussplusbdd}
Under Assumption~\ref{ass:heavy}, fix $0<\varepsilon<T<\infty$ and
consider the Brownian backward SDE
$\,\mathrm{d}\overleftarrow Y_u=\bigl[\frac{\bar\beta(u)}{2}\overleftarrow
Y_u+\bar\beta(u)\nabla\log \pi_{T-u}^{(2)}(\overleftarrow
Y_u)\bigr]\,\mathrm{d}u+\sqrt{\bar\beta(u)}\,\,\mathrm{d} W_u$, $\overleftarrow
Y_0\sim\mathcal N(0,I_d)$ independently of $W$.
Then for each $\tau\in(0,T-\varepsilon]$,
\[
\overleftarrow Y_{\tau}=G_{\tau}+R_{\tau},\qquad G_{\tau}\sim\mathcal N\bigl(0,\sigma_{\tau}^2I_d\bigr),\quad \sigma_{\tau}^2=2e^{2H(\tau)}-1,\quad
\left\lVert R_{\tau}\right\rVert\le K_{\tau}:=2S_{\varepsilon,T}\bigl(e^{H(\tau)}-1\bigr)\ \text{a.s.},
\]
with $H(\tau):=\frac12\int_0^{\tau}\bar\beta(v)\,\mathrm{d} v$ (and $R_{\tau}$ adapted, not independent
of $G_{\tau}$).
\end{lemma}

\begin{proof}
Variation of constants gives
\[
 G_{\tau}=e^{H(\tau)}\overleftarrow Y_0+
     \int_0^{\tau}e^{H(\tau)-H(u)}\sqrt{\bar\beta(u)}\,\mathrm{d} W_u,
 \quad
 R_{\tau}=\int_0^{\tau}e^{H(\tau)-H(u)}\bar\beta(u)
       \nabla\log\pi_{T-u}^{(2)}(\overleftarrow Y_u)\,\mathrm{d} u.
\]
The two terms defining $G_{\tau}$ are independent Gaussians; their summed
per-coordinate variance is $e^{2H(\tau)}+(e^{2H(\tau)}-1)=\sigma_{\tau}^2$.
Lemma~\ref{app:lem_flow} gives
$\|R_{\tau}\|\le S_{\varepsilon,T}\int_0^{\tau}e^{H(\tau)-H(u)}
\bar\beta(u)\,\mathrm{d} u
=2S_{\varepsilon,T}(e^{H(\tau)}-1)=K_{\tau}$.
\end{proof}

\begin{corollary}[Brownian relative excess]\label{app:cor_gauss_tce}
Under Assumption~\ref{ass:heavy}, for every $\tau\in(0,T-\varepsilon]$ the Brownian
generated law satisfies $e_{\mu_{\tau}^{(2)}}(u)=\mathcal O(1)$ and $\Gamma(\mu^{(2)}_{\tau})=0$,
and its radial quantiles grow at most at square-root logarithmic order:
$q_\kappa(\mu^{(2)}_{\tau})=\mathcal
O_d\bigl(\sigma_{\tau}\sqrt{\log\frac{1}{1-\kappa}}\,\bigr)+K_{\tau}$ for $\kappa$ close to $1$.
\end{corollary}

\begin{proof}
Proposition~\ref{prop:tce}(ii) applies to
Lemma~\ref{app:lem_gaussplusbdd}. Moreover, for $r>K_{\tau}$,
\[
 \mathbb{P}(\|\overleftarrow Y_{\tau}\|>r)
 \le2^{d/2}e^{-(r-K_{\tau})^2/(4\sigma_{\tau}^2)},
\]
by $\|\overleftarrow Y_{\tau}\|\le\|G_{\tau}\|+K_{\tau}$ and
$\mathbb{E}e^{\|G_{\tau}\|^2/(4\sigma_{\tau}^2)}=2^{d/2}$.
Inverting this upper bound gives the stated quantile estimate.
\end{proof}

\subsection{Stable tail bounds for Proposition~\ref{prop:tail_separation}}

Throughout this subsection $\mu_{\tau}=\mathrm{Law}(\overleftarrow Y_{\tau})$ is the law
at a fixed generation time $\tau\in(0,T-\varepsilon]$. The normalization of
Appendix~\ref{app:compact_decay} gives
$\nu_\alpha(\|z\|>r)=C_{d,\alpha}r^{-\alpha}$ and
$\mathbb{P}(\|L_1^\alpha\|>r)\sim C_{d,\alpha}r^{-\alpha}$.

\paragraph{Scope.}
The fixed-time tail results in this subsection hold under Assumption~\ref{ass:heavy} for every $\nu>1$.
Corollary~\ref{cor:mixing} and the non-asymptotic generation bound require
$\nu\ge1+\alpha$, including equality. Proposition~\ref{prop:heavy_denoiser}
gives the corresponding global regularity regimes.

The spatial asymptotics of $m_s$ and its uniformly sublinear remainder were established in Proposition~\ref{app:prop_denoiser}; we use that result below.

The next lemma requires only uniform sublinearity of the drift remainder, not
independence from the stable leading term.

\begin{lemma}[Tail stability under a uniformly sublinear drift]\label{app:lem_sublinear}
Fix $\tau>0$ and consider
\begin{equation}\label{eq:app_abstract_sde}
\,\mathrm{d} \mathcal X_u=\bigl(\lambda(u)\mathcal X_u+\mathsf r(u,\mathcal X_u)\bigr)\,\mathrm{d} u+\sigma(u)\,\mathrm{d} L^\alpha_u,
\qquad \mathcal X_0\sim\mathrm{S}\alpha\mathrm{S}_d(1)\ \text{independent of }L^\alpha,
\end{equation}
where $\lambda,\sigma$ are bounded deterministic measurable scalar functions on $[0,\tau]$, and
$\mathsf r$ is jointly measurable, locally Lipschitz in $x$ uniformly on compact subsets of
$[0,\tau]\times\mathbb{R}^d$, of linear growth, and uniformly sublinear:
$\lim_{R\to\infty}\sup_{u\in[0,\tau]}\sup_{\left\lVert x\right\rVert\ge R}\left\lVert \mathsf
r(u,x)\right\rVert/\left\lVert x\right\rVert=0$.
Then \eqref{eq:app_abstract_sde} has a unique non-explosive strong solution and, with
$E(\tau,u):=\exp\Bigl(\int_u^{\tau}\lambda(v)\,\mathrm{d} v\Bigr), \qquad
\rho_{\tau}^\alpha:=E(\tau,0)^\alpha+\int_0^{\tau}E(\tau,u)^\alpha\left\lvert\sigma(u)\right\rvert^\alpha\,\mathrm{d}
u$, one has $\mathbb{P}(\left\lVert \mathcal X_{\tau}\right\rVert>R)\sim
C_{d,\alpha}\rho_{\tau}^\alpha R^{-\alpha}$ as $R\to\infty$.
\end{lemma}

\begin{proof}
Subtracting the additive noise integral reduces the equation to a
Carath\'eodory equation; local Lipschitz continuity and linear growth give
strong existence, uniqueness and non-explosion.
On the same probability space set
\[
 \mathcal W_u=E(u,0)\mathcal X_0+
          \int_0^u E(u,v)\sigma(v)\,\mathrm{d} L_v^\alpha.
\]
The stable-integral characteristic function gives
$\mathcal W_{\tau}\sim\mathrm S\alpha\mathrm S_d(\rho_{\tau})$.
The marginal tail constant agrees with that of the stable L\'evy measure
\citep[Lemma~2.1]{HultLindskog2007}, hence
\begin{equation}\label{eq:app_lincomp_tail}
 \mathbb{P}(\|\mathcal W_{\tau}\|>R)
 \sim C_{d,\alpha}\rho_{\tau}^\alpha R^{-\alpha}.
\end{equation}
Writing
$\mathcal W_u=E(u,0)(\mathcal X_0+M_u)$, the process
$M_u=\int_0^uE(v,0)^{-1}\sigma(v)\,\mathrm{d} L_v^\alpha$
has the law of a stable process at the deterministic clock
$\int_0^u|E(v,0)^{-1}\sigma(v)|^\alpha\,\mathrm{d} v$.
The stable maximal bound \citep[Theorem~13.1]{Schilling2016} therefore gives
\begin{equation}\label{eq:app_Hmax}
 H_{\tau}:=\limsup_{R\to\infty}R^\alpha
          \mathbb{P}\left(\sup_{u\le \tau}\|\mathcal W_u\|>R\right)<\infty.
\end{equation}
For any $\delta>0$, uniform sublinearity gives
$\|\mathsf r(u,x)\|\le\delta\|x\|+K_\delta$.
Variation of constants and Gr\"onwall imply
\[
 \sup_{u\le \tau}\|\mathcal X_u-\mathcal W_u\|
 \le A_\delta\sup_{u\le \tau}\|\mathcal W_u\|+B_\delta,
 \quad
 A_\delta=C_E\delta \tau e^{C_E\delta \tau},\quad
 B_\delta=C_EK_\delta \tau e^{C_E\delta \tau},
\]
where $C_E=\sup_{v\le u\le \tau}E(u,v)<\infty$.
For every $\zeta>0$, \eqref{eq:app_Hmax} bounds the limsup of
$R^\alpha\mathbb{P}(\|\mathcal X_{\tau}-\mathcal W_{\tau}\|>\zeta R)$ by
$H_{\tau}(A_\delta/\zeta)^\alpha$.
Letting $\delta\downarrow0$ yields
\begin{equation}\label{eq:app_negligible}
 R^\alpha\mathbb{P}(\|\mathcal X_{\tau}-\mathcal W_{\tau}\|>\zeta R)\longrightarrow0.
\end{equation}
Sandwich $\mathbb{P}(\|\mathcal X_{\tau}\|>R)$ between the tails of
$\|\mathcal W_{\tau}\|$ at $(1\pm\zeta)R$, with the error in
\eqref{eq:app_negligible}. Equation~\eqref{eq:app_lincomp_tail} and
$\zeta\downarrow0$ give the asserted constant.
\end{proof}

\begin{lemma}[Linear-plus-sublinear decomposition of the backward drift]\label{app:lem_stable_lead}
For $u\in[0,T-\varepsilon]$ set
\begin{equation*}
\ell_0(u):=\bar\beta(u)\Bigl(\tfrac1\alpha-\tfrac{1+\eta}{\alpha}\gamma(T-u)^{-\alpha}\Bigr)<0,
\qquad
c(u):=\tfrac{1+\eta}{\alpha}\,\bar\beta(u)\,a(T-u)\,\gamma(T-u)^{-\alpha}>0 ,
\end{equation*}
and, with $\varkappa_{\nu,s}$ as in \eqref{eq:app_varkappa},
\begin{equation}\label{eq:app_ellnu}
\ell_\nu(u):=\ell_0(u)+c(u)\,\varkappa_{\nu,T-u}
=\begin{cases}
\dfrac{\bar\beta(u)}{\alpha}, & 1<\nu<\alpha,\\[3mm]
\bar\beta(u)\Bigl[\dfrac1\alpha-\dfrac{1+\eta}{\alpha}\cdot\dfrac{k_{d,\alpha}}{c_\pi a(T-u)^{\alpha}+k_{d,\alpha}\gamma(T-u)^{\alpha}}\Bigr], & \nu=\alpha,\\[4mm]
\ell_0(u), & \nu>\alpha .
\end{cases}
\end{equation}
Then the backward drift of \eqref{eq:backward} splits as
\begin{equation}\label{eq:app_WA}
\bar b(u,x)=\ell_\nu(u)\,x+\widetilde{\mathsf r}(u,x),
\qquad
\widetilde{\mathsf r}(u,x):=c(u)\,\mathsf r_{T-u}(x),
\end{equation}
where $\ell_\nu$ and $c$ are bounded on $[0,T-\varepsilon]$ and $\widetilde{\mathsf r}$ is
locally Lipschitz, of linear growth, and uniformly sublinear in the sense of
\eqref{eq:app_sublin}. Setting
\begin{equation}\label{eq:app_rho}
E_\nu(\tau,u):=\exp\Bigl(\int_u^{\tau}\ell_\nu(v)\,\mathrm{d} v\Bigr),
\qquad
\rho_{\nu,\tau}^\alpha:=E_\nu(\tau,0)^\alpha+\eta\int_0^{\tau}E_\nu(\tau,u)^\alpha\,\bar\beta(u)\,\mathrm{d}u\ \in(0,\infty),
\end{equation}
the hypotheses of Lemma~\ref{app:lem_sublinear} hold with $\lambda=\ell_\nu$ and
$\sigma(u)=\bigl(\eta\bar\beta(u)\bigr)^{1/\alpha}$.
\end{lemma}

\begin{proof}
The fractional Tweedie identity gives
$\bar b(u,x)=\ell_0(u)x+c(u)m_{T-u}(x)$.
Substitution of
$m_s(x)=\varkappa_{\nu,s}x+\mathsf r_s(x)$ from
Proposition~\ref{app:prop_denoiser} gives
\eqref{eq:app_WA} and the three coefficients in \eqref{eq:app_ellnu}.
All coefficients are bounded on $[0,T-\varepsilon]$ since
$\gamma(T-u)\ge\gamma(\varepsilon)>0$;
$\ell_0<0$ follows from $\gamma^\alpha<1+\eta$.
Multiplication by the bounded factor $c$ preserves the local regularity,
linear growth and uniform sublinearity of the remainder, so
Lemma~\ref{app:lem_sublinear} applies with the announced coefficients.
\end{proof}

\begin{theorem}[Exact tail of the stable generated law]\label{app:thm_stable_rv}
Assume Assumption~\ref{ass:heavy} for some $\nu>1$. Then for every $\tau\in(0,T-\varepsilon]$,
\begin{equation}\label{eq:app_rv}
\overline G_{\mu_{\tau}}(r)\;\sim\;c_{\tau}\,r^{-\alpha}\quad(r\to\infty),
\qquad
c_{\tau}:=C_{d,\alpha}\,\rho_{\nu,\tau}^\alpha\in(0,\infty),
\end{equation}
with $\rho_{\nu,\tau}$ given by \eqref{eq:app_rho} (written $\rho_{\tau}$ in the main text). In
particular the generated law is power-law of index $\alpha$ \emph{for every target
index $\nu>1$}, and, by Proposition~\ref{prop:tce},
$\Gamma(\mu^{(\alpha)}_{\tau})=\frac1{\alpha-1}$.
\end{theorem}

\begin{proof}
Lemma~\ref{app:lem_stable_lead} verifies every hypothesis of
Lemma~\ref{app:lem_sublinear}, with independent stable initialization.
The latter gives \eqref{eq:app_rv} with
$c_{\tau}=C_{d,\alpha}\rho_{\nu,\tau}^\alpha$ at each fixed generation time.
Proposition~\ref{prop:tce}(i) gives $\Gamma(\mu_{\tau}^{(\alpha)})=1/(\alpha-1)$.
\end{proof}

\subsection{Proofs of Lemma~\ref{lem:emr_asymp} and Proposition~\ref{prop:tail_separation}}

\begin{proof}[\normalfont\bfseries Proof of Lemma~\ref{lem:emr_asymp}]
Asymptotic inversion of regularly varying functions
\citep[Theorem~16]{KeveiRegularVariation} gives
$q_\kappa(\mu)\sim(c_\mu/(1-\kappa))^{1/\alpha}$.
This applies to generalized quantiles as well; either inverse convention
has the same equivalent. Apply the same result to $\pi$ and take the ratio
to obtain \eqref{eq:emr_asymp}.
\end{proof}

\begin{proof}[\normalfont\bfseries Proof of Proposition~\ref{prop:tail_separation}]
The tail regimes follow from Corollary~\ref{app:cor_gauss_tce} and
Theorem~\ref{app:thm_stable_rv}.
The Brownian quantile grows at most as
$\sqrt{\log(1/(1-\kappa))}$, whereas the target quantile grows as
$(c'_\pi/(1-\kappa))^{1/\nu}$, so their ratio vanishes.
For the stable output, apply Lemma~\ref{lem:emr_asymp} with tail constants
$(c_{\tau},c'_\pi)$ and indices $(\alpha,\nu)$.
\end{proof}
\section{Numerical Protocol for Figures~\ref{fig:c_t}--\ref{fig:app_numerical_tails}}\label{app:experiments}

All reported experiments are one-dimensional, synthetic, and training-free.
We set \(\beta(t)\equiv1\), and use \(\alpha=1.5\) and \(\eta=0.5\) for
the stable backward runs. Scores are not learned: the model-specific
continuous scores are evaluated numerically from exact forward characteristic
functions. We distinguish deterministic numerical approximations, such as
\(\widehat C_t(40)\), from empirical Monte-Carlo quantities, such as
\(\widehat W_p\) and the tail estimators.

\subsection{Common setup and score evaluation}\label{app:exp:score}

On the backward grid
\(\tau_k=kh\), \(t_k=T-\tau_k\), \(h=(T-\varepsilon)/N\), the stable EI
scheme is
\begin{equation}\label{eq:app_ei_num}
  Y_{k+1}=e^{-\lambda h}Y_k
  +\frac{1-e^{-\lambda h}}{\lambda}F(Y_k,t_k)
  +(1-e^{-\eta h})^{1/\alpha}Z_k,
  \qquad Z_k\sim\mathrm{S}\alpha\mathrm{S}_1(1),
\end{equation}
where
\begin{equation*}
  \lambda=\frac{\eta}{\alpha},
  \qquad
  F(x,t)=\frac{1+\eta}{\alpha}
  \bigl(x+\alpha S_t^{(\alpha)}(x)\bigr).
\end{equation*}
Thus the linear term and its stable stochastic convolution are integrated
exactly, while \(F\) is frozen at the left endpoint. In particular, the
noise scale in~\eqref{eq:app_ei_num} uses the same \(\eta\) as the drift. The
VP baseline uses the same splitting principle with
\(\lambda=1/2\), \(F_{\mathrm{VP}}(x,t)=x+s_t(x)\), and Gaussian noise of
variance \(1-e^{-h}\). No trajectory is clipped or projected.

For a Student-\(t_\nu\) target with characteristic function \(\phi_\nu\), the
forward characteristic functions are
\begin{align*}
  \phi_t^{(\alpha)}(u)
  &=\phi_\nu(e^{-t/\alpha}u)
    \exp\bigl(-(1-e^{-t})|u|^\alpha\bigr),
    \\
  \phi_t^{(\mathrm{VP})}(u)
  &=\phi_\nu(e^{-t/2}u)
    \exp\bigl(-\tfrac12(1-e^{-t})u^2\bigr).
\end{align*}
Fourier inversion of these exact expressions provides the density and the
numerator of the corresponding score. Separate space--time tables are built
for every pair of target and forward model and interpolated during
propagation. Stable tables use \(2^{19}\) Fourier modes and VP tables use
\(2^{18}\); the main tables have \(129\times8193\) points for the tail
experiment and \(257\times4097\) points for the initialization experiment.
Independent evaluations and tables with doubled Fourier and interpolation
resolutions are used as controls. Outside the tabulated spatial range, the
known score asymptotics are used rather than truncating the state.

\begin{table}[!htbp]
\centering\small
\caption{Parameters of the three reported experiments. Stable backward runs
use \(\eta=0.5\).}
\label{tab:specs}
\setlength{\tabcolsep}{5pt}
\renewcommand{\arraystretch}{1.18}
\begin{tabular}{>{\raggedright\arraybackslash}p{2.4cm}>{\raggedright\arraybackslash}p{3.2cm}>{\raggedright\arraybackslash}p{4.4cm}>{\raggedright\arraybackslash}p{5.2cm}}
\toprule
Experiment & Target and models & Time discretization
& Sampling and displayed quantity\\
\midrule
Denoiser modulus
& Student-\(t_4\); stable \(\alpha=1.5\), VP
& \(64\) times in \([0.05,4.5]\); no backward dynamics
& deterministic \(\widehat C_t(40)\), no particles or seeds\\
Initialization
& Student-\(t_4\); stable \((\alpha,\eta)=(1.5,0.5)\), VP
& \(\varepsilon=0.05\), \(h=0.0125\),
  \(T\in\{0.1,0.15,0.25,0.5,1,2\}\)
& \(12\) seeds, \(n=16384\);
  \(I_p\) for \(p\in\{1.1,1.4\}\)\\
Tail coverage
& Student-\(t_\nu\), \(\nu\in\{1.2,1.5,3\}\);
  stable \(\alpha=1.5\), VP only for \(\nu=1.5\)
& \(T=2\), \(\varepsilon=0.5\), \(N=80\)
& \(12\) seeds, \(n=262144\);
  Hill index, quantile ratio, and tail constant\\
\bottomrule
\end{tabular}

\end{table}

\subsection{Figure~\ref{fig:c_t}: Deterministic denoiser modulus}\label{app:exp:ct}

In this section, we describe how the Lipschitz modulus of the denoiser is computed in Figure \ref{fig:c_t}, for Brownian and stable dynamics in the case of a Student-\(t_4\) target. 
In the Student-\(t_4\) case,
\(\phi_4(u)=2u^2K_2(2|u|)\) for \(u\ne0\), with \(\phi_4(0)=1\).
Using the convention
\(\widehat f(u)=\int_{\mathbb R}f(x)e^{ixu}\,\mathrm{d}x\), the stable
fractional-score numerator has multiplier
\(-iu|u|^{\alpha-2}\), while the ordinary VP-score numerator has multiplier
\(-iu\). A further spatial derivative contributes another factor \(-iu\).
Fourier inversion therefore gives \(p_t,g_t,p_t',g_t'\), from which
\begin{equation*}
  S_t'(x)=\frac{g_t'(x)}{p_t(x)}
  -\frac{g_t(x)p_t'(x)}{p_t(x)^2}.
\end{equation*}
The denoiser derivative is then evaluated from the Tweedie identities,
\begin{equation*}
  m_{t,\alpha}'(x)
  =\frac{1+\alpha(1-e^{-t})S_{t,\alpha}'(x)}{e^{-t/\alpha}},
  \qquad
  m_{t,\mathrm{VP}}'(x)
  =\frac{1+(1-e^{-t})s_t'(x)}{e^{-t/2}}.
\end{equation*}
Figure~\ref{fig:c_t} reports
\begin{equation}\label{eq:app_ct_truncated}
  \widehat C_t(40)
  =\max_{x\in\mathcal G\cap[-40,40]}|\widehat m_t'(x)|.
\end{equation}
Here $\mathcal G$ is the evaluation grid and $\widehat m_t'$ is the
numerically computed derivative. This is a truncated, discretized maximum,
not a closed-form evaluation of
\(\sup_{\mathbb R}|m_t'|\).

The main Fourier grids use \((L,N_{\mathrm{FFT}})=(8192,2^{20})\) for the
stable model and \((1024,2^{19})\) for VP. At seven times,
\(N_{\mathrm{FFT}}\) is doubled at fixed \(L\), leaving
\eqref{eq:app_ct_truncated} stable at the displayed resolution. The maxima
also agree on the nested domains \(|x|\le8,16,24,40\), with every detected
maximizer lying in the innermost domain. Since the calculation does not involve
random samples, Figure~\ref{fig:c_t} shows no error bars.

\subsection{Figure~\ref{fig:exp2}: Initialization error}\label{app:exp:init}

For each \(T\), model, and seed, two independent EI populations are propagated
to \(t=\varepsilon\). One starts from the model-specific stationary reference,
\(\mathrm{S}\alpha\mathrm{S}_1(1)\) or \(\mathcal N(0,1)\); the other starts
from an exact draw of the forward marginal
\begin{equation*}
  X_T\overset{\mathcal L}=
  e^{-T/\alpha}X_0+(1-e^{-T})^{1/\alpha}Z
\end{equation*}
on the stable side, with the Gaussian analogue on the VP side. The two arms
use the same score table and step size but independent initial samples and
innovations.

For equal-size one-dimensional samples \(X,Y\), empirical Wasserstein distance
is computed exactly by sorting,
\begin{equation*}
  \widehat W_p(X,Y)
  =\left(\frac1n\sum_{i=1}^n
  |X_{(i)}-Y_{(i)}|^p\right)^{1/p}.
\end{equation*}
The displayed initialization proxy is
\begin{equation*}
  I_p(T)=\widehat W_p\bigl(
  Y_{\varepsilon,h}^{\mathrm{stat}},
  Y_{\varepsilon,h}^{\pi_T}\bigr),
  \qquad p\in\{1.1,1.4\}.
\end{equation*}
The exact--exact reference
\begin{equation*}
  F_p(T)=\widehat W_p(X_\varepsilon^{(1)},X_\varepsilon^{(2)})
\end{equation*}
uses two independent exact samples from the same \(\pi_\varepsilon\).
It measures finite-sample resolution, and is displayed on the final figure as an empirical sampling baseline to account for the quality of the statistical estimation in the experiment.

Figure~\ref{fig:exp2} shows medians over \(12\) seeds with
pointwise \(90\%\) percentile-bootstrap intervals for the median, obtained by
resampling seed labels jointly across horizons. At \(T=2\), temporal refinements \(N=156,312,624\) and a
doubled score table control the EI and interpolation errors. The choices
\(p=1.1\) and \(p=1.4\) remain below \(\alpha=1.5\).

\subsection{Figures~\ref{fig:numerical_tails} and~\ref{fig:app_numerical_tails}: Tail coverage}\label{app:exp:tails}

Each Student-\(t_\nu\) target determines its own forward marginal, score table,
reverse drift, and target quantiles. Stable runs use
\(\nu\in\{1.2,1.5,3\}\), while the VP baseline is included only at
\(\nu=1.5\). The main runs start from the model-specific stationary
reference and use \(T=2\), \(\varepsilon=0.5\), and \(N=80\), with
\(n=262144\) terminal samples per seed.
Figure~\ref{fig:app_numerical_tails} collects three complementary diagnostics:
the generated tail index, the scale of generated extreme quantiles relative
to the target, and the tail constant relative to the EI prediction.

\paragraph{Tail index.}
Let \(R_{(1)}\ge\cdots\ge R_{(n)}\) be the descending order statistics of
the generated radii \(R_i=|Y_i|\). Panel~(a) reports the reciprocal Hill
estimate \citep{Hill1975}
\begin{equation*}
  \widehat\alpha_k
  =\left(\frac1k\sum_{j=1}^k
  \log\frac{R_{(j)}}{R_{(k+1)}}\right)^{-1},
\end{equation*}
which estimates the power-law exponent from logarithmic excesses above
the threshold \(R_{(k+1)}\). It is evaluated at the pre-specified fractions
\begin{equation*}
  k/n\in\{0.003,0.004,0.006,0.008,0.012,0.018,0.03\},
\end{equation*}
with principal fraction \(k/n=0.006\).
Smaller fractions use fewer, more extreme observations; comparison across
fractions assesses sensitivity to the threshold.
The reference value is \(\alpha=1.5\) for all three stable outputs,
irrespective of the target index \(\nu\).

\paragraph{Extreme-quantile ratio.}
Writing \(s=1-\kappa\), panel~(b) reports
\begin{equation*}
  \widehat{\mathcal M}_\nu(s)
  =\frac{\widehat q_{1-s}(|Y|)}
  {q_{1-s}(|T_\nu|)},
\end{equation*}
where \(\widehat q\) is the empirical quantile of the generated radii
and \(T_\nu\) is a Student-\(t_\nu\) random variable.
The denominator is evaluated from the Student CDF \(F_\nu\) through
\(q_{1-s}(|T_\nu|)=F_\nu^{-1}(1-s/2)\), without a simulated target sample.
The horizontal axis decreases to the right, so the rightmost points
correspond to more extreme quantiles.

For the stable outputs, the predicted slope of
\(\log\mathcal M_\nu(s)\) against \(\log s\) is
\(1/\nu-1/\alpha\). Thus, as \(s\downarrow0\), the ratio should
decrease for \(\nu<\alpha\), approach a finite nonzero constant for
\(\nu=\alpha\), and increase for \(\nu>\alpha\).
At index match, this constant need not equal one: equality additionally
requires matching tail constants.
Slopes are fitted separately within each seed over the pre-specified
extreme-quantile window, then aggregated across seeds.

\paragraph{Tail constant.}
For a putative \(\alpha\)-tail, panel~(c) uses
\begin{equation*}
  \widehat c_k=(k/n)R_{(k+1)}^\alpha.
\end{equation*}
This follows from
\(\mathbb P(|Y|>r)\sim cr^{-\alpha}\), replacing the survival probability
at \(R_{(k+1)}\) by \(k/n\).
The displayed quantity is
\(\widehat c_k/c_{N,h}^{\mathrm{EI}}\), where
\(c_{N,h}^{\mathrm{EI}}\) is the tail constant predicted for the actual
finite-step EI chain.

More precisely, let \(G_k\) denote the deterministic part of one step
of~\eqref{eq:app_ei_num}, and set
\(r_k=\lim_{x\to\infty}G_k(x)/x\).
For these symmetric targets, the same asymptotic slope applies at both
tails. If \(c_k\) is the two-sided \(\alpha\)-tail constant before that
step, asymptotic linearity and addition of independent regularly varying
tails \citep[Corollary~7]{FossKorshunovZachary2009} give
\begin{equation*}
  c_{k+1}=|r_k|^\alpha c_k
  +C_{1,\alpha}(1-e^{-\eta h}),
  \qquad c_0=C_{1,\alpha},
\end{equation*}
where
\(C_{1,\alpha}=\lim_{r\to\infty}
r^\alpha\mathbb P(|Z|>r)\) for
\(Z\sim\mathrm{S}\alpha\mathrm{S}_1(1)\).
We set \(c_{N,h}^{\mathrm{EI}}=c_N\).
The horizontal level one in panel~(c) therefore tests agreement with the
discretized chain's tail constant; it does not assert equality with the
target tail constant or the continuous-time limit.
Panels~(a) and~(c) concern only the stable samplers.

\begin{figure}[!htbp]
  \centering
  \includegraphics[width=\linewidth]{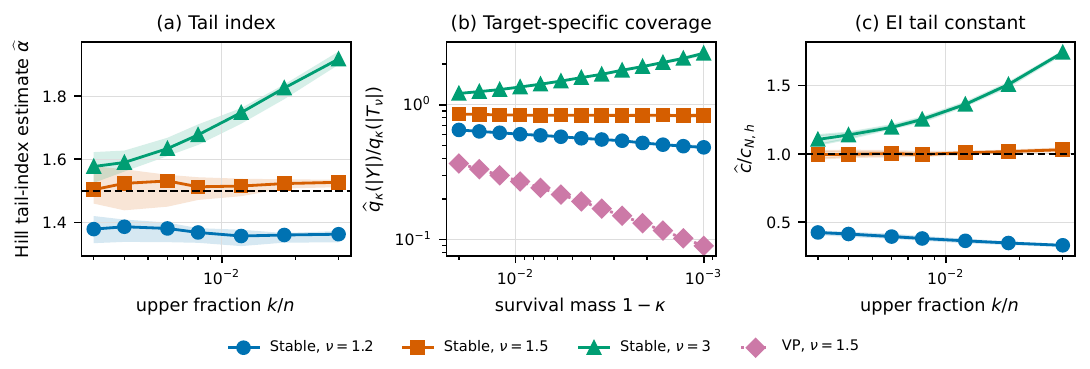}
  \caption{Tail diagnostics for Student-\(t_\nu\) targets with stable
  \((\alpha,\eta)=(1.5,0.5)\), \(T=2\), \(\varepsilon=0.5\), and \(N=80\).
  \textbf{(a)} Hill tail-index estimate; the dashed reference is
  \(\alpha=1.5\).
  \textbf{(b)} Generated-to-target extreme-quantile ratio; the Brownian
  VP baseline is included for \(\nu=1.5\).
  \textbf{(c)} Estimated tail constant normalized by the finite-step
  EI prediction; the dashed reference is one.
  Curves are medians over \(12\) independent seeds, and shaded bands
  show the inter-seed \(16\%\)--\(84\%\) range.}
  \label{fig:app_numerical_tails}
\end{figure}

\paragraph{Interpretation.}
At index match, \(\nu=\alpha=1.5\), the three diagnostics are mutually
consistent: the Hill estimate remains close to \(1.5\), the
extreme-quantile ratio is approximately constant, and the normalized
tail constant remains close to one. The quantile plateau is nevertheless
below one over the displayed range, so generated extreme quantiles remain
smaller than their target counterparts. This illustrates the distinction
between matching the tail index and reproducing the target's extreme
scales: agreement with the EI tail constant in panel~(c) does not remove
the residual target discrepancy measured in panel~(b).

Away from index match, panel~(b) exhibits the predicted directions:
the ratio decreases for \(\nu=1.2<\alpha\) and increases for
\(\nu=3>\alpha\) as the quantiles become more extreme.
The other diagnostics are less conclusive. For \(\nu=1.2\), the Hill
estimates and normalized tail constants remain below their references;
for \(\nu=3\), both drift upward as larger upper fractions are included.
Thus the displayed thresholds do not yield joint stabilization of the
tail-index and tail-constant estimates. These observations support the
qualitative quantile comparison but do not resolve all asymptotic
predictions at the available sample size.

For the Brownian baseline at \(\nu=1.5\), the quantile ratio decreases
markedly as \(s\downarrow0\), whereas the stable ratio remains approximately
flat. This illustrates increasingly undersized Brownian extreme quantiles,
consistently with Proposition~\ref{prop:tail_separation}.

\paragraph{Numerical and statistical controls.}
Temporal refinements \(N\in\{40,80,160\}\), independent high-resolution
score tables, exact-\(\pi_T\) initialization, and direct
\(\mathrm{S}\alpha\mathrm{S}\), Pareto, and Gaussian samples provide
numerical and estimator controls.
The displayed inter-seed bands describe variability across replications
and are not confidence intervals; seed-level bootstrap diagnostics use
\(10000\) resamples.
At index match, the displayed diagnostics are qualitatively consistent
with the predicted tail index, a constant quantile ratio, and the EI
tail constant. Away from index match, threshold drift prevents an
equally conclusive joint interpretation. Pickands estimates, generalized Pareto fits,
and mean-excess ratios are retained as supplementary checks.

\subsection{Implementation and numerical controls}\label{app:exp:controls}

Stable variables use the convention
\(\mathbb E[e^{iuZ}]=e^{-|u|^\alpha}\) and are generated with the
Chambers--Mallows--Stuck construction \citep{ChambersMallowsStuck1976}.
Random seeds are independent across
reported replications; common random numbers are used only for paired
resolution controls. Propagation costs \(O(nN)\), and each one-dimensional
Wasserstein evaluation costs \(O(n\log n)\) because of sorting.

The reported propagation runs were produced with
JAX~\citep{SoftwareBradbury2018JAX} in CPU \texttt{float64}; the
deterministic denoiser calculation uses NumPy~\citep{SoftwareHarris2020NumPy}
and SciPy~\citep{SoftwareVirtanen2020SciPy} FFTs in \texttt{float64}. The environment was Windows~11 with Python~3.12.10,
JAX/JAXLIB~0.7.2, NumPy~2.3.3, SciPy~1.16.2, and \(16\) logical CPU cores.
No GPU result is reported. Configurations, seeds, score-table hashes, raw
seed-level summaries, and figure-generation commands are included in the
supplementary code, available at
\url{https://github.com/ainsi333/stable-sbgm-experiments}.

The project licenses for the stated versions are
\href{https://github.com/jax-ml/jax/blob/jax-v0.7.2/LICENSE}{Apache-2.0 for JAX},
\href{https://github.com/numpy/numpy/blob/v2.3.3/LICENSE.txt}{BSD-3-Clause for NumPy},
and \href{https://github.com/scipy/scipy/blob/v1.16.2/LICENSE.txt}{BSD-3-Clause for SciPy}.
These identify the upstream project licenses; bundled third-party components
retain their own applicable licenses and notices.

\end{document}